\documentclass{article}
\usepackage{scalefnt}
\usepackage{lipsum}
\usepackage{amsfonts}
\usepackage{graphicx}
\usepackage{epstopdf}
\usepackage{algorithmic}
\ifpdf
  \DeclareGraphicsExtensions{.eps,.pdf,.png,.jpg}
\else
  \DeclareGraphicsExtensions{.eps}
\fi

\usepackage{enumitem}
\setlist[enumerate]{leftmargin=.5in}
\setlist[itemize]{leftmargin=.5in}

\usepackage{amsopn}

\usepackage[
    backend=biber,
    style=authoryear-comp,
    maxcitenames=1,
    uniquename=false,
    sortcites=ynt,
    uniquelist=false,
    natbib=true,
    hyperref=true,
    maxbibnames=10,
    giveninits=true,
    uniquename=init,
    url=false,
    isbn=false,
    doi=false
 ]{biblatex}
 \DeclareNameAlias{sortname}{last-first}
\DeclareFieldFormat[article]{title}{``#1''} \DeclareFieldFormat[article]{journaltitle}{\mkbibemph{#1},} \DeclareFieldFormat[article]{volume}{\textbf{#1}} \DeclareFieldFormat[article]{number}{\mkbibparens{#1}}  \AtEveryBibitem{\clearfield{doi} \clearfield{issn} \clearfield{url} \clearfield{editor} }
\DefineBibliographyStrings{english}{ editor  = {}, editors = {}, }

\DeclareFieldFormat{editortype}{\ifboolexpr{ test {\ifnumequal{\value{listcount}}{1}} and test {\ifnumless{\value{liststop}}{2}} }
  {} {\mkbibparens{#1}} }

\usepackage{makecell}
\usepackage{booktabs}
\usepackage{amsmath}
\usepackage{amssymb}
\usepackage{mathtools}
\usepackage{bbm}
\usepackage{pgfplots}
\usepackage{subcaption}
\usepackage{hyperref}
\usepackage{cleveref}

\def\cl{L_W}
\def\cmin{\mathrm{d}_{\mathrm{min}}}
\def\cmax{\|W\|_\infty}
\def\Lipf{L_f}
\def\Lipfl{L_{f^{(l)}}}

\def\d{\mathrm{dist}}

\def\A{\mathbf{A}}
\def\X{\mathbf{X}}

\DeclareMathOperator{\Var}{Var}
\renewcommand{\P}{\mu}
 \renewcommand{\L}{\mathcal{L}}

\def\L2L{L^2(\mathcal{J}) \rightarrow L^2(\mathcal{J})}

\def\W{\tilde{W}}
\def\E{\mathbb{E}} 
\def\dd{\mathrm{d}}

 \usepackage{PRIMEarxiv}

\usepackage{amsthm}
\newtheorem{definition}{Definition}[section]
\newtheorem{proposition}{Proposition}[section]
\newtheorem{theorem}{Theorem}[section]

\newtheorem{lemma}{Lemma}[section]

\newtheorem{corollary}{Corollary}[section]

\newtheorem{remark}{Remark}[section]

\newcommand*\widebar[1]{\@ifnextchar^{{\wide@bar{#1}{0}}}{\wide@bar{#1}{1}}}

\title{Generalization Bounds for Message Passing Networks on Mixture of Graphons
}

\author{
  Sohir Maskey \\
  Ludwig-Maximilians-Universität München \\
  Munich, 
  Germany\\
  \texttt{maskey@math.lmu.de} \\
\And
  Gitta Kutyniok \\
  Ludwig-Maximilians-Universität München \\
  Munich Center for Machine Learning (MCML) \\
  Munich, 
  Germany\\
  \texttt{kutyniok@math.lmu.de} \\
     \And
  Ron Levie \\
  Technion - Israel Institute of Technology \\
   Haifa, 
  Israel\\
  \texttt{levieron@technion.ac.il } \\
}

\begin{document}
\maketitle

\begin{abstract}
We study the generalization capabilities of Message Passing Neural Networks (MPNNs), a prevalent class of Graph Neural Networks (GNN). We derive generalization bounds specifically for MPNNs with normalized sum aggregation and mean aggregation. Our analysis is based on a data generation model incorporating a finite set of template graphons. Each graph within this framework is generated by sampling from one of the graphons with a certain degree of perturbation.
In particular, we extend previous MPNN generalization results to a more realistic setting, which includes the following modifications: 1)  we analyze simple random graphs with Bernoulli-distributed edges instead of weighted graphs; 2) we sample both graphs and graph signals from perturbed graphons instead of clean graphons; and 3) we analyze sparse graphs instead of dense graphs.
In this more realistic and challenging scenario, we provide a generalization bound that decreases as the average number of nodes in the graphs increases. Our results imply that MPNNs with higher complexity than the size of the training set can still generalize effectively, as long as the graphs are sufficiently large.
\end{abstract}


\section{Introduction}
Graph Neural Networks (GNNs) \citep{scarselliGNNs, Bronstein_2017, hamilton2017inductive}  have become a powerful tool for analyzing graph-structured data. They have been widely adopted in recent years as a general-purpose tool across various fields of applied science and many industries.
 In particular, the class of GNNs called \emph{message passing neural networks (MPNNs)} \citep{gilmermpnn} has achieved considerable success  in various areas, such as protein and drug design \citep{STOKES2020688, alphafold}, molecular docking \citep{corso2023diffdock}, material sciences \citep{merchant2023scaling}, guiding human intuition \citep{davies2021advancing}  and many more. In most of these problems, the input to the MPNN is a graph with features on the nodes, and the MPNN either returns a single feature for the entire graph or an output feature for each node.  
 
 MPNNs are deep architectures that perform a spatial convolution-like operation to update the node features of the graph at each layers. A single layer of a MPNN updates the node features  of the input graph by first computing messages along edges, using learnable functions. In a second step, at each layer, every node aggregates all incoming messages in a permutation invariant manner  and updates its embedding accordingly. These node-level operations are often followed by a global pooling layer, after the last message passing layer, and a standard multi-layer-perceptron (MLP) to get a single feature output for the whole graph. 

Due to the significant practical success of MPNNs, there is a growing interest in understanding their theoretical properties. Researchers have delved into various aspects of MPNNs, including expressivity, assessed notably through the $1$-WL test \citep{xu2018how, Morris_Ritzert_Fey_Hamilton_Lenssen_Rattan_Grohe_2019}, oversmoothing \citep{Li_Oversmoothing1, Li_Oversmoothing2, nt_gnnsarelowpassfilters, keriven_oversmoothing}, and convergence properties \citep{levie2021transferability, keriven2020convergence, ruiz_transferability, maskey2023transferability}. While there is some understanding of why MPNNs generalize well to unseen graphs in supervised learning tasks, it remains somewhat limited. Additional details can be found in the subsequent discussion.

\subsection{Uniform Generalization Bounds}

The ability of MPNNs to generalize in supervised learning tasks can be described using a general approach from statistical learning called \emph{uniform convergence/generalization bounds}. Given a data--label pair $(\mathbf{x}, \mathbf{y})$ that is jointly drawn from an unkown distribution $\mu$, we consider a loss function $\mathcal{L}:\mathbb{R}^{d} \to \mathbb{R}_{+}$  that measures the discrepancy between the true label $\mathbf{y}$ and the output of the MPNN $\Theta$ on $\mathbf{x}$, via $\mathcal{L}(\Theta(\mathbf{x}),\mathbf{y})$. In statistical learning, an important quantity is the \emph{expected loss},  also called the \emph{statistical risk}, defined as
\[
R_{\mathrm{exp}} = \mathbb{E}_{(x,y) \sim \mu} \left[ \mathcal{L}\left( \Theta(x), y \right) \right].
\]
In practice, we only have access to a training set $\mathcal{T}= \{(\mathbf{x}^1, \mathbf{y}^1), \ldots, (\mathbf{x}^m, \mathbf{y}^m)\}$ that was sampled randomly and independently from the data distribution $\mu$.  To approximate the statistical risk, we define the \emph{empirical risk} as
\[
R_{\mathrm{emp}} = \frac{1}{m}\sum_{i=1}^m \mathcal{L}\left( \Theta(x_i), y_i \right).
\]
While the expected loss is the quantity of interest, in practice, we usually perform \emph{empirical risk minimization} (ERM), which involves minimizing the empirical risk over some hypothesis class $\mathcal{H}$ of MPNNs. This approximation is justified if the \emph{generalization error} $\mathrm{GE}$ (also called the \emph{representativeness}), defined as
\begin{equation}
\label{eq:GE}
   \mathrm{GE} = \sup_{\Theta \in \mathcal{H}} \left|R_{\mathrm{exp}}(\Theta) - R_{\mathrm{emp}}(\Theta) \right|, 
\end{equation}
decays to zero with increasing training set size $m$. This decay should hold in high probability with respect to the random choice of the training set.
 Note that in \Cref{eq:GE}, we cannot fix $\Theta$ to be the trained network $\Theta_{\mathcal{T}}$, and treat the empirical risk as a Monte Carlo approximation of the statistical risk. Indeed, the trained network  $\Theta_{\mathcal{T}}$ depends on the training set $\mathcal{T}$, and hence varies over the probability space of all training sets.

Most existing works that bound the generalization error for MPNNs are only applicable to specific types of MPNNs, and many are vacuous (generalization error $ >> 1$) even in simple scenarios \citep{scarselli2018vapnik, verma2019stability, pmlr-v119-garg20c, liao2021a, levie2023graphonsignal}. An exception is the work of the authors \citep{maskey2022generalization}, which shows non-vacuous generalization bounds for general  MPNNs under the assumption that graphs are generated by random graph models, and graphs that are generated by the same random graph model belong to the same class. However, that work is also limited. For example, graphs are assumed to be sampled without error from the random graph model, and the analysis is limited to dense weighted graphs.  The goal of the current paper is to extend the results of \citep{maskey2022generalization} to a more realistic setting, where the sampled graphs are unweighted, sparse, and noisy.

\subsection{Main Contribution}

The objective of this paper is to derive uniform generalization bounds for general MPNNs which are more realistic than previously proposed bounds, without any specific assumption about the architecture of the MPNN (e.g., number of parameters and types of message functions). To derive such bounds, we consider a generative model of sparse graphs with node features. 
We suppose that graphs with node features are sampled from a collection of \emph{graphons} with \emph{signals}. 
 We choose graphons as they are universal models that can approximate any graph, both in the sense of modelling exchangeable sequences (the  Aldous-Hoover representation theorem \citep{ALDOUS81,ALDOUS85,Hoover79,Hoover82}) and in the sense of graph limit theory \citep{Borgs2007graph,lovasz}. 
 For example, graphons   encompass traditional random graph models, such as Erdös-Rényi graphs \citep{erdos59a}, stochastic block models \citep{holland1983stochastic}, and random geometric graph models \citep{RGGPenrose}.

Our generative model is defined as follows. We consider a finite set of graphons and continuous signals (called the \emph{template graphon-signals}), and assume that the domains of the graphons and the continuous signals are compact metric spaces with finite Minkowski dimensions.  Graphs and their node features are randomly sampled via the following steps. First, one of the graphon-signals is randomly sampled. Then, the graphon-signal is perturbed randomly and attenuated to introduce sparsity. Next, points are randomly and independently sampled from the metric space. On these points,  the continuous signal is evaluated to obtain a discrete signal. Finally, by connecting pairs of points with edges--using probabilities from the sparsified graphon--a simple  random graph is obtained.  In this manner, noisy, random, sparse and simple graphs with noisy signals are sampled, which we call \emph{graph-signals}.  We consider a classification setting where graph-signals that are sampled from the same template graphon-signal belong to the same class. Based on this data generation model, we derive non-asymptotic generalization bounds for supervised graph classification tasks. For further details, refer to  \Cref{subsec:assumptions} and \Cref{Graph Classification}.

We support our theoretical results with numerical experiments in which we compare our generalization bound, as presented in \Cref{thm:main_gen_bound_deformed_graphon}, with existing bounds over several graph-signal classification datasets. 
The numerical results illustrate that our bounds are significantly tighter, by several orders of magnitude.

As stated above, the results in this paper expand upon and strengthen the preliminary results presented in \citep{maskey2022generalization}, by analyzing the more realistic setting of sparse simple random noisy graph-signals instead of dense weighted clean graph-signals. These extended results requires non-trivial extensions to the proof techniques.

\subsection{Related Work}
In this subsection we briefly survey different approaches for studying the convergence and generalization rates of GNNs. 

\paragraph{Convergence and Transferability of GNNs} 
The generalization capabilities of graph neural networks (GNNs) are closely linked to their convergence, as demonstrated in \citet{maskey2022generalization}. The concept of convergence in GNNs was first studied in \citet{levie2021transferability}, where the authors model graphs as samples from a limit object. They showed that as the number of nodes in the sampled graphs grows, the output of a GNN applied to these graphs converges to the output of the same GNN applied to the limit object. 
A result of GNN convergence is GNN \emph{transferability}, which refers to the ability to transfer a fixed GNN between different graphs that are sampled from the same limit object.
Many other studies have shown that spectral-based GNNs are linearly stable with respect to perturbations of the input graphs \citep{levie2021transferability, Gama_2020, kenlay2021stability}. Additionally, some works have shown that spectral-based GNNs are transferable when the input graphs approximate the same limit graphon, as seen in \citep{keriven2020convergence,https://doi.org/10.48550/arxiv.2112.04629, 9356126, maskey2022generalization, maskey2023transferability, cordonnier2023convergence}.  Following this analysis, \citet{cervino2023learning} demonstrated that gradients of spectral-based GCNNs are transferable under graphs approximating the same graphon. \citet{cai2022convergenceIGN} extended convergence results for spectral-based GNNs to invariant graph networks \citep{maron2018invariant}. Finally, \citet{le2024limits} presented transferability results for graphops, a generalization of graphons that accounts for sparse graphs.

\paragraph{Generalization bounds of GNNs}  \citet{scarselli2018vapnik} derived generalization bounds for implicitly defined GNNs by computing their VC-dimension. The work of \citet{du_graphNTK} analyzed the generalization capabilities of GNNs in the infinite-width limit. Additionally, \citet{pmlr-v119-garg20c} and \citet{liao2021a} derived data-dependent generalization bounds for specific MPNNs with sum aggregation using Rademacher complexity and PAC-Bayes approaches. 
\citet{morris2023wl} showed a connection between the number of graphs distinguishable by the 1-WL test and GNNs' VC dimension. Additionally, \citet{franks2024weisfeiler} derived conditions under which increasing the expressivity of GNNs beyond 1-WL leads to a reduction in their VC dimension.
The work of \citet{levie2023graphonsignal} introduced the concept of graphon-signal cut distance and demonstrated that any MPNN exhibits Lipschitz properties concerning this distance metric. Consequently,  \citet{levie2023graphonsignal}  established generalization bounds for MPNNs in the context of arbitrary graph-signal distributions. Notably, due to the broad generality of the data distribution,   the generalization bound in \citep{levie2023graphonsignal}  exhibits the slow convergence rate $O(1/\log\log(\sqrt{m}))$, where $m$ denotes the number of graphs in the training set. This behavior is asymptotically considerably slower than the generalization bound proposed in this paper, which follows a $m^{-1/2}$ behavior, under a certain prior on the data distribution.

\section{MPNNs and their Generalization on Mixtures of Graphons}
We denote simple or weighted graphs by $G=(V,E)$, where $V=[N]:=\{1, \ldots, N\}$ is the node set, and $E$ denotes the set of edges. The adjacency matrix of $G$ is denoted by $\mathbf{A} \in \mathbb{R}^{N \times N}$. If $G$ is simple, $\mathbf{A}$ has entries $a_{i,j} = 1$ if $(i,j) \in E$, and $0$ otherwise, for every $i,j\in V$. Weighted graphs have general edge weights in $[0,1]$. For $i \in V$, we define the neighbourhood $\mathcal{N}(i)$ of node $i$ by $\mathcal{N}(i)=\{ j \in V \, | \, a_{i,j} > 0 \}$.
Given a graph $G$ with $N$ nodes and adjacency matrix   $\mathbf{A} = (a_{i,j})_{i,j=1}^N$, we define the \emph{degree} of each node $i\in[N]$ to be
\[\mathrm{d}_i = \sum_{j=1}^N a_{i,j}.\]
For simple graphs, we have $\mathrm{d}_i = |\mathcal{N}(i)|$.

We study graphs with a feature  $\mathbf{f}_i\in\mathbb{R}^F$ at each node $i\in V$, where $F\in\mathbb{N}$ is called the feature dimension. We call the vector $\mathbf{f}= \{\mathbf{f}_1, \ldots, \mathbf{f}_N\} \in \mathbb{R}^{N \times F}$ the signal. We call $\mathcal{G}_N := \{0,1\}^{N \times N}$ the set of directed graphs with size $N \in \mathbb{N}$. We further define $\mathcal{S}^{N, F} := \mathcal{G}_N \times \mathbb{R}^{N\times F}$ as the space of all graph-signals with $N$ nodes and signals with feature dimension $F$. Finally, we define the space of all graph-signals as $\mathcal{S}^{F} := \bigcup_{N \in \mathbb{N}} \mathcal{S}^{N, F}$.

Given metric spaces $(\mathcal{X}, d_\mathcal{X})$ and $(\mathcal{Y}, d_\mathcal{Y})$, a function $g: \mathcal{X} \to \mathcal{Y}$ is called Lipschitz continuous if there exists a constant  $L_g\geq 0$ such that for every $x,x' \in \mathcal{X}$, we have
\[
d_\mathcal{Y}( g(x),g(x') ) \leq L_g d_\mathcal{X}(x,x').
\]
If the spaces are subsets of $\mathbb{R}^d$ for some $d\in\mathbb{N}$, 
we always endow them with the $L^{\infty}$-metric.
 The $\varepsilon$-covering number $\mathcal{C}(\chi, d_\chi; \varepsilon)$ of the metric space $(\chi, d_\chi)$ is defined as the minimum number of balls with radius $\varepsilon$ necessary to cover $\chi$, if there exists such a finite number.

For a metric-space signal $f: \chi \rightarrow \mathbb{R}^F$ and sample points $\X=(X_1, \ldots, X_N) \in \chi^N$, we define the sampling operator $S^\X$ by
\[
S^\X f \coloneqq \big(f(X_i) \big)_{i=1}^N \in \mathbb{R}^{N \times F}. 
\]
Given  a graph  $\mathbf{f} \in \mathbb{R}^{N \times F}$, we define its norm $\|\mathbf{f}\|_{\infty;\infty}$ by
\begin{equation*}
    \|\mathbf{f}\|_{\infty;\infty} \coloneqq \max_{i=1, \ldots, N} \max_{j=1, \ldots, F}  |\mathbf{f}_{i,j}|.
\end{equation*}
Finally, we define the distance $\d(f, \mathbf{f} )$  between $\mathbf{f}$ and $f$  as
\begin{equation}
\label{eq:main distGraphMetric}
    \d(f, \mathbf{f} ) \coloneqq  \|\mathbf{f} -  (S^\X f) \|_{\infty;\infty}.
\end{equation}

\subsection{Random Graph Models}
\label{subsec: RGM}
In this subsection, we define generative models of graphs, called \emph{random graph-signal models (RGSMs)}. RGSMs are based on a choice of a domain from which nodes are sampled. This domain is taken to be a \emph{metric-probability space} $(\chi, d, \mu)$, where  $\chi$ is a set, $d$ is a metric and $\mu$ is a Borel probability measure.  The nodes of random graphs are modeled as random independent samples from the metric probability space $\chi$. To model the connectivity of random graphs, we consider an adjacency structure on the space $\chi$, namely, a measurable function $W:\chi \times \chi \to [0,1]$, called a \emph{graphon}. We also consider a signal $f:\chi\rightarrow \mathbb{R}^F$ over the graphon domain. A \emph{random graph-signal} is then sampled from a RGSM as defined next. 

\begin{definition}
\label{def:RGM}
A \emph{random graph-signal model (RGSM)}  is defined as a tuple $\{\chi,W,f,\alpha\}$, or, in short, $\{W,f,\alpha\}$
of metric-probability space $\chi$, a measurable function $W:\chi \times \chi \to [0,1]$, called a \emph{graphon}, a measurable function  $f: \chi \to \mathbb{R}^F$, called a \emph{metric-space signal} and a parameter $\alpha \geq 0$, called the \emph{sparsity parameter}.  For every $N\in\mathbb{N}$,
a \emph{random graph-signal $\{G,\mathbf{f}\}$ of size $N$ sampled from the RGSM} is defined as follows. Let $X_1, \ldots, X_N$ be $N$  random independent samples  from $(\chi,\mu)$. The adjacency matrix  $\A=(a_{i,j})_{i,j=1}^N$ of $G$ is defined as a random variable, where each entry $a_{i,j}$ is a Bernoulli random variable $\mathrm{Ber}(N^{-\alpha} W(X_i, X_j))$ with $\mathbb{P}(a_{i,j}=1)=N^{-\alpha}W(X_i, X_j)$ and $\mathbb{P}(a_{i,j}=0)=1-N^{-\alpha}W(X_i, X_j)$.
The random signal $\mathbf{f}=(f_i)_{i=1}^N$ is defined by $\mathbf{f}_i = f(X_i)$. We say that $\{G,\mathbf{f}\}$ is \emph{drawn} from $W$, and denote $\{G,\mathbf{f}\} \sim_{\alpha} \{W,f\}$.  
\end{definition}

The parameter $\alpha$ in \Cref{def:RGM} controls the sparsity level of the graphs sampled from the RGSM. Setting $\alpha=0$ corresponds to sampling dense graphs, with average degree behaving like $N$ as $N\rightarrow\infty$.  Choosing $\alpha = 1$ corresponds to sampling sparse graphs, where the average degree is constant as $N\rightarrow\infty$.  

We furthermore note that the difference to the random graph models $\{\chi, W,f\}$ considered by \cite{maskey2022generalization} is that edges are sampled randomly via $\mathrm{Ber}(X_i,X_j)$ in our work. Note that \cite{maskey2022generalization} only sampled edge weights between two nodes $i$ and $j$ by defining them as the values $W(X_i,X_j)$.  By sampling simple graphs with Bernoulli edges, we can model sparse graphs, which is impossible in the weighted graph approach. 

\paragraph{Random Noise}
\Cref{def:RGM} assumes that the RGSM $\{W,f, \alpha\}$  is observed without noise. To make the model more realistic, we include noise as a parameter in the generative models. This adjustment allows for a more accurate representation of real-world scenarios where data may be corrupted or uncertain.
Denote the balls 
\begin{equation*}
    B_\varepsilon^\infty(\chi^2)= \left\{ U \in L^\infty(\chi^2) \; | \; \|U\|_{L^\infty(\chi^2)} \leq \varepsilon \right\}, ~B_\varepsilon^\infty(\chi)= \left\{ g \in L^\infty(\chi) \; | \; \|g\|_{L^\infty(\chi)} \leq \varepsilon \right\}.
\end{equation*}
We extend \Cref{def:RGM} as follows.

\begin{definition}
\label{def:RGM with noise}
Let $\{\chi,W,f,\alpha\}$ be a RGSM. Let $(V,g)$ be a random variable with values in $B_\varepsilon^\infty(\chi^2) \times B_\varepsilon^\infty(\chi)$ (where the balls are endowed with any Borel probability measure $\sigma$). We call the tuple $\{\chi,W,f,V,g,\alpha\}$ a \emph{noisy random graph-signal model}, and denote it in short by $\{W,f,\alpha,\varepsilon\}$. We define a \emph{random noisy graph-signal}  $\{G, \mathbf{f}\}$ as a random graph-signal  sampled from the random graph-signal model $\{\chi,W+V,f+g,\alpha\}$.
 We say that $\{G, \mathbf{f}\}$ is \emph{drawn} from $\{W,f,\alpha,\varepsilon\}$  \emph{with noise}, and denote $\{G, \mathbf{f}\} \sim_\alpha \{W,f,\varepsilon\}$.
\end{definition}

When $\alpha=0$ in \Cref{def:RGM} and  \Cref{def:RGM with noise} we drop the subscript in $\{G, \mathbf{f}\} \sim_\alpha \{W,f\}$ and $\{G, \mathbf{f}\} \sim_\alpha \{W,f,\varepsilon\}$, respectively.

\subsection{Message Passing Neural Networks}
\emph{Message passing neural networks (MPNNs)} are mappings between graph-signals to some finite dimensional space $\mathbb{R}^{d}$ based on a sequence of  local computations along a number of layers. At each layer, MPNNs update the signal value at each node  by local computations on the graph, in which messages are sent between nodes and their neighbors along the edges of the graph. We call the mapping that assigns to each edge in a graph $G$ a \emph{message} in $\mathbb{R}^{H}$ a \emph{message kernel}. A message kernel can be represented by a vector $\mathbf{U} \in \mathbb{R}^{N^2 \times H}$.
In this work we consider MPNNs in which all messages sent to each node are either averaged or summed and divided by $N$ to obtain the updated node feature for the next layer. For a given graph $G$ with adjacency matrix $\mathbf{A} = (a_{i,j})_{i,j=1}^N$,  the \emph{mean aggregation operator} $M_\A$ maps message kernels $\mathbf{U} \in \mathbb{R}^{N^2 \times H}$ to signals by 
\begin{equation*}
\begin{aligned}
M_{\A}: \mathbb{R}^{N^2 \times H}  \to \mathbb{R}^{N \times H}, \ \ \ 
\mathbf{U}  \mapsto \mathbf{f} \coloneqq \left(\frac{1}{\mathrm{d}_i} \sum_{j=1}^N a_{i,j}\mathbf{U}_{i,j,:} \right)_{i=1}^N
\end{aligned}
\end{equation*}
and the \emph{normalized sum aggregation} $S_\A$ is defined by
\begin{equation*}
\begin{aligned}
S_{\A}: \mathbb{R}^{N^2 \times H}  \to \mathbb{R}^{N \times H}, \ \ \ 
\mathbf{U}  \mapsto \mathbf{f} \coloneqq \left(\frac{1}{N} \sum_{j=1}^N a_{i,j}\mathbf{U}_{i,j,:} \right)_{i=1}^N.
\end{aligned}
\end{equation*}
The full definition of MPNNs is presented next.

\begin{definition}
\label{def:MPNN}
Let $\mathcal{A}_\A$ be either $M_\A$ or $S_\A$.
Let $T \in \mathbb{N}$ be a parameter called the \emph{number of layers}. For $t=1, \ldots, T$, let  $\Phi^{(t)}: \mathbb{R}^{2 F_{t-1}} \to \mathbb{R}^{H_{t-1}}$ and $\Psi^{(t)}:\mathbb{R}^{F_{t-1} + H_{t-1}} \to \mathbb{R}^{F_{t}}$ be functions called the \emph{message} and \emph{update} functions, where $F_t \in \mathbb{N}$ is called the \emph{feature dimension} of layer $t$, and $H_t \in \mathbb{N}$  the \emph{message dimension}. Let $\Upsilon:\mathbb{R}^{F_T}\to \mathbb{R}^{F_{T+1}}$, where $F_{T+1} \in \mathbb{N}$ is called the \emph{output dimension}, be a mapping called the \emph{post pooling layer}.
The corresponding \emph{parameters of the message passing neural network} 
are defined to be the tuple $
((\Phi^{(t)}, \Psi^{(t)})_{t=1}^T, \Upsilon)
$.

The corresponding \emph{message passing neural network (MPNN)} is  the mapping $\Theta$ that takes graph-signals as inputs, and returns outputs in $\mathbb{R}^{F_{T+1}}$, defined by the following sequence of operations. Let $\{G,\mathbf{f}\}$ be a graph-signal with $N$ nodes and signal $\mathbf{f}\in\mathbb{R}^{N\times F_0}$, where $N \in \mathbb{N}_{+}$. Let  $\mathbf{A} = (a_{i,j})_{i,j=1}^N$ be the adjacency matrix of $G$.
 For each $t \in \{ 1, \ldots, T \}$, we define  \emph{layer $t$ of the MPNN} $\Theta^{(t)}$, as the function that maps the input $\{G,\mathbf{f}\}$ to the graph-signal $\{G,\mathbf{f}^{(t)}\}$, where $\mathbf{f}^{(t)}\in\mathbb{R}^{N\times F_t}$, $t\in[T]$,  are defined sequentially by 
\[
\{G,\mathbf{f}^{(0)}\}=\{G,\mathbf{f}\},
\]
and for every node $i,j \in [N]$
\begin{equation}
\label{eq:graphAgg}
    \begin{aligned}
 \mathbf{u}_{i,j}^{(t)} & := \Phi^{(t)}(\mathbf{f}_i^{(t-1)}, \mathbf{f}_j^{(t-1)}) \\
 \mathbf{m}_i^{(t)} & := \mathcal{A}_\A(\mathbf{u}^{(t)})_i \\
 \mathbf{f}_i^{(t)} & := \Psi^{(t)}(\mathbf{f}_i^{(t-1)}, \mathbf{m}_i^{(t)})\\
\end{aligned}
\end{equation}
for every $i\in [N]$.  
The MPNN $\Theta$, applied to $\{G,\mathbf{f}\}$, is then defined by 
\[
\Theta(G,\mathbf{f}) =  \Upsilon \left( \frac{1}{N}\sum_{i=1}^N \mathbf{f}^{(T)}_i\right)\in\mathbb{R}^{T+1}.
\] 
\end{definition}

Next, we define notations for the mappings between consecutive layers of a MPNN.  For $t=1, \ldots, T$, we define the mapping from the $(t-1)$'th layer to the $t$'th layer of the MPNN as 
\begin{equation}
\label{eq:mpnn_layerwise_mapping}
\begin{aligned}
 \Lambda^{(t)}: \mathcal{S}^{F_{t-1}} &\to \mathcal{S}^{F_{t}}, \ \ \ 
\{G, \mathbf{f}^{(t-1)} \} & \mapsto \{G, \mathbf{f}^{(t)} \}.
\end{aligned} 
\end{equation}
We can then write a MPNN (without post pooling layer) as a composition of message passing layers, i.e., 
\[
\Theta^{(T)} =  \Lambda^{(T)} \circ  \Lambda^{(T-1)} \circ \ldots \circ  \Lambda^{(1)}. 
\]

While we only consider MPNNs with mean or normalized sum aggregation, as defined in \Cref{eq:graphAgg}, we note that other popular choices are sum, max or min aggregation.

 \subsection{Continuous MPNNs}
\emph{Continuous message passing neural networks (cMPNNs)} are applications of MPNNs on RGSMs. The definition of cMPNNs  is akin to the definition of cGCNs by \cite{keriven2020convergence} and follows Definition 2.4 in \citep{maskey2022generalization}. 

We first generalize the discrete mean and normalized sum aggregations to  their continuous versions. 
 For this, let $U:\chi \times \chi \to \mathbb{R}^H$ be a function, where $U(x,y)$ is interpreted as a message sent from point $y$ to $x$ in the metric space $\chi$. We call such $U$ a \emph{message kernel} as before.   For a graphon $W:\chi \times \chi \to [0,1]$, the continuous mean aggregation operator $M_W: L^\infty(\chi^2) \to L^\infty(\chi)$ is then defined by 
\[
M_W(U)(x) := \int_\chi \frac{W(x,y)}{\mathrm{d}_W(x)} U(x,y) d\mu(y),
\]
where 
\begin{equation}
\label{eq:Wdeg}
 \mathrm{d}_W(x) = \int_\chi W(x,y) d\mu(y)   
\end{equation}
 is the graphon degree of $W$ at $x \in \chi$, and normalized sum aggregation, or \emph{integral aggregation} is defined by 
 \[
S_W(U)(x) := \int_\chi W(x,y)U(x,y) d\mu(y),
\]
The following definition is similar to Definition 2.4 in \citep{maskey2022generalization}.
\begin{definition}
\label{def:cMPNN}
Let $\mathcal{A}_W$ be either $M_W$ or $S_W$. Consider a list of parameters of a MPNN, with message and update functions $\Phi^{(t)}: \mathbb{R}^{2 F_{t-1}} \to \mathbb{R}^{H_{t-1}}$ and $\Psi^{(t)}:\mathbb{R}^{F_{t-1} + H_{t-1}} \to \mathbb{R}^{F_{t}}$ with $T$ layers  and post pooling layer $\Upsilon$. For each $t \in \{ 1, \ldots, T \}$, we define $\Theta^{(t)}$  as the mapping 
that maps the input graphon $W$ with metric-space signal $f=f^{(0)}: \chi \to \mathbb{R}^{F_0}$ 
to the signal in the $t$-th layer by
\begin{equation}
    \Theta^{(t)}: \mathcal{W} \times L^\infty(\chi) \rightarrow  \mathcal{W} \times  L^\infty(\chi), \ \ \ \{W, f\} \mapsto \{W,f^{(t)}\},
    \label{cMPNNdef}
\end{equation}
where $f^{(t)}$ are defined sequentially as follows. For any $x,y\in\chi$, 
\begin{equation}
    \begin{aligned}
& \mu^{(t)}(x,y) = \Phi^{(t)}\big(f^{(t-1)}(x), f^{(t-1)}(y)\big) \\
& g^{(t)} = \mathcal{A}_W \mu^{(t)} \\
& f^{(t)}(x) = \Psi^{(t)}\Big(f^{(t-1)}(x), g^{(t)}(x)\Big).
\end{aligned}
\end{equation}
Here, $\mu^{(t)}$ is called the \emph{message kernel} at layer $t$, and $g^{(t)}$ the aggregated message.
The \emph{continuous message passing neural network (cMPNN)} $\Theta$ is then defined by 
\[
\Theta(W,f) =  \Upsilon \left(  \int_\chi f^{(T)}(x) d\P(x)\right) .
\] 
\end{definition}

Note that the output of a cMPNN is a single vector $\Theta_\mathcal{W}(W,f) \in \mathbb{R}^{F_{T+1}}$. Therefore, it is possible to compare the output of a graph MPNN and a cMPNN after pooling by computing their distance in any chosen norm in $\mathbb{R}^{F_{T+1}}$. In this paper, we use the supremum norm for this purpose.

Similarly to the graph MPNN case (see \Cref{eq:mpnn_layerwise_mapping}), and using a slight abuse of notation, we define $\Lambda^{(t)}$ as the mapping from the $(t-1)$'th layer to the $t$'th layer of the cMPNN: $f^{(t-1)}\mapsto f^{(t)}$.
Thus, we can express
\begin{equation}
\label{def:LayerMapping_main}
\Theta^{(T)} =  \Lambda^{(T)} \circ  \Lambda^{(T-1)} \circ \ldots \circ  \Lambda^{(1)}.
\end{equation}

\subsection{Graph Classification setting}
\label{Graph Classification}
We study graph classification tasks where each class is defined by a finite set of RGSMs. In this context, each graph-signal is generated by sampling a RGSM, subject to noise, and the class of the signal is determined by the underlying RGSM. For simplicity, we assume that each class is uniquely associated with a single RGSM, without loss of generality. It's noteworthy that the analysis remains unaffected by the specific number of classes; what matters is the total number of RGSMs involved.

\subsubsection{Data Distribution: Mixture of Graphons}
As outlined above, the generalization analysis we present in this paper is data-dependent. That is, we focus on a probability measure $\mu$ on the graph-signal space $\mathcal{S}^F$, from which we sample graph-signal pairs along with their respective labels. The construction of this probability space and the related distribution is detailed in \Cref{sec:appendix_prob_measure}, and for the sake of brevity, we provide an outline of a sampling procedure for a graph-signal pair consistent with the distribution defined there.

We consider a multi-class graph classification scenario with classes $j=1, \ldots, \Gamma$. To sample a graph-signal $\mathbf{x}$ with class label $\mathbf{y}$, we first select the class according to the probability $\gamma_j$, that is, for $(\mathbf{x},\mathbf{y}) \sim \mu$ and $j = 1, \ldots, \Gamma$, we have $\gamma_j = \mathbb{P}(\mathbf{y} = j)$. Independently of this class selection, we sample the number of nodes, $N \in \mathbb{N}$, from a discrete distribution. Once a class label $\mathbf{y} \in \{1, \ldots, \Gamma\}$ and a graph size $N$ have been selected, a random graph-signal $\{G,\mathbf{f}\}$ with $N$ nodes, is drawn from $\{W^{\mathbf{y}}, f^{\mathbf{y}}, \varepsilon\}$ with noise. We call the distribution $\mu$ a \emph{mixture of graphons}.

Lastly, we write $\mathcal{T} \sim \mu^m$ to describe a dataset $\mathcal{T}$ of $m$ samples $(\mathbf{x}^1, \mathbf{y}^1), \ldots, (\mathbf{x}^m, \mathbf{y}^m)$, which are drawn independently  from $\mu$.

\subsubsection{Assumptions on RGSMs}
\label{subsec:assumptions}
In this subsection, we introduce some restrictions that we assume on all RGSMs discussed in this paper. 
We start by restricting the space of graphons.

\begin{definition}
\label{ass:graphon}
Let $(\chi,d, \mu)$ be a metric-probability space  and  $W:\chi^2\rightarrow[0,1]$ be a graphon. We say that $W$ is an \emph{admissible graphon} if the following holds. \begin{enumerate}
\item \label{ass:graphon1}
The space $\chi$ is compact, and there exist $D_\chi, C_\chi \geq 0 $ such that the covering number of $\chi$ satisfies  $\mathcal{C}(\chi, d_\chi; r) \leq C_\chi r^{-D_\chi} $ for every $r>0$. \footnote{This is related to the  Minkowski dimension of $\chi$, which is defined as the infimum over the set of all such possible $D_\chi$.}  
\item \label{ass:diamChi}
The diameter of $\chi$ is bounded by 1. Namely, $\mathrm{diam}(\chi):=\sup_{x,y\in\chi}d(x,y) \leq 1$.
\item \label{ass:graphon11}
    For every $y \in \chi$, the function $W(\cdot, y)$ is Lipschitz continuous (with respect to its first variable) with Lipschitz constant $L_W$. 
    \item \label{ass:GraphonLip2nd} For every $x \in \chi$, the function $W(x,\cdot)$ is Lipschitz continuous (with respect to its second variable) with Lipschitz constant $L_W$. 
    \item \label{ass:graphon12}There exists a constant $\cmin > 0$ such that for every $x \in \chi$  the degree of the graphon (see \Cref{eq:Wdeg}) satisfies $\mathrm{d}_W(x) \geq \cmin$.
    \item \label{ass:graphon_diagonal_1} For all $x \in \chi$ we have $W(x,x) = 1$.
\end{enumerate}
\end{definition}
If not stated otherwise, when we consider an admissible graphon $W$, we assume that it satisfies the assumption in \Cref{ass:graphon} with underlying space $(\chi, d, \mu)$, constants $C_\chi, D_\chi$ for \Cref{ass:graphon1}, Lipschitz constant $L_W$, and $\cmin$ for \Cref{{ass:graphon12}}. We note that \Cref{{ass:graphon12}} is only required for MPNNs with mean aggregation. 

In the remainder of the paper, we consider a classification setting as described in \Cref{Graph Classification}. We assume that the the graphon $W^j$ corresponding to each class  $j=1, \ldots, \Gamma$ is admissible over the metric-probability space $\chi^j$. We denote by $C,D$ the maximal covering parameters over all classes, namely, $\mathcal{C}(\chi^j, d_{\chi^j}; r) \leq C r^{-D}$ for every class $j$. We similarly denote the maximal Lipschitz  constant of all graphons by $L_W$ and minimal degree by $\mathrm{d}_{\min}$. We assume that for every class $j$,  the metric-space signal $f^j:\chi^j\rightarrow \mathbb{R}^F$ is Lipschitz continuous with Lipschitz constant $L_{f^j}$. We denote by $L_f$ the maximal signal Lipschitz constant of all classes.

\subsubsection{Assumptions on the Loss Function and MPNN}
For the graph classification task we only consider loss functions $\mathcal{L}$ that are Lipschitz continuous with Lipschitz constant $L_\mathcal{L} \in \mathbb{R}$. We note that cross-entropy is not Lipschitz, but cross-entropy composed on softmax is. Hence, to implement standard classification, we consider the loss to be this composition, and the last layer of the network is a linear classifier without softmax. 
Lastly, we define for $L,B,K, T \in \mathbb{N}$  the hypothesis space $\mathrm{Lip}_{L, B, K, T}$ that comprises all MPNNs $\Theta = ((\Phi^{(t)}, \Psi^{(t)})_{t=1}^T, \Upsilon )$ that satisfy the following conditions: for every $t=1,\ldots, T$, the message and updates functions $\Phi^{(t)}$ and $\Psi^{(t)}$ are Lipschitz continuous with
 $L_{\Phi^{(t)}}, L_{\Psi^{(t)}}  \leq L$  and  $\|\Phi^{(t)}(0,0)\|_\infty , \|\Psi^{(t)}(0,0)\|_\infty \leq B$, and  $\Upsilon$ is Lipschitz continuous with $L_{\Upsilon} \leq K$.

\subsection{The Main Generalization Result}
We now present our main result on generalization of MPNNs on mixture of graphon models

\begin{theorem}
\label{thm:main_gen_bound_deformed_graphon}
There exist constants $C, C'>0$ such that 
\begin{equation}
\label{eq:main_gen_bound_deformed_graphon}
\begin{aligned}   &   \E_{\mathcal{T}\sim \mu^m}\left[\sup_{\Theta \in \mathrm{Lip}_{L,B}} \Big(R_{emp}(\Theta)   - R_{exp}(\Theta) \Big)^2  \right]   \leq   \frac{2^\Gamma8\|\mathcal{L}\|_\infty^2\pi}{m} 
\\ &   + \frac{2^\Gamma L_{\mathcal{L}}^2   }{m}  \left( C \cdot \E_{N \sim \nu} \left[ \frac{1+\log(N)}{N}N^{2\alpha} + \frac{1 + \log(N)}{N^{1/(D_{\chi} + 1)}}N^{2\alpha} + \mathcal{O}\left( \exp(-N) \right) \right] + C'\varepsilon\right),
\end{aligned}
\end{equation}
where $C$ and $C'$ are specified in \Cref{appendix: Proof Gen Bound} in the Appendix.
\end{theorem}

The constants $C$ and $C'$ in \Cref{thm:main_gen_bound_deformed_graphon} represent the complexity associated with the hypothesis space $\mathrm{Lip}_{L,B,K,T}$. This complexity depends on several parameters: the upper bound $L$ of the Lipschitz constants for the message and update functions, the Lipschitz constant $K$ of the final classifier layer, the depth $T$ of the MPNN, and the regularity of the underlying RGSMs $\{\chi^j, W^j\}_{j=1}^\Gamma$. Asymptotically, the complexity of $\mathrm{Lip}_{L,B,K,T}$ is bounded by $\mathcal{O}\left( L^{2T}\max_{j=1, \ldots, \Gamma}\left( \sqrt{\log(C_{\chi^j})} + \sqrt{D_{\chi^j}}\right)L_{W^j}\|W^j\|_\infty^T K \right)$.

To interpret the bound in \Cref{eq:main_gen_bound_deformed_graphon}, we first note that in typical MPNN learning settings the complexity terms associated with the hypothesis class are very high. Therefore, the first term of the bound in \Cref{eq:main_gen_bound_deformed_graphon} is negligible. 
From the second term of the bound in \Cref{eq:main_gen_bound_deformed_graphon}, we observe that data distributions with larger graphs tend to generalize better. To understand this, note that typical uniform generalization bounds tend to converge to zero like $m^{-1}$ -- they decay in the number of samples of the training set. In the MPNN case,  since message passing is a computation which is shared among all neighborhoods of the graphs, the generalization error does not only treat   graph-signals as samples, but rather treats each neighborhood  of each graph-signal as one sample. However, since neighborhoods are correlated, and the amount of correlation depends on the dimension $D_{\chi}$ of the underlying metric space, the decay behaves like $m^{-1}N^{-1/(D_{\chi}+1)+2\alpha}$ rather than $m^{-1}N^{-1}$. Still, when the sparsity level satisfies $\alpha < 1/(2D_{\chi}+2)$, large graphs lead to lower generalization error.

We provide a concise proof for the case when $\Theta$ is a MPNN with mean aggregation.
The proof for the scenario where $\Theta$ is a MPNN with normalized sum aggregation is simpler and
follows the same steps. For the sake of brevity, we omit this latter proof.

The proof of \Cref{thm:main_gen_bound_deformed_graphon} involves multiple steps, detailed in \Cref{sec:unif_cov} and \Cref{sec:gen_ana}, and elaborated in full in the Appendix. First, we prove a uniform convergence results in \Cref{sec:unif_cov}:  in high probability, for every MPNN, the difference between its output on a graphon-signal and a sampled graph-signal is of the order \(\mathcal{O}(N^{-\frac{1}{2(D_\chi+1)}})\). This result, outlined in \Cref{cor: main convergence}, assumes a lower bound on the number of nodes in the sampled graph-signal. We overcome this limitation by bounding the worst-case error between the
graph MPNN and the cMPNN. Given that this ``worst-case'' event has exponentially small probability, we bypass the requirement for a lower bound on the number of nodes. We moreover convert the analysis from a high-probability framework to an expectation-based framework, leading to  \Cref{cor:main_unifExpValue}. Finally, the proof of \Cref{thm:main_gen_bound_deformed_graphon} follows from \Cref{cor:main_unifExpValue} and by invoking the Bretagnolle-Huber-Carol inequality (see \Cref{lemma:BHCineq} in the Appendix).

\section{Uniform Convergence of graph MPNNs to corresponding cMPNNs}
\label{sec:unif_cov}
In this section given a RGSM and a randomly sampled graph-signal, we show that there exists an event of high probability (with respect to the sampling of the graph-signal), in which every MPNN from the hypothesis class $\mathrm{Lip}_{L,B,K}$, when applied to the sampled graph-signal, approximates the MPNN applied on the graphon-signal. We stress that the event of high probability is uniform in  $\mathrm{Lip}_{L,B,K}$. In particular, this result represents a stronger outcome compared to standard transferability analysis \citep{ruiz_transferability, keriven2020convergence, maskey2023transferability}, where different MPNNs require different events.  We note that the uniform analysis is required for generalization analysis, since in learning settings the network depends on the sampled dataset, and cannot be treated as fixed and predefined.

Our primary goal is to show that as the number of nodes in graph sampled from the RGSM increases, the difference between the outputs of the graph MPNN and the corresponding cMPNN  decreases. To accomplish this, we conduct a layer-wise analysis. 

 \begin{proposition}
 \label{prop: main uniform monte carlo}
Let $W$ be an admissible graphon. Suppose that $\mathbf{X}=\{X_1, \ldots, X_N\}$ are drawn i.i.d. via $\mu^N$, noise $(V,g)$ is drawn via $\sigma$, where $\sigma$ is a Borel probability measure over $B_\varepsilon^\infty(\chi^2)\times B_\varepsilon^\infty(\chi)$,
and $\A_{i,j} \sim \mathrm{Ber}\big(W(X_i, X_j) + V(X_i,X_j)\big)$ for $i,j=1, \ldots, N$.  Then, for every $p \in (0,1/2)$, there exists an event of probability at least $1-2p$ such that for every $f:\chi \to \mathbb{R}^F$ and $\Phi: \mathbb{R}^{2F} \to \mathbb{R}^H$ with Lipschitz constants bounded by $L_f$ and $L_\phi$ respectively,
 \[
\begin{aligned}
&  \max_{X_i\in \mathbf{X}} \left\| \frac{1}{N} \sum_{j=1}^N \A_{i,j} \Phi\big(
 f(X_i), f(X_j)
 \big) - \int_\chi W(X_i,y) \Phi\big(
 f(X_i), f(y)\big) d\mu(y) \right\|_{\infty} \\ &  \leq C(L_f, L_\Phi, L_W, \|W\|_\infty, D_\chi, C_\chi)\frac{  \sqrt{\log(2N/p)}}{N^{\frac{1}{2(D_\chi+1)}}}    +    C(L_f, L_\Phi) \varepsilon,
\end{aligned}
\]
where $C(L_f, L_\Phi, L_W, \|W\|_\infty, D_\chi, C_\chi)$ and $ C(L_f, L_\Phi)$ are constants that depend linearly on the parameters specified in the respective brackets.
\end{proposition}
 \Cref{prop: main uniform monte carlo} provides a bound on the difference between discrete and continuous aggregation of messages. The bound is proportional to  $N^{-1/2(D_\chi+1)}$, where $D_\chi$ is the Minkowski dimension of $\chi$. While standard Monte Carlo results provide bounds that are proportional to $N^{-1/2}$ and do not depend on the dimension of the underlying space, we emphasize that the bound of \Cref{prop: main uniform monte carlo} holds uniformly for any choice of the message function and metric-space signal. 

The complete statement of \Cref{prop: main uniform monte carlo} is provided in  \Cref{lemma:uniform montecarlo} in \Cref{AppendixB} in the Appendix. The proof of this result is quite technical and is included \Cref{AppendixB}. We then derive the following corollary that bounds the sample error in one layer of a MPNN.

\begin{corollary}
\label{cor: main error first layer}
Let $W$ be an admissible graphon. Suppose that $\mathbf{X}=\{X_1, \ldots, X_N\}$ are drawn i.i.d. from $\chi$ via $\mu$, the noise $(V,g)$ is drawn via $\sigma$, where $\sigma$ is a Borel probability measure over $B_\varepsilon^\infty(\chi^2)\times B_\varepsilon^\infty(\chi)$, and $\A_{i,j} \sim \mathrm{Ber}\big( W(X_i,X_j) + V(X_i,X_j)\big)$. 
Let $\mathcal{A}_W$ be either $M_W$ or $S_W$, and let $\mathcal{A}_\A$ be either $M_\A$ or $S_\A$ respectively. Let $\dd=\cmin$  if $\mathcal{A}=M$, and $\dd=1$ if $\mathcal{A}=S$. 
For $p \in (0, \frac{1}{4})$ and $N \in \mathbb{N}$ such that  
\begin{equation}
\begin{aligned}
\label{eq:main lowerBoundGraphSizeN}
\sqrt{N} \geq   \max\Bigg\{  &   4\sqrt{2} \frac{\sqrt{ \log(2N/p)}}{\dd}, \\ &  4 \Big(\zeta   \frac{\cl}{\dd} \big(\sqrt{\log (C_\chi)} +  \sqrt{D_\chi}\big) +
    \frac{\sqrt{2} \cmax + \zeta \cl}{ \dd } \sqrt{\log 2/p}\Big)\Bigg\},
\end{aligned}
\end{equation}
there exists an event $\mathcal{F}_{\rm Lip}^p$
with probability $\mu(\mathcal{F}_{\rm Lip}^p) \geq 1-4p$  such that  for every  choice of constants $L_f,L_{\Phi},L_{\Psi}>0$ and Lipschitz continuous functions $f:\chi \to \mathbb{R}^{F}$ with  Lipschitz constant bounded by $L_f$, $\Phi: \mathbb{R}^{2F} \to \mathbb{R}^{H}$ with   Lipschitz constant bounded by $L_\Phi$, and $\Psi: \mathbb{R}^{F+H} \to \mathbb{R}^{F_1}$ with   Lipschitz constant bounded by $L_\Psi$,
\begin{equation}
\label{eq: main error first layer}
\begin{aligned}
 & \max_{X_i} \left\| \Psi\Big(f(X_i), \mathcal{A}_\A \big(\Phi (f,f)\big) (X_i) \Big)  - \Psi\Big( f(X_i), \mathcal{A}_W \big(\Phi (f,f)\big) (X_i) \Big)  \right\|_\infty   \\
& \leq  D(L_f, L_\Phi, L_W, \|W\|_\infty, D_\chi, C_\chi)  \frac{ \sqrt{\log(N/p)}  }{N^{\frac{1}{2(D_\chi+1)}} \dd} + D(L_f, L_\Phi)  \varepsilon  + \mathcal{O}\left(    \frac{ \sqrt{\log(N)} }{\sqrt{N}} \right),
\end{aligned}
\end{equation}
where $D(L_f, L_\Phi, L_W, \|W\|_\infty, D_\chi, C_\chi)$ and $ D(L_f, L_\Phi)$ constants that are dependent on the parameters specified in the respective brackets.  
\end{corollary}
\begin{proof}
We provide a concise proof for the case when $\mathcal{A}_\A$ is mean aggregation. The proof for normalized sum aggregation is simpler and follows similarly. Constants and certain details are omitted for brevity. For  the full proof, including all constants and detailed explanations, refer to \Cref{AppendixB}, from \Cref{lemma:bernoulliHoeffdings} to \Cref{cor:monte carlo after update function}, in the Appendix.

An application of Hölder and Dudley's inequalities shows that for sufficiently large $N$ the value $d_{\A}(X_i)= \frac{1}{N}\sum_{i=1}^N \A_{i,j}$ and the graphon degree $d_W(X_i)$ are close in high probability, i.e., for any $p\in (0,1)$ we have with probability at least $1-p$: for all $X_i = X_1, \ldots, X_N$,
\begin{equation}
    \label{eq:degrees close}
    \|\mathrm{d}_{\A}(X_i)  -\mathrm{d}_W(X_i)\|_\infty \lesssim  \frac{\sqrt{\log(N/p)}}{\sqrt{N}} + \varepsilon.
\end{equation}

 We consider the joint event $\mathcal{F}^p_{\mathrm{Lip}}$ of probability at least $1-4p$ in which \Cref{prop: main uniform monte carlo} and \Cref{eq:degrees close} hold.
Then for  $X_i = X_1, \ldots, X_N$
\begin{equation}
\label{eq:Cor3.2}
\begin{aligned}
& \left\|\frac{1}{N} \sum_{j=1}^N \frac{\A_{i,j}}{\mathrm{d}_{\A}(X_i)} \Phi\big(f(X_i),f(X_j)\big)
-  \int \frac{W(X_i,y)}{\mathrm{d}_W(X_i)}\Phi\big(f(X_i),f(y)\big) d\mu(y)\right\|_\infty \\
  &  \leq
\left\|
\frac{1}{N} \sum_{j=1}^N \frac{\A_{i,j}}{\mathrm{d}_{\A}(X_i)} \Phi\big(f(X_i),f(X_j)\big)
- 
\frac{1}{N} \sum_{j=1}^N \frac{\A_{i,j}}{\mathrm{d}_W(X_i)}\Phi\big(f(X_i),f(X_j)\big)
\right\|_\infty  \\
& +
\left\|
\frac{1}{N} \sum_{j=1}^N \frac{\A_{i,j}}{\mathrm{d}_W(X_i)}\Phi\big(f(X_i),f(X_j)\big) - \int \frac{W(X_i,y)}{\mathrm{d}_W(X_i)}\Phi\big(f(X_i),f(y)\big) d\mu(y) \right\|_\infty
\\
& \leq  \left| \frac{1}{\mathrm{d}_{\A}(X_i)} - \frac{1}{\mathrm{d}_W(X_i)}\right| \mathrm{d}_{\A}(X_i) \|\Phi(f,f)\|_\infty \\ & + \left|\frac{1}{\mathrm{d}_W(X_i)}\right| \left\| \frac{1}{N} \sum_{j=1}^N {\A_{i,j}}  \Phi\big(f(X_i),f(X_j)\big)
- 
\int  W(X_i,y) \Phi\big(f(X_i),f(y)\big) d\mu(y) \right\|_\infty  \\
& = \text{(I)} + \text{(II)}
\end{aligned}
\end{equation}
By \eqref{eq:degrees close}, in the event of $\mathcal{F}_{\mathrm{Lip}}^p$ we have 
\[
\begin{aligned}
 \left| \frac{1}{\mathrm{d}_{\A}(X_i)} - \frac{1}{\mathrm{d}_W(X_i)} \right|\mathrm{d}_{\A}(X_i) & \leq
\left| \frac{\mathrm{d}_W(X_i) - \mathrm{d}_{\A}(X_i)}{\mathrm{d}_{\A}(X_i)\mathrm{d}_W(X_i)} \right|\mathrm{d}_{\A}(X_i)
\\
 & \lesssim \frac{1}{ \mathrm{d}_{\mathrm{min}} }  \left(\varepsilon +  \frac{\sqrt{\log(N/p)}}{\sqrt{N}} \right).
\end{aligned}
\] 
Hence, for every considered $\Phi$ and $f$ 
\[
\begin{aligned}
& \left| \frac{1}{\mathrm{d}_{\A}(X_i)} - \frac{1}{\mathrm{d}_W(X_i)}\right| \mathrm{d}_{\A}(X_i) \|\Phi(f,f)\|_\infty
\\
& \lesssim  \frac{  \|\Phi(f,f)\|_\infty}{ \mathrm{d}_{\mathrm{min}} }  \left({\varepsilon} + \frac{\sqrt{\log(N/p)}}{\sqrt{N}}\right).
\end{aligned} 
\]

Furthermore, the second term (II) on the RHS of \Cref{eq:Cor3.2} is bounded by \Cref{prop: main uniform monte carlo} and by  $\mathrm{d}_W \geq \cmin$, i.e., for every  $\Phi$ and $f$ that satisfy the conditions of \Cref{prop: main uniform monte carlo}
\[
\begin{aligned}
&  \left|\frac{1}{\mathrm{d}_W(X_i)}\right|\left\| \frac{1}{N} \sum_{j=1}^N {\A(X_i,X_j)}  \Phi\big(f(X_i),f(X_j)\big)
- 
\int  W(X_i,y) \Phi\big(f(X_i),f(y)\big) d\mu(y) \right\|_\infty \\ &  \lesssim  \frac{1}{\cmin} \left(  \frac{  \sqrt{\log(N/p)}}{N^{\frac{1}{2(D_\chi+1)}}}    +    \varepsilon\right).
\end{aligned}
\]

Hence, we bound the RHS of \Cref{eq:Cor3.2} by
\begin{align*}
  &  \left\|\frac{1}{N} \sum_{j=1}^N \frac{\A(X_i,X_j)}{\mathrm{d}_{\A}(X_i)} \Phi\big(f(X_i),f(X_j)\big)
-  \int \frac{W(X_i,y)}{\mathrm{d}_W(X_i)}\Phi\big(f(X_i),f(y)\big) d\mu(y)\right\|_\infty \\ & \lesssim \frac{\sqrt{\log(N/p)}}{N^{\frac{1}{2(D_\chi+1)}} \cmin} + \frac{\sqrt{\log(N)}}{\sqrt{N} \cmin} + \varepsilon.
\end{align*}
We use the Lipschitz continuity of $\Psi$ to finish the proof.
\end{proof}

\begin{remark}
For MPNNs with normalized sum aggregation  \Cref{cor: main error first layer} can be improved. First, assumption \Cref{eq:main lowerBoundGraphSizeN} on the lower bound on the number of nodes can be omitted. Furthermore, the bound can be improved: the term $\mathcal{O}(\log(N)/N)$ may be omitted. This adjustment is applicable to all subsequent results.
\end{remark}

From \Cref{cor: main error first layer}, the error between the cMPNN and graph MPNN in a single layer is bounded,  assuming that there is no error in the previous layer. Additionally, it is worth noting that cMPNNs preserve the Lipschitz continuity and boundedness of Lipschitz continuous and bounded input metric-space signals, as demonstrated in Lemma B.7 and Lemma B.9 in \citep{maskey2022generalization}. Consequently, we can use \Cref{cor: main error first layer} recursively, leading to the following straightforward corollary.

\begin{corollary}
\label{cor:main layerwise error}
Let $W$ be an admissible graphon. Let $p \in (0, \frac{1}{4})$.  
 Consider a graph $\{G,\mathbf{f}\} \sim \{W,f, \varepsilon\}$ with $N$ nodes and corresponding graph features,  where $N$ satisfies  \Cref{eq:main lowerBoundGraphSizeN}. Let $\dd=\cmin$  if the MPNN uses mean aggregation, and $\dd=1$ if it uses sum aggregation. If the event $\mathcal{F}_{\rm Lip}^p$ from  \Cref{cor: main error first layer} occurs, then the following is satisfied:  
 for every MPNN $\Theta$ and $f:\chi \to \mathbb{R}^F$ with Lipschitz constant $L_f$,  
\begin{equation}
\begin{aligned}
\label{eq:main layerwise error}
  & \d\left( \Lambda^{(t+1)}(G, S^X f^{(t)}),  \Lambda^{(t+1)}(W, f^{(t)})\right)   \leq  D(L_{f^{(t)}}, L_{\Phi^{(t+1)}}, L_W, \|W\|_\infty, D_\chi, C_\chi)  \frac{\sqrt{ \log(N/p)}  }{N^{\frac{1}{2(D_\chi + 1)}} \dd} \\ & + D(L_{f^{(t)}}, L_{\Phi^{(t+1)}})  \varepsilon  + \mathcal{O}\left(    \frac{ \sqrt{\log(N)} }{\sqrt{N}} \right)
\end{aligned}
\end{equation}
for all $t = 0, \ldots, T-1$, where $f^{(t)}=\Theta^{(t)}(W, f)$ as defined in \Cref{cMPNNdef}, and $\Lambda^{(t+1)}$ is defined in  \Cref{def:LayerMapping_main}. 
\end{corollary}

Building on the previous result in \Cref{cor:main layerwise error}, we can establish a recurrence relation between the errors for consecutive layers of a MPNN, as shown in the following lemma.

\begin{lemma}
\label{lemma:main_rec_rel}
Suppose that  the assumptions of \Cref{cor:main layerwise error} hold. 
If the event $\mathcal{F}_{\mathrm{Lip}}^p$ from \Cref{cor: main error first layer} occurs, then,  for every MPNN $\Theta$ and $f:\chi \to \mathbb{R}^F$ with Lipschitz constant $L_f$, the following recurrence relation holds: 
\[
\begin{aligned}
\d( \Theta^{(t+1)}(G, \mathbf{f}), \Theta^{(t+1)}(W, f)  )&  \leq   K^{(t+1)}  \d( \Theta^{(t)}(G, \mathbf{f}), \Theta^{(t)}(W, f)  )  \\&  + D_1^{(t)}   \frac{ \sqrt{\log(N/p) } }{N^{\frac{1}{2(D_\chi+1)}} \dd}  + D_2^{(t)}\varepsilon +  \mathcal{O}\left(\frac{\sqrt{ \log(N) }}{\sqrt{N}}\right)  
\end{aligned}
\]
for $t=0, \ldots, T-1$, where
\begin{equation}
    \label{eq: main lemmaC6}
K^{(t+1)}  =     L_{\Psi^{(t+1)}}  \max\left\{1, L_{\Phi^{(t+1)}} \right\}   .
 \end{equation}
 and $D_1^{(t)}= D(L_{f^{(t)}}, L_{\Phi^{(t+1)}}, L_W, \|W\|_\infty, D_\chi, C_\chi)$, $D_2^{(t)} = D(L_{f^{(t)}}, L_{\Phi^{(t+1)}})$ are the constants from \Cref{cor:main layerwise error}.  \end{lemma}

The recurrence relation from \Cref{lemma:main_rec_rel} can be  solved, which leads to the first main result on uniform convergence.

 \begin{theorem}
\label{thm:main convergence}
Let $W: \chi^2 \to [0,1]$ be an admissible graphon. Let $\dd=\cmin$  if the MPNN uses mean aggregation, and $\dd=1$ if it uses sum aggregation. Then, there exist constants $E,E' > 0$ such that with probability at least $1-4p$: for every MPNN $\Theta$ and Lipschitz continuous function $f: \chi \to \mathbb{R}^F$, if $\{G,\mathbf{f}\} \sim   \{W,f, \varepsilon\}$  such that the number of nodes $N$ in the random graph-signal $G$ satisfies \Cref{eq:main lowerBoundGraphSizeN}, then
\begin{equation*}
\begin{aligned}
  &  \|
    \Theta(G, \mathbf{f}) - \Theta(W, f) 
    \|_\infty^2    \leq 
    E   \frac{ \log(N/p) }{\dd^2 N^{1/(D_\chi + 1)}} + E' \varepsilon  + \mathcal{O}\left(    \frac{ \log(N)  }{\dd^2 N} \right).
\end{aligned}
\end{equation*}  
\end{theorem}
\begin{proof}
 We present a shortened  proof of \Cref{thm:main convergence}, and refer to \Cref{thm:convwithoutpooling} and \Cref{cor:convAfterPooling} in the Appendix for the complete statement and its proof. 
 
 The proof consists of solving the the recurrence relation from \Cref{lemma:main_rec_rel} with \Cref{lemma:RecRecGen}. Note that $ \d \big(\Theta^0(G, \mathbf{f})  ,\Theta^0(W, f) \big) \leq  \varepsilon$ as the noise is sampled from $B_\varepsilon^\infty(\chi^2) \times  B_\varepsilon^\infty(\chi)$. Hence, we get 
for every MPNN $\Theta$  and every Lipschitz continuous $f:\chi \to \mathbb{R}^F$ with Lipschitz constant $L_f$,
\[
 \d \big(\Theta^T(G, \mathbf{f})  ,\Theta^T(W, f) \big) \leq \sum_{t=1}^{T} Q^{(t)} \prod_{t' = t+1}^{T} K^{(t')} + \varepsilon \prod_{t = 1}^{T} K^{(t)},
\]
  where $Q^{(t)}$ and $K^{(t)}$ are defined in \Cref{lemma:main_rec_rel}, respectively. The computation of the exact bound is carried out in \Cref{thm:convwithoutpooling} in the Appendix. 
  To prove the bound after the post pooling layer, i.e., global pooling and the application of a post pooling layer, we use another concentration of measure and the Lipschitz continuity of the post pooling layer. This last step is carried out in \Cref{cor:convAfterPooling} in the in the Appendix. 
\end{proof}

\Cref{thm:main convergence} establishes that the distance between the graph MPNN and the corresponding cMPNN decreases as the number of nodes in randomly sampled graphs increases. However, this result only applies to sampling dense graphs, i.e., $\{G,\mathbf{f}\} \sim_{0} \{W,f, \varepsilon\}$ in terms of \Cref{def:RGM with noise}. To account for sparser graphs, we make the following considerations regarding an admissible graphon $W$ with minimal degree $\cmin$: for $\alpha \geq 0$, we interpret the term $N^{\alpha}W$ as another graphon with a lower bound for its average degree, given by $\int_\chi {N^{-\alpha} W}(x,y) \mu(y) \geq N^{-\alpha} \cmin$. Thus, we obtain the following corollary, which holds for arbitrary sparsity factors $\alpha \geq 0$

\begin{corollary}
\label{cor: main convergence}
Let $W: \chi^2 \to [0,1]$ be a Lipschitz continuous graphon and $\alpha \geq 0$. Then, there exist constants $E,E' > 0$ such that with probability at least $1-4p$: for every MPNN $\Theta$ and Lipschitz continuous function $f: \chi \to \mathbb{R}^F$, if $\{G,\mathbf{f}\} \sim_\alpha  \{W,f, \varepsilon\}$ such that the number of nodes $N$ satisfies \Cref{eq:main lowerBoundGraphSizeN}, then
\begin{equation}
\label{eq: main convergence}
\begin{aligned}
  &  \|
    \Theta(G, \mathbf{f}) - \Theta(W, f) 
    \|_\infty^2    \leq 
    E   \frac{ \log(N/p)N^{2\alpha}  }{N^{1/(D_\chi + 1)}} + E' \varepsilon  + \mathcal{O}\left(    \frac{ \log(N)N^{2\alpha}  }{N} \right).
\end{aligned}
\end{equation}  
\end{corollary}

\Cref{cor: main convergence} serves as a fundamental intermediate result for proving \Cref{thm:main_gen_bound_deformed_graphon}. We demonstrate in \Cref{sec:gen_ana} how the uniform convergence results in \Cref{cor: main convergence} can be utilized to derive generalization bounds for the graph classification setting from \Cref{Graph Classification}.

\textbf{Discussion.}
The exact constants $E$ and $E'$ in \Cref{cor: main convergence} are derived in \Cref{AppendixB}, specifically in \Cref{cor:convAfterPooling} therein.  These constants depend polynomially on the Lipschitz constants of the message and update functions, and the Lipschitz constant of the graphon. The degree of the polynomial is constant in $T$. 
The convergence of graph MPNNs to cMPNNs is limited by the sparsity of the sampled graphs, with a trade-off between convergence speed and sparsity:  the RHS in \Cref{eq: main convergence} only converges to $0$ for $\alpha < \frac{1}{2(D_\chi+1)}$, and converges slower for sparse graphs in this convergence regime.

We lastly remark that
\cite{keriven2020convergence} proved convergence of spectral GNNs for  graphs sampled from RGSMs which are not necessarily dense. Their convergence result is not uniform in the choice of the GNN, and it is not clear how it can be generalized to such a result. 

\section{Generalization Analysis of MPNNs on Mixture of Graphons}
\label{sec:gen_ana}
In this section, we state the main results of our work, which provide generalization bounds in graph classification tasks.

We present a corollary that extends \Cref{thm:main convergence} by considering graphs of arbitrary sizes and reformulating the result to hold in expectation instead of high probability. As demonstrated in the following corollary, these steps only introduce a factor that decreases exponentially with respect to the number of nodes of the sampled graph.

\begin{corollary}
\label{cor:main_unifExpValue}
Let $(\chi,d, \P) $ be a metric-probability space and
$W$ be a Lipschitz continuous graphon.   Consider a graph-signal $\{G,\mathbf{f}\} \sim_\alpha \{W,f,\varepsilon \}$ with $N$ nodes. Then,
for every $f:\chi \to \mathbb{R}^{F}$ with Lipschitz constant $L_f$, 
\[
\begin{aligned}
&  \E_{X_1, \ldots, X_N \sim \mu^N} \left[\sup_{\Theta \in \mathrm{Lip}_{L,B}} \left\| 
\Theta(G, \mathbf{f}) - 
\Theta(W, f) \right\|_\infty^2 \right] \\&  \leq  4(1 + \sqrt{\pi}) \Bigg(  T_1 \frac{\big(1+\log(N)\big)N^{2\alpha}}{N^{\frac{1}{D_\chi+1} }} + T_2 \frac{\big(1+\log(N)\big)N^{2 \alpha}}{N}
 +  T_3 \varepsilon \Bigg) +  \mathcal{O} \left(   \exp(-N)\right) .
\end{aligned}
\]
where the constants $T_1$, $T_2$ and $T_3$ are defined in \Cref{eq:constantsS1toS4} in the Appendix. 
\end{corollary}

The proof of \Cref{cor:main_unifExpValue} is derived by first bounding the worst-case error between the graph MPNN and the cMPNN. The result from \Cref{thm:main convergence} is then applied to all possible values of $p \in (0,1)$ and a series of Gaussians is obtained, which can be bounded using standard methods. For further details, see \Cref{thm:unifExpValue} and its proof. With this foundation, we can now state and prove the following theorem.

\begin{theorem}
There exist constants $C, C'>0$ such that 
\begin{equation*}
\begin{aligned}   &   \E_{\mathcal{T}\sim \mu^m}\left[\sup_{\Theta \in \mathrm{Lip}_{L,B}} \Big(R_{emp}(\Theta)   - R_{exp}(\Theta) \Big)^2  \right]   \leq   \frac{2^\Gamma8\|\mathcal{L}\|_\infty^2\pi}{m} 
\\ &   + \frac{2^\Gamma L_{\mathcal{L}}^2   }{m}  \left( C \cdot \E_{N \sim \nu} \left[ \frac{1+\log(N)}{N}N^{2\alpha} + \frac{1 + \log(N)}{N^{1/(D_{\chi} + 1)}}N^{2\alpha} + \mathcal{O}\left( \exp(-N) \right) \right] + C'\varepsilon\right),
\end{aligned}
\end{equation*}
where $C$ and $C'$ are specified in \Cref{appendix: Proof Gen Bound} in the Appendix. 
\end{theorem}
    \begin{proof}
Given $\mathbf{m}=\{m_1,\ldots,m_{\Gamma}\}$, $\sum_{j=1}^{\Gamma}m_j=m$, and $\mathcal{G}^{\mathbf{m}}$ as the space of datasets with $m_j$ samples from each class $j=1,\ldots,\Gamma$. We represent the conditional choice of the dataset on the choice of $\mathbf{m}$ by $\mathcal{T}_{\mathbf{m}}:=\{\{G_i^j, \mathbf{f}_i^j\}_{i=1}^{m_j}\}_{j=1}^{\Gamma} \sim \mu_{\mathcal{G}^{\mathbf{m}}}$. Define $\mathcal{M}_k$ as the set of all $\mathbf{m}$ with $2\sqrt{m} k \leq \sum_{j=1}^\Gamma |m_j - m\gamma_j| <2\sqrt{m} (k+1)$.

Note that $\{m_1, \ldots, m_{\Gamma}\}$ is an i.i.d. multinomial random variable with parameters $m$ and $\{\gamma_1, \ldots, \gamma_\Gamma\}$. By the Breteganolle-Huber-Carol inequality (see \Cref{lemma:BHCineq}), we have 
    $\mathbb{P}\left(\mathbf{m} \in \mathcal{M}_k\right) \leq 2^\Gamma \exp(-2k^2) $ for any $k>0$. Thus,  we decompose the expected generalization error into series of Gaussians, \begin{equation}
    \label{eq:thmC13-1-2}
    \begin{aligned}
    & \E_{\mathcal{T} \sim \mu^m} \left[ \sup_{\Theta \in \mathrm{Lip}_{L,B}} \left( \frac{1}{m} \sum_{i=1}^m  \mathcal{L}(\Theta(G_i, \mathbf{f}_i), y_i) - \E_{(G, \mathbf{f},y)\sim \mu}\left[ \mathcal{L}(\Theta(G, \mathbf{f}), y) \right] \right)^2 \right] \\
& \leq \sum_{k} \mathbb{P}\big(
\mathbf{m}\in \mathcal{M}_k
\big) \times 
\sup_{\mathbf{m}\in \mathcal{M}_k} \E_{\mathcal{T}_{\mathbf{m}} \sim \mu_{\mathcal{G}^{\mathbf{m}}}} \left[\sup_{\Theta \in \mathrm{Lip}_{L,B}}\left(\sum_{j=1}^\Gamma \left( \frac{1}{m}\sum_{i=1}^{m_j} \mathcal{L}(\Theta(G_i^j,\mathbf{f}_i^j), y_j) \right.\right.\right.\\  & \quad \ \ \quad \ \ \quad \ \ \quad \ \ \quad \ \ \quad \ \ \quad \ \ \quad \ \ \quad \ \ \quad \ \ \quad \ \  \left.\left.\left. - \frac{1}{m} \sum_{i=1}^{m \gamma_j}\E_{(G^j, \mathbf{f}^j )\sim \mu_{\mathcal{G}_j}}\left[ \mathcal{L}(\Theta(G^j,\mathbf{f}^j), y_j) \right] \right) \right)^2 \right] \\
&  \leq  \sum_k \mathbb{P}\big(
\mathbf{m}\in \mathcal{M}_k
\big) \times  \E_{\mathcal{T}_{\mathbf{m}} \sim  \mu_{\mathcal{G}^{\mathbf{m}}}} \left[\sup_{\Theta \in \mathrm{Lip}_{L,B}}2 \left(\sum_{j=1}^\Gamma \left( \frac{1}{m}\sum_{i=1}^{m \gamma_j} \mathcal{L}(\Theta(G_i^j,\mathbf{f}_i^j),y_j) \right.\right.\right.\\  & \quad \ \ \quad \ \     \quad \ \ \quad   \left.\left.\left. - \frac{1}{m} \sum_{i=1}^{m \gamma_j}\E_{(G^j, \mathbf{f}^j )\sim \mu_{\mathcal{G}_j}}\left[ \mathcal{L}(\Theta(G^j,\mathbf{f}^j), y_j) \right] \right) \right)^2 \right]
\\ &  +  \sum_k \mathbb{P}\big(
\mathbf{m}\in \mathcal{M}_k
\big) \times  \E_{\mathcal{T}_{\mathbf{m}} \sim  \mu_{\mathcal{G}^{\mathbf{m}}}} \left[2
\left(\sum_{j=1}^\Gamma \left(
\frac{1}{m} |m\gamma_j - m_j| \|\mathcal{L}\|_\infty
\right)\right)^2
\right].
 \end{aligned}
\end{equation}
The second term on the RHS of \Cref{eq:thmC13-1-2} can be easily bounded by $2^\Gamma \frac{8\|\mathcal{L}\|_\infty^2}{m}\pi$. For the first term, we observe that it can bounded by the variance of the loss and subsequently by the expected difference between the output of the graph MPNN and the corresponding cMPNN,
\begin{equation}
\begin{aligned}
 & \E_{\mathcal{T}_{\mathbf{m}} \sim \mu_{\mathcal{G}^{\mathbf{m}}}} \left[\sup_{\Theta \in \mathrm{Lip}_{L,B}}\left(\sum_{j=1}^\Gamma \left( \frac{1}{m}\sum_{i=1}^{m \gamma_j} \mathcal{L}(\Theta(G_i^j,\mathbf{f}_i^j),y_j)  \right.\right.\right.\\ 
   & \quad \ \ \quad \ \     \quad \ \ \quad   \left.\left.\left. - \frac{1}{m} \sum_{i=1}^{m \gamma_j}\E_{(G^j, \mathbf{f}^j )\sim \mu_{\mathcal{G}_j}}\left[ \mathcal{L}(\Theta_{G^j}(\mathbf{f}^j), y_j) \right] \right) \right)^2 \right]
\\
= & \Gamma\sum_{j=1}^\Gamma \frac{\gamma_j}{m} \Var_{(G^j, \mathbf{f}^j) \sim \mu_{\mathcal{G}_j}}\left[\sup_{\Theta \in \mathrm{Lip}_{L,B}} \mathcal{L}(\Theta(G^j, \mathbf{f}^j),y_j) \right]
\\ 
\leq & \Gamma\sum_{j=1}^\Gamma \frac{\gamma_j}{m} \E_{(G^j, \mathbf{f}^j) \sim \mu_{\mathcal{G}_j}}\left[\sup_{\Theta \in \mathrm{Lip}_{L,B}} L_\mathcal{L}^2\| \Theta(G^j, \mathbf{f}^j) - \Theta (W^j, f^j)\|_\infty^2 \right].
\end{aligned}
\end{equation}
We can now apply \Cref{cor:main_unifExpValue} to get
\begin{equation*}
\begin{aligned}
\leq \Gamma & \sum_{j=1}^\Gamma \frac{\gamma_j}{m} L_\mathcal{L}^2 \Bigg(4(1 + \sqrt{\pi}) \Bigg(  T_1 \frac{\big(1+\log(N)\big)N^{2\alpha}}{N^{\frac{1}{D_\chi+1} }} + T_2 \frac{\big(1+\log(N)\big) N^{2 \alpha}}{N}
 + T_3 \varepsilon \Bigg) \\ &+  \mathcal{O} \left(   \exp(-N)\right)  \Bigg).
\end{aligned}
\end{equation*}
Now, using $\mathbb{P}\left(\mathbf{m} \in \mathcal{M}_k\right) \leq 2^\Gamma \exp(-2k^2)$ and assembling the constants finishes the proof.
\end{proof}

\paragraph{Discussion and Comparison with \citep{maskey2022generalization}}
We recall that \citet{maskey2022generalization} provide generalization bounds for dense weighted graphs sampled from RGSMs without considering possible noise and random edges. If we ignore noise in our setting, i.e., set $\varepsilon=0$ and consider only dense graphs, i.e., $\alpha=0$, we can compare our generalization bound, as detailed in \Cref{thm:main_gen_bound_deformed_graphon}, with that of \citet[Theorem 3.3]{maskey2022generalization}.  We observe that considering random edges slightly worsens the generalization bound with respect to the average graph size by an additional term of $\mathcal{O}\left(\E_{N \sim \nu}\left[ \frac{\log(N)}{N}\right] \right)$. However, the constants $C$ and $C'$ follow the same asymptotics as the constants in the generalization bound in  \citep{maskey2022generalization}. Therefore, the asymptotics of both bounds are equivalent except for a $\log(N)$ factor which is negligible for moderately large graphs. This is backed by our numerical experiments. 

\pgfplotsset{every tick label/.append style={font=\tiny}}
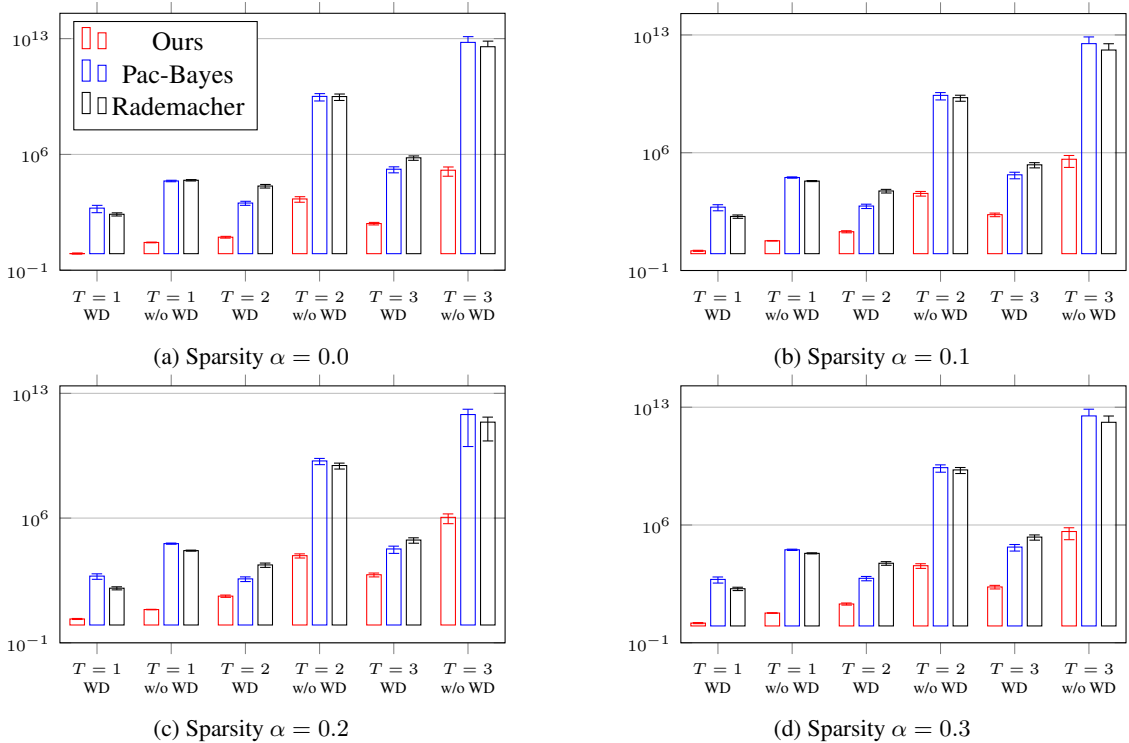
\begin{figure}[t]
  \begin{subfigure}[b]{0.5\linewidth}
    \centering
    {\begin{tikzpicture}
\begin{semilogyaxis}[
ymin=.1,
    ymajorgrids=true,
    legend pos=north west,
    ybar,
    bar width=5.5pt,
    width=7.5cm,
    height=5cm,
    cycle list name=color list,
    xtick=data,
xticklabels={
        $T=1$ \\ WD, \(T=1\) \\ w/o WD, 
        \(T=2\) \\ WD, \(T=2\) \\ w/o WD, 
        \(T=3\) \\ WD, \(T=3\) \\ w/o WD
      },
  x tick label style={align=center},
    error bars/y dir=both,
    error bars/y explicit,
]

\addplot+[error bars/.cd, y dir=both, y explicit] 
table [x=Layers, y=ours, y error=ours_std, col sep=comma] {mean_00.csv};

\addplot+[error bars/.cd, y dir=both, y explicit] 
table [x=Layers, y=pac_bayes, y error=pac_bayes_std, col sep=comma] {mean_00.csv};

\addplot+[error bars/.cd, y dir=both, y explicit] 
table [x=Layers, y=rademacher, y error=rademacher_std, col sep=comma] {mean_00.csv};

\legend{Ours, Pac-Bayes, Rademacher}
\end{semilogyaxis}
\end{tikzpicture}}
\caption{Sparsity $\alpha=0.0$} 
    \label{fig:mean-1}
\end{subfigure}\begin{subfigure}[b]{0.5\linewidth}
    \centering
    \begin{tikzpicture}
\begin{semilogyaxis}[
ymin=.1,
    ymajorgrids=true,
    legend pos=north west,
    ybar,
    bar width=5.5pt,
    width=7.5cm,
    height=5cm,
    cycle list name=color list,
    xtick=data,
    xticklabels={         $T=1$ \\ WD, \(T=1\) \\ w/o WD,          \(T=2\) \\ WD, \(T=2\) \\ w/o WD,          \(T=3\) \\ WD, \(T=3\) \\ w/o WD       },   x tick label style={align=center},
error bars/y dir=both,
    error bars/y explicit,
]

\addplot+[error bars/.cd, y dir=both, y explicit] 
table [x=Layers, y=ours, y error=ours_std, col sep=comma] {mean_01.csv};

\addplot+[error bars/.cd, y dir=both, y explicit] 
table [x=Layers, y=pac_bayes, y error=pac_bayes_std, col sep=comma] {mean_01.csv};

\addplot+[error bars/.cd, y dir=both, y explicit] 
table [x=Layers, y=rademacher, y error=rademacher_std, col sep=comma] {mean_01.csv};

\end{semilogyaxis}
\end{tikzpicture}
    \caption{Sparsity $\alpha=0.1$} 
    \label{fig:mean-2}
\end{subfigure} 
  \begin{subfigure}[b]{0.5\linewidth}
    \centering
    \begin{tikzpicture}
\begin{semilogyaxis}[
ymin=.1,
    ymajorgrids=true,
    legend pos=north west,
    ybar,
    bar width=5.5pt,
    width=7.5cm,
    height=5cm,
    cycle list name=color list,
    xtick=data,
xticklabels={         $T=1$ \\ WD, \(T=1\) \\ w/o WD,          \(T=2\) \\ WD, \(T=2\) \\ w/o WD,          \(T=3\) \\ WD, \(T=3\) \\ w/o WD       },   x tick label style={align=center},
    error bars/y dir=both,
    error bars/y explicit,
]

\addplot+[error bars/.cd, y dir=both, y explicit] 
table [x=Layers, y=ours, y error=ours_std, col sep=comma] {mean_02.csv};

\addplot+[error bars/.cd, y dir=both, y explicit] 
table [x=Layers, y=pac_bayes, y error=pac_bayes_std, col sep=comma] {mean_02.csv};

\addplot+[error bars/.cd, y dir=both, y explicit] 
table [x=Layers, y=rademacher, y error=rademacher_std, col sep=comma] {mean_02.csv};

\end{semilogyaxis}
\end{tikzpicture}
    \caption{Sparsity $\alpha=0.2$} 
    \label{fig:mean-3} 
  \end{subfigure}\begin{subfigure}[b]{0.5\linewidth}
    \centering
    \begin{tikzpicture}
\begin{semilogyaxis}[
ymin=.1,
    ymajorgrids=true,
    legend pos=north west,
    ybar,
    bar width=5.5pt,
    width=7.5cm,
    height=5cm,
    cycle list name=color list,
    xtick=data,
xticklabels={         $T=1$ \\ WD, \(T=1\) \\ w/o WD,          \(T=2\) \\ WD, \(T=2\) \\ w/o WD,          \(T=3\) \\ WD, \(T=3\) \\ w/o WD       },   x tick label style={align=center},
    error bars/y dir=both,
    error bars/y explicit,
]

\addplot+[error bars/.cd, y dir=both, y explicit] 
table [x=Layers, y=ours, y error=ours_std, col sep=comma] {mean_03.csv};

\addplot+[error bars/.cd, y dir=both, y explicit] 
table [x=Layers, y=pac_bayes, y error=pac_bayes_std, col sep=comma] {mean_03.csv};

\addplot+[error bars/.cd, y dir=both, y explicit] 
table [x=Layers, y=rademacher, y error=rademacher_std, col sep=comma] {mean_03.csv};
\end{semilogyaxis}
\end{tikzpicture}
    \caption{Sparsity $\alpha=0.3$} 
    \label{fig:mean-4} 
  \end{subfigure} 
  \caption{Comparison of Generalization Bounds for GraphSage with mean aggregation: Our Theoretical Analysis vs. PAC-Bayesian (Liao et al., 2021) and Rademacher Complexity (Garg et al., 2020) for Binary Classification Using Erdös-Rényi and SBM Graphs. Each subplot corresponds to different sparsity levels $\alpha \in \{0,0.1,0.2,0.3\}$ of the underlying RGSM. For each subplot, we test six different training conditions: $T=1$ with weight decay (WD), $T=1$ without weight decay (w/o WD), $T=2$ with WD, $T=2$ w/o WD, $T=3$ with WD, and $T=3$ w/o WD.}
  \label{fig:mean} 
\end{figure}

\section{Experiments}
\label{sec:experiments}
We evaluate our proposed generalization bounds by examining two classical RGSMs: the Erdős-Rényi model (ERM) and a two-class Stochastic Block Model (SBM), relaxed to be a continuous graphon.

For our experimental framework, we generate four different synthetic datasets each consisting of  100,000 random graphs with 50 nodes and varying sparsity. 
More precisely, for each sparsity $\alpha \in \{ 0, 0.1, 0.2, 0.3 \}$, 
we generate equally many graphs for each of the two considered RGSMs: the ERM, based on the graphon $W_1(x,y)=0.4\cdot 50^{-\alpha}$, and a relaxed SBM that is based on the graphon $W_2(x,y) =  \left(\frac{\sin(2\pi x)\sin(2\pi y)}{2 \pi} + 0.25\right)\cdot 50^{-\alpha}$ over the unit square $[0,1]^2$.

For the MPNN architecture, we consider GraphSAGE \citep{hamilton2017inductive} with both mean and normalized sum aggregation. We investigate the impact of varying architectural depths ($T=1,2,3$) on the model's performance. 
Our analysis requires bounding the Lipschitz constants of the message-passing and update functions. To evaluate the effect of these Lipschitz constant, we consider two training approaches that lead to different Lipschitz bounds. First, we apply weight decay regularization, which serves to lower the Lipschitz constants, improving the model's generalization capability. Second, train the model without any regularization. For each training configuration—defined by the number of layers and the presence or absence of regularization—we train the MPNN and then measure the resulting Lipschitz constants of the network. For benchmarking purposes, we also calculate two alternative generalization bounds: one based on PAC-Bayes theory \citep{liao2021a} and another based on Rademacher complexity \citep{pmlr-v119-garg20c}. We refer to \Cref{sec:appendix_exp} for more details on the dataset, model and training setup. 

Our results,  depicted in \Cref{fig:mean} and \Cref{fig:norm_sum}, indicate that the generalization bounds from our analysis are significantly tighter than the PAC-Bayes and Rademacher bound. The experiments highlight that our theoretical framework does not only provide insight into the asymptotic generalization behavior of MPNNs,  but also gives concrete, numerical bounds that validate the model's generalization capability in practical contexts. Notably, with a one-layer MPNN configuration, our theory guarantees a generalization gap below $1$.

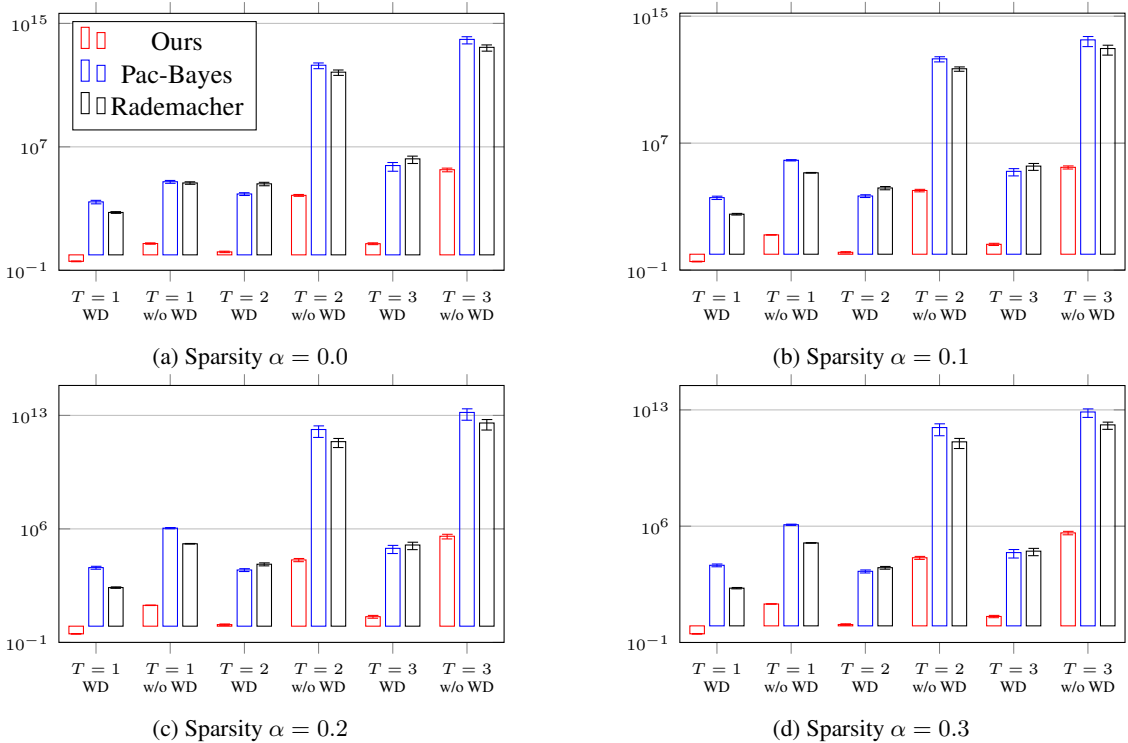
\begin{figure}
  \begin{subfigure}[b]{0.5\linewidth}
    \centering
\begin{tikzpicture}
\begin{semilogyaxis}[
ymin=.1,
    ymajorgrids=true,
    legend pos=north west,
    ybar,
    bar width=5.5pt,
    width=7.5cm,
    height=5cm,
    cycle list name=color list,
    xtick=data,
xticklabels={         $T=1$ \\ WD, \(T=1\) \\ w/o WD,          \(T=2\) \\ WD, \(T=2\) \\ w/o WD,          \(T=3\) \\ WD, \(T=3\) \\ w/o WD       },   x tick label style={align=center},
    error bars/y dir=both,
    error bars/y explicit,
]

\addplot+[error bars/.cd, y dir=both, y explicit] 
table [x=Layers, y=ours, y error=ours_std, col sep=comma] {sum_00.csv};

\addplot+[error bars/.cd, y dir=both, y explicit] 
table [x=Layers, y=pac_bayes, y error=pac_bayes_std, col sep=comma] {sum_00.csv};

\addplot+[error bars/.cd, y dir=both, y explicit] 
table [x=Layers, y=rademacher, y error=rademacher_std, col sep=comma] {sum_00.csv};

\legend{Ours, Pac-Bayes, Rademacher}
\end{semilogyaxis}
\end{tikzpicture}
\caption{Sparsity $\alpha=0.0$} 
    \label{fig:norm_sum-1}
\end{subfigure}\begin{subfigure}[b]{0.5\linewidth}
    \centering
    \begin{tikzpicture}
\begin{semilogyaxis}[
ymin=.1,
    ymajorgrids=true,
    legend pos=north west,
    ybar,
    bar width=5.5pt,
    width=7.5cm,
    height=5cm,
    cycle list name=color list,
    xtick=data,
xticklabels={         $T=1$ \\ WD, \(T=1\) \\ w/o WD,          \(T=2\) \\ WD, \(T=2\) \\ w/o WD,          \(T=3\) \\ WD, \(T=3\) \\ w/o WD       },   x tick label style={align=center},
    error bars/y dir=both,
    error bars/y explicit,
]

\addplot+[error bars/.cd, y dir=both, y explicit] 
table [x=Layers, y=ours, y error=ours_std, col sep=comma] {sum_01.csv};

\addplot+[error bars/.cd, y dir=both, y explicit] 
table [x=Layers, y=pac_bayes, y error=pac_bayes_std, col sep=comma] {sum_01.csv};

\addplot+[error bars/.cd, y dir=both, y explicit] 
table [x=Layers, y=rademacher, y error=rademacher_std, col sep=comma] {sum_01.csv};

\end{semilogyaxis}
\end{tikzpicture}
    \caption{Sparsity $\alpha=0.1$} 
    \label{fig:norm_sum-2}
\end{subfigure} 
  \begin{subfigure}[b]{0.5\linewidth}
    \centering
    \begin{tikzpicture}
\begin{semilogyaxis}[
ymin=.1,
    ymajorgrids=true,
    legend pos=north west,
    ybar,
    bar width=5.5pt,
    width=7.5cm,
    height=5cm,
    cycle list name=color list,
    xtick=data,
xticklabels={         $T=1$ \\ WD, \(T=1\) \\ w/o WD,          \(T=2\) \\ WD, \(T=2\) \\ w/o WD,          \(T=3\) \\ WD, \(T=3\) \\ w/o WD       },   x tick label style={align=center},
    error bars/y dir=both,
    error bars/y explicit,
]

\addplot+[error bars/.cd, y dir=both, y explicit] 
table [x=Layers, y=ours, y error=ours_std, col sep=comma] {sum_02.csv};

\addplot+[error bars/.cd, y dir=both, y explicit] 
table [x=Layers, y=pac_bayes, y error=pac_bayes_std, col sep=comma] {sum_02.csv};

\addplot+[error bars/.cd, y dir=both, y explicit] 
table [x=Layers, y=rademacher, y error=rademacher_std, col sep=comma] {sum_02.csv};

\end{semilogyaxis}
\end{tikzpicture}
    \caption{Sparsity $\alpha=0.2$} 
    \label{fig:norm_sum-3} 
  \end{subfigure}\begin{subfigure}[b]{0.5\linewidth}
    \centering
    \begin{tikzpicture}
\begin{semilogyaxis}[
ymin=.1,
    ymajorgrids=true,
    legend pos=north west,
    ybar,
    bar width=5.5pt,
    width=7.5cm,
    height=5cm,
    cycle list name=color list,
    xtick=data,
xticklabels={         $T=1$ \\ WD, \(T=1\) \\ w/o WD,          \(T=2\) \\ WD, \(T=2\) \\ w/o WD,          \(T=3\) \\ WD, \(T=3\) \\ w/o WD       },   x tick label style={align=center},
    error bars/y dir=both,
    error bars/y explicit,
]

\addplot+[error bars/.cd, y dir=both, y explicit] 
table [x=Layers, y=ours, y error=ours_std, col sep=comma] {sum_03.csv};

\addplot+[error bars/.cd, y dir=both, y explicit] 
table [x=Layers, y=pac_bayes, y error=pac_bayes_std, col sep=comma] {sum_03.csv};

\addplot+[error bars/.cd, y dir=both, y explicit] 
table [x=Layers, y=rademacher, y error=rademacher_std, col sep=comma] {sum_03.csv};
\end{semilogyaxis}
\end{tikzpicture}
    \caption{Sparsity $\alpha=0.3$} 
    \label{fig:norm_sum-4} 
  \end{subfigure} 
  \caption{Comparison of Generalization Bounds for GraphSage with normalized sum aggregation. See caption of \Cref{fig:mean} for more details.}
  \label{fig:norm_sum} 
\end{figure}

\section{Conclusion}
In this work, we derived a novel generalization bound for MPNNs on graph-signals sampled from mixture of graphon models.  
The bound decreases as the average number of nodes in the graphs increases.
Our findings hence suggest that MPNNs can still generalize effectively, even when their complexity (in terms of number of layers, Lipschitz constants of the message functions, etc.) exceeds the size of the training set, provided that the graphs in the dataset are sufficiently large. 
Our work extends previous works \citep{maskey2022generalization} by considering a more realistic setting, specifically by extending the analysis to simple, sparse, and noisy graphs with Bernoulli-distributed edges. 
Unlike past works, our generalization bounds are not vacuous: the bound in some simple settings is lower than $1$.

While our findings are promising, we acknowledge the limitations of our theory. The data generation model we used is based on a finite set of Lipschitz continuous graphons, which may not fully capture the diversity of some real-world graph dataset. Moreover, we focused on MPNNs with mean or normalized sum aggregation. Other schemes, like max aggregation, may behave differently and requires different proof techniques. 
Potential directions for future work include developing a generalization theory for max aggregation MPNNs,
and considering more sophisticated models of sparse graphs, e.g., graphops \citep{backhausz2018action}.

\section*{Acknowledgments}
S. M. acknowledges partial support by the NSF-Simons Research Collaboration on the Mathematical and Scientific Foundations of Deep Learning (MoDL) (NSF DMS 2031985), by DFG SPP 1798 (KU 1446/27-2) and by the BMBF-project 05M20 MaGriDo (Mathematics for Machine Learning Methods for Graph-Based Data with Integrated Domain Knowledge).

G. K. acknowledges partial support by the DAAD programme Konrad Zuse Schools of Excellence in Artificial Intelligence, sponsored by the Federal Ministry of Education and Research. G. Kutyniok also acknowledges support from the Munich Center for Machine Learning (MCML) as well as the German Research Foundation under Grants DFG-SPP-2298, KU 1446/31-1 and KU 1446/32-1 and under Grant DFG-SFB/TR 109 and Project C09.

R. L. acknowledges partial support by ISF (Israel Science Foundation) grant \#1937/23 (Analysis of Graph Deep Learning Using Graphon Theory). 

\printbibliography

\newpage
\appendix

\section{Outline of the Appendix}

In \Cref{appendix:prep}, we introduce notations for the remainder of the appendix. In \Cref{AppendixB}, we study the uniform convergence and outline the proof of \Cref{cor: main convergence}. Finally, in \Cref{appendix: Proof Gen Bound}, we derive the proof of \Cref{thm:main_gen_bound_deformed_graphon}.

\section{Basic Definitions}
\label{appendix:prep}
We consider metric spaces $(\chi,d)$, where $\chi$ is a set and $d:\chi\times\chi\rightarrow \left[0,\infty\right)$ is a metric. We denote by $B_{\varepsilon}(x) = \{y\in \chi\ | \  d(x,y)<\varepsilon\}$   the ball around $x\in\chi$ with radius $\varepsilon>0$.
Unless stated otherwise, we denote graphs by $G=\{V,E\}$, and their corresponding adjacency matrices by $\mathbf{A}$.

\subsection{Sampled Graphs, Degree, and Aggregation}

Given a graph $G=(V,E)$ with $N$ nodes, we often identify each node $i$ by some  point $X_i\in\chi$, for every $i=1,\ldots,N$. 
Given a graphon $W:\chi^2\rightarrow [0,1]$, and sample points $\mathbf{X}= \{X_1, \ldots, X_N\}\in\chi^N$, the corresponding sampled graph $G$ with adjacency matrix $\mathbf{A}=\{a_{i,j}\}_{i,j=1}^N$ is defined as follows:  in case $G$ is weighted, we define $a_{i,j}=W(X_i,X_j)$, and in case $G$ is simple we define $a_{i,j}$ as a Bernoulli variable with probability $W(X_i,X_j)$ for $a_{i,j}=1$. We also denote $\mathbf{A}(X_i,X_j)=a_{i,j}$.
For a signal $\mathbf{f}:V\rightarrow\mathbb{R}^F$, 
 we denote, by abuse of notation,  $\mathbf{f}(X_i) = \mathbf{f}_i$ for $i= 1,\ldots, N$.

\begin{definition}
For samples $\mathbf{X}=\{X_1, \ldots, X_N\}\in \chi^N$, we define the sampling operator $S^X$, for every  metric-space signal $f: \chi \rightarrow \mathbb{R}^F$ by 
\[S^Xf = \big\{f(X_i) \big\}_{i=1}^N \in \mathbb{R}^{N \times F}. \]
\end{definition}
Note that the sampling operator is well defined over $L^p(\chi)$ if the sample points $\mathbf{X}$ are random. Indeed, while the evaluation of $L^p(\chi)$ functions at deterministic points is not well defined, random points are themselves functions (random variables), and  sampling becomes composition, which  is well defined in $L^p(\chi)$.

Next, we define various notions of degree.

\begin{definition}
\label{def:degrees}
Let $(\chi,d, \mu)$ be a metric-probability space. 
Let $W:\chi \times \chi \to \left[0,\infty\right)$ be a graphon, $\mathbf{X}= \{X_1, \ldots, X_N\}\in\chi^N$ sample points, and $G$ the corresponding sampled graph with adjacency matrix $\mathbf{A}=\{a_{i,j}\}_{i,j=1}^N$ (simple or weighted).
\begin{enumerate}
    \item We define the \emph{graphon degree} of $W$ at $x\in\chi$ by
\begin{equation}
\label{eq:d_W}
 \mathrm{d}_W(x) = \int_{\chi} W (x,y ) d\P(y).    
\end{equation}
\item Given a point $x\in\chi$ that need not be in $\mathbf{X}$, we define the \emph{graph-graphon degree} of $\mathbf{X}$  at $x$ as the random variable
\begin{equation}
    \label{eq:d_X}
    \mathrm{d}_{\mathbf{X}}(x) = \frac{1}{N} \sum_{i=1}^N W(x, X_i).
\end{equation}
\item 
The \emph{normalized degree} of $G$ at the node $X_c\in \mathbf{X}$ is defined as 
\begin{equation}
    \label{eq:d_G}
    \mathrm{d}_{\mathbf{A}} (X_c) = \frac{1}{N} \sum_{i=1}^N \mathbf{A}(X_c, X_i).
\end{equation}
\end{enumerate}
\end{definition}

Based on the different version of degrees in \Cref{def:degrees}, we define three corresponding 
 versions of mean aggregation.

\begin{definition}
\label{def:contMeanAgg}
Let $(\chi,d, \mu)$ be a metric-probability space. Given a graphon $W:\chi \times \chi \to [0,1]$, we define the \emph{continuous mean aggregation} of the metric space message kernel $U:\chi\times\chi\rightarrow\mathbb{R}^F$ by 
\[ M_WU:\chi\rightarrow \mathbb{R}^F  , \quad
M_W U = \int_\chi
\frac{W(\cdot, y)}{ \mathrm{d}_W(\cdot) }  U(\cdot, y) d\P(y).
\]
\end{definition}
In \Cref{def:contMeanAgg}, $U(x,y)$ represents a message sent from the point $y$ to the point $x$ in the metric space $\chi$. Let $F', H \in \mathbb{N}$. Given a metric-space signal $f:\chi\rightarrow\mathbb{R}^{F'}$ and a message function $\Phi: \mathbb{R}^{2F'} \to \mathbb{R}^H$, we denote 
\[\Phi (f,f):\chi^2\rightarrow \mathbb{R}^{F}, \quad  (x,y)\mapsto \Phi(f(x),f(y)).\]
This leads to
\[
M_W \Phi (f,f) = \int_\chi
\frac{W(\cdot, y)}{ \mathrm{d}_W(\cdot) }  \Phi\big(f(\cdot), f(y)\big) d\P(y).
\]

\begin{definition}
\label{def:graphkernelMeanAgg}
Let $W$ be a graphon and $\mathbf{X}= \{X_1, \ldots, X_N\}\in\chi^N$ sample points.
For a metric-space message kernel  $U:\chi\times \chi \rightarrow \mathbb{R}^F$,
we define the \emph{graph-kernel mean aggregation} by 
\[ M_{\mathbf{X}} U:\chi\rightarrow \mathbb{R}^F  , \quad
M_{\mathbf{X}} U = \frac{1}{N} \sum_j
\frac{W(\cdot, X_j)}{ \mathrm{d}_X(\cdot) }  U(\cdot, X_j).
\]
\end{definition}

Note that in the definition of $M_{\mathbf{X}}$, messages are sent from graph nodes to arbitrary points in the metric space. Hence, $M_{\mathbf{X}} U: \chi \rightarrow \mathbb{R}^F$ is a metric space signal.

\begin{definition}
Let $G$ be a simple graph with nodes $\mathbf{X} =  \{X_1, \ldots, X_N\}\in\chi^N$ and adjacency matrix $\mathbf{A}$.
For  a graph message kernel $\mathbf{U}:\mathbf{X}\times \mathbf{X} \rightarrow \mathbb{R}^F$, we define the \emph{mean aggregation}  as 
\[ M_{\mathbf{A}} \mathbf{U}:\mathbf{X}\rightarrow\mathbb{R}^F, \quad 
(M_{\mathbf{A}}  \mathbf{U})(X_i) =  \frac{1}{N} \sum_j
\frac{A(X_i, X_j)}{ \mathrm{d}_{\mathbf{A}} (X_i) }  \mathbf{U}(X_i, X_j). 
\] 
\end{definition}

Note that for a graph message kernel,  $\mathbf{U}(X_i,X_j)$ represents a message sent from the node $X_j$ to the node $X_i$.
Note moreover that $M_A \mathbf{U}:\mathbf{X}\rightarrow \mathbb{R}^F$ is a signal.  
Given a signal $\mathbf{f}:\mathbf{X} \to \mathbb{R}^F$, which can be written as $\mathbf{f} = \{\mathbf{f}_i\}_i$, and a message function $\Phi:\mathbb{R}^{2F} \rightarrow \mathbb{R}^{H}$, we denote
\[
\Phi( \mathbf{f}, \mathbf{f} ) := \big(\Phi( \mathbf{f}_i, \mathbf{f}_j)\big)_{i,j=1}^N.
\]
Hence, given the message kernel  $\mathbf{U}(X_i,X_j)=\Phi(\mathbf{f}(X_i), \mathbf{f}(X_j))$, we have
\[
M_A \mathbf{U} = M_A \Phi(\mathbf{f},\mathbf{f}) = \frac{1}{N} \sum_{j=1}^N
\frac{\mathbf{A}(\cdot, X_j)}{ \mathrm{d}_{\mathbf{A}}(\cdot) }  \Phi\big(\mathbf{f}(\cdot), \mathbf{f}(X_j)\big).
\]

\subsection{Norms, Distances and Lipschitz Continuity}

Next, we define the different norms used in our analysis.

\begin{definition}
$ $
\begin{enumerate}
\item For a vector $\mathbf{z}=(z_1,\ldots,z_F) \in \mathbb{R}^F$, we define as usual
\[
\|\mathbf{z}\|_\infty = \max_{ 1 \leq k \leq F } |z_k|.
\]
    \item 
For a function $g : \chi \to \mathbb{R}^F$,  we define
\[
\|g\|_\infty  = \max_{ 1 \leq k \leq F } \sup_{x \in \chi} \big| \big(g(x)\big)_k \big|,
\]
\item Given a graph with $N$ nodes, we define the norm $\| \mathbf{f} \|_{\infty;\infty}$ of graph feature maps $\mathbf{f}=(\mathbf{f}_1,\ldots,\mathbf{f}_N) \in \mathbb{R}^{N \times F}$, with feature dimension $F$, as
\[
\|\mathbf{f}\|_{\infty; \infty} = \max_{i=1, \ldots, N} \|\mathbf{f}_i\|_{\infty}.
\]
\end{enumerate}
\end{definition}

For a metric-space signal $f:\chi \to \mathbb{R}^F$ and a signal $\mathbf{f} \in \mathbb{R}^{N \times F}$, we define the distance ${\rm dist}$  as
\begin{equation}
\label{eq:distGraphMetric}
    \d(f, \mathbf{f} ) =  \|\mathbf{f} -  (S^Xf) \|_{\infty;\infty}.
\end{equation}

Let $(\mathcal{Y}, d_\mathcal{Y})$ be a metric space and consider $g: \mathcal{Y} \to \mathbb{R}^F$ for some $F \in \mathbb{N}$. We say that $g$ is \emph{Lipschitz continuous} if there exits a constants $L_g$ such that for all $y,y' \in \mathcal{Y}$
\[
\|g(y) - g(y') \|_\infty \leq L_gd_\mathcal{Y}(y,y').
\]
If the domain $\mathcal{Y}$ is Euclidean, we endow it with the $L^\infty$-metric.

\subsection{Message passing neural networks}

Given a MPNN, we define the \emph{formal bias} of the update and message functions  as 
\begin{equation}
    \label{eq:formalBias}
    \|\Psi^{(l)}(0,0)\|_\infty \quad  \text{and} \quad
\|\Phi^{(l)}(0,0)\|_\infty
\end{equation} 
respectively. 

Next, we introduce notations for the mappings between consecutive layers of a MPNN. 
Let $\Theta = ((\Phi^{(l)})_{l=1}^T, (\Psi^{(l)})_{l=1}^T)$  be a MPNN with $T$ layers and  feature dimensions $(F_l)_{l=1}^T$. For $l=1, \ldots, T$, we denote the mapping from the $(l-1)$'th layer to the $l$'th layer of the graph MPNN by 
\begin{equation}
    \label{def:LayerMapping}
    \begin{aligned}
 \Lambda^{(l)}_{\Theta_A}: \mathbb{R}^{N  \times F_{l-1}} &\to \mathbb{R}^{N \times F_l} \\
\mathbf{f}^{(l-1)} & \mapsto \mathbf{f}^{(l)}.
\end{aligned}
\end{equation}

Similarly, we denote by $\Lambda_{\Theta_W}^{(l)}$ as the mapping from the $(l-1)$'th layer to the $l$'th layer of the cMPNN $f^{(l-1)}\mapsto f^{(l)}$.

Using  \Cref{def:LayerMapping} we can write a MPNN as a composition of message passing layers
\[
\Theta^{(T)}_{A} =  \Lambda^{(T)}_{\Theta_A} \circ  \Lambda^{(T-1)}_{\Theta_A}\circ \ldots \circ  \Lambda^{(1)}_{\Theta_A}
\]
and 
\[
\Theta^{(T)}_{W } =  \Lambda^{(T)}_{\Theta_W} \circ  \Lambda^{(T-1)}_{\Theta_W}\circ \ldots \circ  \Lambda^{(1)}_{\Theta_W}
\]

\section{Convergence Analysis}
\label{AppendixB}
 In this section, we present the proofs for the results discussed in \Cref{sec:unif_cov}. Specifically, our primary objective is to derive \Cref{cor: main convergence}, which constitutes the final result of this section.

\begin{lemma}
\label{lemma:good edges}
Let $(\chi,d, \P) $ be a metric-probability space and
$W$ be a graphon. Let $\{I_j\}_{j\in \mathcal{J}}$ be any set of measurable subsets  $I_j\subset \chi$.
Let $\{X_1, \ldots, X_N\} \sim \mu^N$ be drawn i.i.d. from $\chi$ via $\mu$ and $\mathbf{A}(X_k,X_i) \sim \mathrm{Ber}\big( W(X_k,X_i)\big)$. For every $p \in (0,1)$, there exists an event with probability at least $1-p$  such that
\[
  \max_{j =1,\ldots, |\mathcal{J}|} \max_{k=1, \ldots, N}\frac{1}{N} \left| \sum_{i=1}^N \left(\mathbf{A}(X_k,X_i)  -W(X_k,X_i) \right) \mathbbm{1}_{I_j}(X_i) \right| \leq \frac{1}{\sqrt{2}} \frac{\sqrt{\log(2|\mathcal{J}|N/p)}}{\sqrt{N}}.
\]
\end{lemma}
\begin{proof}
Let $j = 1, \ldots, |\mathcal{J}|$ and let $q \in (0, \frac{1}{|\mathcal{J}|N})$.  First, condition $X:=(X_1,\ldots,X_N) = (\tilde{X}_1,\ldots,\tilde{X}_N):=\tilde{X}$ for some arbitrary deterministic  $(\tilde{X}_1,\ldots,\tilde{X}_N)\in\chi^N$.
Fix $k$, and consider the independent Bernoulli random variables $A(\tilde{X}_k,\tilde{X}_i)\mathbbm{1}_{I_{j}}(\tilde{X_i})$, where $i=1,\ldots, N$. By Hoeffding's inequality
(see \Cref{thm:Hoeffdings}) on these variables, for each $k$ there is an event of probability at least $1-q$ such that
\begin{align*}
    \mathbb{P}\left(  \left| \sum_{i=1}^N \left(A(X_k,X_i)-W(X_k,X_i)\right) \mathbbm{1}_{I_{j}}(X_i) \right| \leq \frac{1}{\sqrt{2}} \frac{\sqrt{\log(2/q)}}{\sqrt{N}} \; \Bigg| \; X=(\tilde{X}_1, \ldots, \tilde{X}_N) \right) \leq 1-q
\end{align*}
 Now, intersect the $N$ events corresponding to $k=1,\ldots,N$  to get 
\begin{align*}
   & \mathbb{P}\left(  \max_{k=1, \ldots, N}\left| \sum_{i=1}^N \left(A(X_k,X_i)-W(X_k,X_i)\right) \mathbbm{1}_{I_{j}}(X_i)\right| \leq \frac{1}{\sqrt{2}} \frac{\sqrt{\log(2/q)}}{\sqrt{N}} \; \Bigg| \; X=(\tilde{X}_1, \ldots, \tilde{X}_N) \right) \\ &\leq 1-Nq
\end{align*}

By the law of total probability, we have 

\begin{align*} 
     & \mathbb{P}\left( \max_{k=1, \ldots, N}\left| \sum_{i=1}^N \left(A({X}_k,{X}_i)-W({X}_k,{X}_i)\right) \mathbbm{1}_{I_{j}}(X_i)\right| \leq \frac{1}{\sqrt{2}}   \frac{\sqrt{\log(2/q)}}{\sqrt{N}}\right) \\
     &
     = \int_{\chi^N}\mathbb{P}\Bigg( \max_{k=1, \ldots, N}\left| \sum_{i=1}^N \left(A(X_k,X_i)-W(X_k,X_i)\right) \mathbbm{1}_{I_{j}}(X_i)\right|   \leq   \frac{1}{\sqrt{2}} \frac{\sqrt{\log(2/q)}}{\sqrt{N}} 
    \\&
    \; \Bigg| \;  X=(\tilde{X}_1, \ldots,\tilde{X}_N) \Bigg) d{\mu^N}(\tilde{X}) \\
     \leq 1-Nq.
\end{align*} 
Now, intersect the $|\mathcal{J}|$ events corresponding to the sets $I_j$ for $j=1, \ldots, |\mathcal{J}|$,  to get 
\begin{align*}
    & \mathbb{P}\left(\max_{j=1, \ldots, |\mathcal{J}|} \max_{k=1, \ldots, N}\left| \sum_{i=1}^N \left(A({X}_k,{X}_i)-W({X}_k,{X}_i)\right) \mathbbm{1}_{I_{j}}(X_i)\right| \leq \frac{1}{\sqrt{2}} \frac{\sqrt{\log(2/q)}}{\sqrt{N}}\right) \leq 1-|\mathcal{J}|Nq.
\end{align*}
Set $p:=q/|\mathcal{J}|N$ to finish the proof.
\end{proof}

Recall that, given a metric space $\mathcal{M}$, we denote by $B_\varepsilon^\infty(\mathcal{M})$ the ball in $L^{\infty}(\mathcal{M})$ of radius $\varepsilon$ centered about $0$. 
\begin{lemma}
\label{lemma:uniform montecarlo}
Let $(\chi,d, \P) $ be a metric-probability space and
$W$ be an admissible graphon. Suppose that $\{X_1, \ldots, X_N \}\sim \mu^N$ are drawn i.i.d. from $\chi$ via  $\mu$, $(V,g)$ is drawn from $B_\varepsilon^\infty(\chi^2)\times B_\varepsilon^\infty(\chi)$ via some Borel probability measure $\nu$,    and $\mathbf{A}(X_k,X_i) \sim \mathrm{Ber}\big( W(X_k,X_i) + V(X_k,X_i)\big)$. Set $\W = W+V$. For every $p \in (0,\frac{1}{2})$, there exists an event $\mathcal{E}^{p}$ with probability $\P^N(\mathcal{E}^{p}) \geq 1-2p$, such that for  every choice of the constants $L_f,L_{\Phi}>0$ and  every Lipschitz continuous function $f:\chi \to \mathbb{R}^F$ with Lipschitz constant bounded by  $L_f$ and Lipschitz continuous function $\Phi:\mathbb{R}^{2F} \to \mathbb{R}^H$ with Lipschitz constant bounded by  $L_\Phi$, we have \[
\begin{aligned}
& \max_{k=1, \ldots, N} \left\| \frac{1}{N} \sum_{i=1}^N \A(X_k,X_i) \Phi\big(
f(X_k), f(X_i)
 \big) - \int_\chi W(X_k,y) \Phi\big(
 f(X_k), f(y)\big) d\mu(y) \right\|_\infty \\ 
 & \leq N^{-\frac{1}{2(D_\chi+1)}} \Bigg(L_\Phi L_f +  C_\chi  \|\Phi(f,f)\| \frac{1}{\sqrt{2}} {\sqrt{\log(C_\chi)+\frac{D_\chi}{2(D_\chi+1)} \log(N)+\log(2N/p)}}  \\
  &  + L_W\|\Phi(f,f)\|_\infty 
+  C_\chi \left\|W\right\|_\infty    \left\|\Phi(f,f) \right\|_\infty \frac{1}{\sqrt{2}}\sqrt{\log(C_\chi)+\frac{D_\chi}{2(D_\chi+1)} \log(N)+ \log(2/p)} \\ &  +  \|\Phi(f,f)\|_\infty L_W + \|W\|_\infty L_\Phi L_f \Bigg) + \textcolor{black}{\big( L_\Phi   \|f\|_\infty  + \|\Phi(0,0)\|_\infty\big) \varepsilon}.
\end{aligned}
\]
\end{lemma}
\begin{proof}
Let $r > 0$.
By \Cref{ass:graphon1} in \Cref{ass:graphon}, there exists an open covering $(B_j)_{j\in \mathcal{J}}$ of $\chi$  by a family of  balls with radius $r$  
such that $|\mathcal{J}| \leq C_\chi r^{-D_\chi}$.
For $j = 2, \ldots, |\mathcal{J}|$, we define $I_j := B_j \setminus \cup_{i < j} B_i$, and define $I_1=B_1$. Hence, $(I_j)_{j \in \mathcal{J}}$ is a family of measurable sets  such that $I_j \cap I_i = \emptyset$ for all $i\neq j \in \mathcal{J}$,  $\bigcup_{j \in \mathcal{J}} I_j = \chi$, and $\mathrm{diam}(I_j) \leq 2r$ for all $j \in \mathcal{J}$, where by convention $\mathrm{diam}(\emptyset)=0$. For each $I_j \in \mathcal{J}$, let $z_j$ be the center of the ball $B_j$.

 Next, we compute a concentration of error bound on the difference between the measure of $I_j$ and its Monte Carlo approximation, which is uniform in $I_j\in\mathcal{J}$.  
Let $I_j \in \mathcal{J}$ and $q \in (0,1)$. By Hoeffding's inequality, there is an event $\mathcal{E}_j^q$ with probability  $\mu(\mathcal{E}_j^q)\geq 1-q$, in which
\begin{equation}
\label{eq:stepfctEvent1}
\left\| \frac{1}{N} \sum_{i=1}^N \big(\mathbbm{1}_{I_j}(X_i) - \mu(I_k)\big)\right\|_\infty \leq \frac{1}{\sqrt{2}}\frac{\sqrt{ \log(2/q)}}{\sqrt{N}}.
\end{equation}
Consider the event 
\[\mathcal{E}_{\rm Lip}^{|\mathcal{J}|q} = \bigcap_{j=1}^{|\mathcal{J}|}\mathcal{E}_j^q,\]
with probability  $\mu^N(\mathcal{E}_{\rm Lip}^{|\mathcal{J}|q})\geq 1- |\mathcal{J}|q $.
In this event,  \Cref{eq:stepfctEvent1} holds for every $I_j\in\mathcal{J}$. We change the failure probability variable $p = |\mathcal{J}|q$, and denote $\mathcal{E}_{\rm Lip}^{p}= \mathcal{E}_{\rm Lip}^{|\mathcal{J}|q}$.

Now, condition on $(V,g) \in B_\varepsilon^\infty(\chi^2)\times B_\varepsilon^\infty(\chi)$ and set $\tilde{W}:= W + V$ and $\tilde{f}:= f + g$.
By \Cref{lemma:good edges} there exists an event  $\mathcal{E}_{\rm Ber}^{p}$ with probability at least $1-p$ over the Bernoulli edges and the choice of $X_1, \ldots, X_N$  in which
\begin{equation}
    \label{eq:good edge event}
  \max_{j=1, \ldots, |\mathcal{J}|}  \max_{k=1, \ldots, N} \frac{1}{N} \left| \sum_{i=1} \A(X_k,X_i)-\tilde{W}(X_k,X_i) \mathbbm{1}_{I_j}(X_i) \right|  \leq \frac{1}{\sqrt{2}} \frac{\sqrt{\log(2N/p)}}{\sqrt{N}}
\end{equation}
    holds.
    Now assume that $\mathcal{E}^{p}:=\mathcal{E}_{\rm Lip}^{p} \cap \mathcal{E}_{\rm Ber}^{p}$ occurs.
We have for every $k=1, \ldots, N$
\begin{equation}
\label{eq:123uniformBound2}
\begin{aligned}
&   \left\| \frac{1}{N }\sum_{i=1}^N  \A(X_k,X_i) \Phi\big(
f(X_k), f(X_i)
 \big) 
 - \int_\chi W(X_k,y) \Phi\big(
 f(X_k), f(y)\big) d\mu(y)  d\P(y)\right\|_\infty \\
 & \leq \left\|\frac{1}{N }\sum_{i=1}^N  \A(X_k,X_i) \Phi\big(
 f(X_k), f(X_i)
 \big) - \frac{1}{N }\sum_{i=1}^N \sum_{j \in \mathcal{J}} \A(X_k,X_i)   \Phi\big(
 f(X_k), f(z_j)  \big)  \mathbbm{1}_{I_j}(X_i) 
  \right\|_\infty
 \\
 & + \left\|\frac{1}{N }\sum_{i=1}^N \sum_{j \in \mathcal{J}} \A(X_k,X_i)   \Phi\big(
 f(X_k), f(z_j)  \big)  \mathbbm{1}_{I_j}(X_i)  - 
    \tilde{W}(X_k,X_i)   \Phi\big(
 f(X_k), f(z_j)  \big)  \mathbbm{1}_{I_j}(X_i) \right\|_\infty
 \\
 & + \left\| \frac{1}{N }\sum_{i=1}^N \sum_{j \in \mathcal{J}}  \tilde{W}(X_k,X_i)   \Phi\big(
 f(X_k), f(z_j)  \big)  \mathbbm{1}_{I_j}(X_i)   - 
 W(X_k,X_i)   \Phi\big(
 f(X_k), f(z_j)  \big)  \mathbbm{1}_{I_j}(X_i) \right\|_\infty
  \\
  & + \left\|\frac{1}{N }\sum_{i=1}^N \sum_{j \in \mathcal{J}}  W(X_k,X_i)   \Phi\big(
 f(X_k), f(z_j)  \big)  \mathbbm{1}_{I_j}(X_i)   -    W(X_k,z_j)   \Phi\big(
 f(X_k), f(z_j)  \big)  \mathbbm{1}_{I_j}(X_i)  \right\|_\infty
  \\
 & + \left\|\frac{1}{N} \sum_{i=1}^N \sum_{j \in \mathcal{J}}  W(X_k,z_j)   \Phi\big(
 f(X_k), f(z_j)  \big)  \mathbbm{1}_{I_j}(X_i)    - \int_\chi  W(X_k,z_j)   \Phi\big(
 f(X_k), f(z_j)  \big)  \mathbbm{1}_{I_j}(y)   \right\|_\infty
 \\
 & + \left\| \int_\chi \sum_{j \in \mathcal{J}} W(X_k,z_j)   \Phi\big(
 f(X_k), f(z_j)  \big)  \mathbbm{1}_{I_j}(y)  d\mu(y) -  W(X_k,y)   \Phi\big(
 f(X_k), f(z_j)  \big)\mathbbm{1}_{I_j}(y) d\P(y) \right\|_\infty
 \\
 & + \left\| \int_\chi \sum_{j \in \mathcal{J}}  W(X_k,y)   \Phi\big(
 f(X_k), f(z_j)  \big) \mathbbm{1}_{I_j}(y)d\P(y)   -   W(X_k,y)   \Phi\big(
 f(X_k), f(y)  \big) d\P(y) \right\|_\infty
 \\
 & =: (1) + (2) +(3)  + (4) + (5) + (6) +(7).
 \end{aligned}
\end{equation}
We bound each term separately. To bound (1), we define for each $X_i$ the unique index $j_i \in \{1, \ldots, |\mathcal{J}|\}$  s.t. $X_i \in I_{j_i}$. Then,  
\begin{align*}
  & \left\|\frac{1}{N }\sum_{i=1}^N  \A(X_k,X_i) \Phi\big(
 f(X_k), f(X_i)
 \big) - \frac{1}{N }\sum_{j \in \mathcal{J}} \sum_{i=1}^N  \A(X_k,X_i)   \Phi\big(
 f(X_k), f(z_j)  \big)  \mathbbm{1}_{I_j}(X_i) 
  \right\|_\infty  \\
  &
  \leq \frac{1}{N } \sum_{i=1}^N \left\| \Phi\big(
 f(X_k), f(X_i)  \big)  - \Phi\big(
 f(X_k), f(z_{j_i})  \big) \right \|_\infty \\
 & \leq L_\Phi L_f r.
\end{align*}

We proceed by bounding (2).  In the event of $\mathcal{E}_{\rm Ber}^{p}$,we get for every $k=1, \ldots, N$ \begin{align*}
  &   \left\|\frac{1}{N }\sum_{i=1}^N \sum_{j \in \mathcal{J}} \A(X_k,X_i)   \Phi\big(
 f(X_k), f(z_j)  \big)  \mathbbm{1}_{I_j}(X_i)  - 
    W(X_k,X_i)   \Phi\big(
 f(X_k), f(z_j)  \big)  \mathbbm{1}_{I_j}(X_i) \right\|_\infty \\ 
& = \frac{1}{N }\sum_{j \in \mathcal{J}}\left\|\Phi\big(
 f(X_k), f(z_j)  \big)\right\|_\infty \left| \sum_{i=1}^N  (\A(X_k,X_i)    \mathbbm{1}_{I_j}(X_i)  - 
   W(X_k,X_i)) \mathbbm{1}_{I_j}(X_i) \right| \\
   & \leq |\mathcal{J}| \|\Phi(f,f)\|_\infty\frac{1}{\sqrt{2}} \frac{\sqrt{\log(2|\mathcal{J}|N/p)}}{\sqrt{N}}  
\end{align*}
Recall that $|\mathcal{J}| \leq C_\chi r^{-D_\chi}$. Then,
\begin{align*}
 & \left\|\frac{1}{N }\sum_{i=1}^N \sum_{j \in \mathcal{J}} \A(X_k,X_i)   \Phi\big(
 f(X_k), f(z_j)  \big)  \mathbbm{1}_{I_j}(X_i)  -    W(X_k,X_i)   \Phi\big(
 f(X_k), f(z_j)  \big)  \mathbbm{1}_{I_j}(X_i) \right\|_\infty
\\
& \leq C_\chi r^{-D_\chi} \left\|\Phi(f,f) \right\|_\infty \frac{1}{\sqrt{2}}\frac{\sqrt{\log(C_\chi)-D_\chi \log(r)+ \log(2N/p)}}{\sqrt{N}}
\end{align*}

\textcolor{black}{
For (3), we calculate:
\begin{align*}
  &  \left\| \frac{1}{N }\sum_{i=1}^N \sum_{j \in \mathcal{J}}  \tilde{W}(X_k,X_i)   \Phi\big(
 f(X_k), f(z_j)  \big)  \mathbbm{1}_{I_j}(X_i)   -   W(X_k,X_i)   \Phi\big(
 f(X_k), f(z_j)  \big)  \mathbbm{1}_{I_j}(X_i) \right\|_\infty \\
 & = \left\| \frac{1}{N }\sum_{i=1}^N    \tilde{W}(X_k,X_i)   \Phi\big(
 f(X_k), f(z_{j_i})  \big)      - 
 \frac{1}{N }\sum_{i=1}^N   W(X_k,X_i)   \Phi\big(
 f(X_k), f(z_{j_i})  \big)   \right\|_\infty \\
  & = \left\| \frac{1}{N }\sum_{i=1}^N   \Big( \tilde{W}(X_k,X_i)  - W(X_k,X_i) \Big) \Phi\big(
 f(X_k), f(z_{j_i})  \big)   \right\|_\infty \\
 & \leq  \left\|W - \tilde{W}\right\|_\infty   \frac{1}{N }\sum_{i=1}^N  \left\|  \Phi\big(
 f(X_k), f(z_{j_i})  \big)   \right\|_\infty \\
 & \leq \left\| U\right\|_\infty     \left\| \Phi\big(
 f,f  \big)  \right\|_\infty \leq   \left\| \Phi\big(
 f,f  \big)  \right\|_\infty\varepsilon.
\end{align*}
Note that $ \left\| \Phi\big(
 f,f  \big)  \right\|_\infty  \leq L_\Phi \|f\|_\infty + \|\Phi(0,0)\|_\infty$. Hence,
 \begin{align*}
  &  \left\| \frac{1}{N }\sum_{i=1}^N \sum_{j \in \mathcal{J}}  \tilde{W}(X_k,X_i)   \Phi\big(
 f(X_k), f(z_j)  \big)  \mathbbm{1}_{I_j}(X_i)   -   W(X_k,X_i)   \Phi\big(
 f(X_k), f(z_j)  \big)  \mathbbm{1}_{I_j}(X_i) \right\|_\infty \\
 &  \leq  \Big( L_\Phi   \|f\|_\infty  + \|\Phi(0,0)\|_\infty\Big) \varepsilon.
\end{align*}}

To bound (4), we calculate  
\begin{align*}
  &  \left\|\frac{1}{N }\sum_{i=1}^N \sum_{j \in \mathcal{J}}  W(X_k,X_i)   \Phi\big(
 f(X_k), f(z_j)  \big)  \mathbbm{1}_{I_j}(X_i)   -   W(X_k,z_j)   \Phi\big(
 f(X_k), f(z_j)  \big)  \mathbbm{1}_{I_j}(X_i)  \right\|_\infty \\
 & \left\|\frac{1}{N }\sum_{i=1}^N   W(X_k,X_i)   \Phi\big(
 f(X_k), f(z_{j_i})  \big)    -  \frac{1}{N }\sum_{i=1}^N     W(X_k,z_j)   \Phi\big(
 f(X_k), f(z_{j_i})  \big)   \right\|_\infty \\
& \leq  L_W r \| \Phi(f,f)\|_\infty.
\end{align*}

To bound (5),  since we are in the event $\mathcal{E}_{\rm Lip}^{p}$, for every $k$ we have  
\begin{align*}
     & \left\|\frac{1}{N} \sum_{i=1}^N \sum_{j \in \mathcal{J}}  W(X_k,z_j)   \Phi\big(
 f(X_k), f(z_j)  \big)  \mathbbm{1}_{I_j}(X_i)    - \int_\chi  W(X_k,z_j)   \Phi\big(
 f(X_k), f(z_j)  \big)  \mathbbm{1}_{I_j}(y)   \right\|_\infty \\
 & \leq \Bigg\|\sum_{j \in \mathcal{J}} 
 \Bigg( 
 \frac{1}{N} \sum_{i=1}^N    W(X_k,z_j)   \Phi\big(
 f(X_k), f(z_j)  \big)  \mathbbm{1}_{I_j}(X_i) \\ &    - \int_\chi   W(X_k,z_j)   \Phi\big(
 f(X_k), f(z_j)  \big)  \mathbbm{1}_{I_j}(y)  
 \Bigg)
 \Bigg\|_\infty  \\
 & \leq \sum_{j \in \mathcal{J}} 
 \left\|W\right\|_\infty     \left\| \Phi(f,f) \right\|_\infty \left| 
 \frac{1}{N} \sum_{i=1}^N     \mathbbm{1}_{I_j}(X_i)    - \int_\chi    \mathbbm{1}_{I_j}(y)   
 \right| \\
 & \leq |\mathcal{J}|  \left\|W\right\|_\infty     \left\|\Phi(f,f)\right\|_\infty \frac{1}{\sqrt{2}} \frac{\sqrt{\log(2|\mathcal{J}|/p)}}{\sqrt{N}}.
 \end{align*}

Recall that $|\mathcal{J}| \leq C_\chi r^{-D_\chi}$. Then, 
\begin{align*}
  & \left\|\frac{1}{N} \sum_{i=1}^N \sum_{j \in \mathcal{J}}  W(X_k,z_j)   \Phi\big(
 f(X_k), f(z_j)  \big)  \mathbbm{1}_{I_j}(X_i)    - \int_\chi  W(X_k,z_j)   \Phi\big(
 f(X_k), f(z_j)  \big)  \mathbbm{1}_{I_j}(y)   \right\|_\infty \\ & \leq C_\chi r^{-D_\chi} \left\|W\right\|_\infty     \left\|\Phi(f,f)\right\|_\infty \frac{1}{\sqrt{2}}\frac{\sqrt{\log(C_\chi)-D_\chi \log(r)+ \log(2/p)}}{\sqrt{N}}.
\end{align*}

For (6), we calculate 
\begin{align*}
    &  \left\| \int_\chi \sum_{j \in \mathcal{J}} W(X_k,z_j)   \Phi\big(
 f(X_k), f(z_j)  \big)  \mathbbm{1}_{I_j}(y)  d\mu(y) -   W(X_k,y)   \Phi\big(
 f(X_k), f(z_j)  \big) d\P(y) \mathbbm{1}_{I_j}(y)\right\|_\infty \\
 & \leq \sum_{j \in \mathcal{J}} \int_{I_j  }\left\|
 W(X_k,z_j)   \Phi\big(
 f(X_k), f(z_j)  \big)   d\mu(y) -   W(X_k,y)   \Phi\big(
 f(X_k), f(z_j)  \big) d\P(y) \right\|_\infty \\
 & \leq r \|\Phi(f,f)\|_\infty L_W   .
\end{align*}

For (7), we calculate
\begin{align*}
   &  \left\|\int_\chi\sum_{j \in \mathcal{J}} W(X_k,y)   \Phi\big(
 f(X_k), f(z_j)  \big) d\P(y) \mathbbm{1}_{I_j}(y)   - \int_\chi W(X_k,y)   \Phi\big(
 f(X_k), f(y)  \big) d\P(y) \right\|_\infty \\
 & \leq \sum_{j \in \mathcal{J}} \int_{I_j  }\left\|
 W(X_k,y)   \Phi\big(
 f(X_k), f(z_j)  \big)   d\mu(y) -   W(X_k,y)   \Phi\big(
 f(X_k), f(y)  \big) d\P(y) \right\|_\infty \\
 & \leq r \|W \|_\infty L_\Phi L_f 
\end{align*}

All together, we get 
\begin{align*}
&  \left\| \frac{1}{N }\sum_{i=1}^N  \A(X_k,X_i) \Phi\big(
 f(X_k), f(X_i)
 \big) 
 - \int_\chi W(X_k,y) \Phi\big(
 f(X_k), f(y)\big) d\mu(y)  d\P(y)\right\|_\infty \\
 & \leq rL_\Phi L_f +  C_\chi r^{-D_\chi} \|\Phi(f,f)\| \frac{1}{\sqrt{2}} \frac{\sqrt{\log(C_\chi)-D_\chi \log(r)+\log(2N/p)}}{\sqrt{N}} + L_Wr\|\Phi(f,f)\|_\infty
 \\ & +  C_\chi r^{-D_\chi}\left\|W\right\|_\infty     \left\|\Phi(f,f)\right\|_\infty \frac{1}{\sqrt{2}}\frac{\sqrt{\log(C_\chi)-D_\chi \log(r)+ \log(2/p)}}{\sqrt{N}} \\ & + r \|\Phi(f,f)\|_\infty L_W + r \|W\|_\infty L_\Phi L_f.
\end{align*}
We set $r = N^{-\frac{1}{2(D_\chi+1)}}$, which leads to 
\begin{align*}
  &  \left\| \frac{1}{N }\sum_{i=1}^N  \A(X_k,X_i) \Phi\big(
 f(X_k), f(X_i)
 \big) 
 - \int_\chi W(X_k,y) \Phi\big(
 f(X_k), f(y)\big) d\mu(y)  d\P(y)\right\|_\infty \\
 & \leq N^{-\frac{1}{2(D_\chi+1)}} \Bigg(L_\Phi L_f +  C_\chi  \|\Phi(f,f)\| \frac{1}{\sqrt{2}} {\sqrt{\log(C_\chi)+\frac{D_\chi}{2(D_\chi+1)} \log(N)+\log(2N/p)}} \\ & + L_W\|\Phi(f,f)\|_\infty 
  +  C_\chi \left\|W\right\|_\infty     \left\|\Phi(f,f)\right\|_\infty \frac{1}{\sqrt{2}}\sqrt{\log(C_\chi)+\frac{D_\chi}{2(D_\chi+1)} \log(N)+ \log(2/p)}  \\
  &  +  \|\Phi(f,f)\|_\infty L_W + \|W\|_\infty L_\Phi L_f \Bigg)+ \textcolor{black}{  \big( L_\Phi   \|f\|_\infty  + \|\Phi(0,0)\|_\infty\big) \varepsilon.}
\end{align*}
 We conclude by applying the law of total probability.
\end{proof}

\begin{lemma}
\label{lemma:bernoulliHoeffdings}
Let $(\chi,d, \P) $ be a metric-probability space and
$W$ be a graphon. Let $\mathbf{X}=\{X_1, \ldots, X_N \}\sim \mu^N$ be drawn i.i.d. from $\chi$ via $\mu$ and $\mathbf{A}(X_k,X_i) \sim \mathrm{Ber}\big( W(X_k,X_i)\big)$. For every $p \in (0,1)$, there exists an even with probability at least $1-p$ such that for all $X_i \in\{ X_1, \ldots, X_n\}$
\[
\begin{aligned}
& 
\left|\mathrm{d}_{\mathbf{A}}(X_i) - \mathrm{d}_{\mathbf{X}}(X_i)\right| =  \left| \frac{1}{N} \sum_{j=1}^N \mathbf{A}(X_i,X_j)  - \frac{1}{N} \sum_{j=1}^N W(X_i,X_j) \   \right|  \leq  \frac{1}{\sqrt{2}} \frac{ \sqrt{\log(2N/p)}  }{\sqrt{N}}
\end{aligned}
\]
\end{lemma}
The proof is similar to the proof of \Cref{lemma:good edges}, and we omit it.

\textcolor{black}{
\begin{lemma}
 \label{degree close when perturbed}
 Let $(\chi, d, \mu)$ be a metric-probability space and $W$ be an admissible graphon (\Cref{ass:graphon}). Let $\X = \{X_1, \ldots, X_N\} \sim \mu^N$ be drawn i.i.d. from $\chi$ via $\mu$, $(V,g)$ be drawn from $B_\varepsilon^\infty(\chi^2)\times B_\varepsilon^\infty(\chi)$ via the Borel probability measure $\nu$,  and $\mathbf{A}(X_k,X_i) \sim \mathrm{Ber}\big( W(X_k,X_i) + V(X_k,X_i)\big)$. 
 Suppose that $N$ satisfies  
\begin{equation}
\begin{aligned}
\label{eq:lowerBoundGraphSizeN}
\sqrt{N} \geq   \max\Bigg\{  &   4\sqrt{2} \frac{\sqrt{ \log(2N/p)}}{\mathrm{d}_{\mathrm{min}}}, \\ &  4 \Big(\zeta   \frac{\cl}{\cmin} \big(\sqrt{\log (C_\chi)} +  \sqrt{D_\chi}\big) +
    \frac{\sqrt{2} \cmax + \zeta \cl}{ \cmin } \sqrt{\log 2/p}\Big)\Bigg\},
\end{aligned}
\end{equation}
where $\zeta$ is defined as 
\begin{equation}
    \label{eq:zeta}
    \zeta  := \frac{2}{\sqrt{2}}e\Big(\frac{2}{\ln(2)} +1 \Big)\frac{1}{\sqrt{\ln(2)}} C
\end{equation}
and $C$ is the universal constant from Dudley' inequality (see \cite[Theorem 8.1.6]{vershynin_2018}).
Then, for any $p \in (0, \frac{1}{2})$, with probability at least $1-2p$, we have
\[
\begin{aligned}
 & \max_{i=1, \ldots, N} |\rm d_{\mathbf{A}}(X_i) - d_W(X_i) |  \leq  \varepsilon +  \frac{1}{\sqrt{2}} \frac{\sqrt{\log(2N/p)}}{\sqrt{N}} \\
  & + \frac{  \Big(\zeta \cl (\sqrt{\log (C_\chi)} +  \sqrt{D_\chi}) + (\sqrt{2}\cmax+\zeta \cl) \sqrt{\log2/p} \Big) }{ \sqrt{N}}
\end{aligned}
 \]
 and 
 \[
\min_{i=1, \ldots, N} \mathrm{d}_{\mathbf{A}}(X_i) \geq \frac{\mathrm{d}_{\mathrm{min}}}{2} - \varepsilon.
\]
\end{lemma}}

 \begin{proof}
\textcolor{black}{
We denote $\tilde{W} = W + V$.
We calculate 
\begin{align*}
& \left\|
\frac{1}{N} \sum_{i=1}^N \tilde{W}(\cdot,X_i)  - \int_\chi W(\cdot, x)   d\P(x)\right\|_\infty \\
& \leq \left\|
\frac{1}{N} \sum_{i=1}^N \tilde{W}(\cdot,X_i)  -\frac{1}{N} \sum_{i=1}^N W(\cdot,X_i) \right\|_\infty + \left\|
\frac{1}{N} \sum_{i=1}^N W(\cdot,X_i) -  \int_\chi W(\cdot, x)   d\P(x)\right\|_\infty 
\\
&  \leq  \varepsilon + 
\frac{  \Big(\zeta \cl (\sqrt{\log (C_\chi)} +  \sqrt{D_\chi}) + (\sqrt{2}\cmax+\zeta \cl) \sqrt{\log2/p} \Big) }{ \sqrt{N}},
\end{align*}
where the last inequality holds by \Cref{lemma:kerivenLemma4} with probability at least $1-p$.
Now, note that, by \Cref{lemma:bernoulliHoeffdings}, we have with probability at least $1-p$, 
\[
\max_{i=1, \ldots, N} |\rm d_A(X_i) - \frac{1}{N} \sum_{i=1}^N \tilde{W}(X_i,X_j) | \leq \frac{1}{\sqrt{2}} \frac{\sqrt{\log(2N/p)}}{\sqrt{N}}.
\]
Hence, in the joint event of probability $1-2p$, we have
\begin{align*}
  & \max_{i=1, \ldots, N} |\rm d_{\mathbf{A}}(X_i) - d_W(X_i) |  \leq  \varepsilon +  \frac{1}{\sqrt{2}} \frac{\sqrt{\log(2N/p)}}{\sqrt{N}} \\
  & + \frac{  \Big(\zeta \cl (\sqrt{\log (C_\chi)} +  \sqrt{D_\chi}) + (\sqrt{2}\cmax+\zeta \cl) \sqrt{\log2/p} \Big) }{ \sqrt{N}}.
\end{align*}
Furthermore, since  \Cref{eq:lowerBoundGraphSizeN} holds, we have in this event,
 \[
 \max_{i=1, \ldots, N}
\left| \mathrm{d}_{\mathbf{A}}(X_i) - \mathrm{d}_{W}(X_i) \right| \leq \varepsilon + \frac{\cmin}{2}.
 \]
 Thus, by $ \mathrm{d}_{W}(\cdot) \geq \cmin$,  
 \[
\min_{i=1, \ldots, N} \mathrm{d}_{\mathbf{A}}(X_i) \geq \frac{\mathrm{d}_{\mathrm{min}}}{2} - \varepsilon.
\]}
 \end{proof}

 \begin{lemma}
\label{lemma:C2}
Let $(\chi,d, \mu) $ be a metric-probability space and
$W$ be an admissible graphon.  Suppose that $\X = \{X_1, \ldots, X_N\} \sim \mu^N$ are drawn i.i.d. from $\chi$ via $\mu$, $(V,g)$ is drawn from $B_\varepsilon^\infty(\chi^2)\times B_\varepsilon^\infty(\chi)$ via $\nu$  and $\A(X_k,X_i) \sim \mathrm{Ber}\big( W(X_k,X_i) + V(X_k,X_i)\big)$.  Let $p \in (0, \frac{1}{4})$. If $N \in \mathbb{N}$ satisfies \Cref{eq:lowerBoundGraphSizeN}, there exists an event $\mathcal{F}_{\rm Lip}^p$ with probability $\mu(\mathcal{F}_{\rm Lip}^p) \geq 1-4p$ such that  for every choice of constants $L_f,L_{\Phi}>0$ and Lipschitz continuous functions $f:\chi \to \mathbb{R}^{F}$ with  Lipschitz constant  bounded by  $L_f$  and $\Phi: \mathbb{R}^{2F} \to \mathbb{R}^{H}$ with   Lipschitz constant  bounded by  $L_\Phi$, the following is satisfied  
\begin{equation}
\begin{aligned}
\label{eq:lemmab5-12}
& \max_{X_1, \ldots, X_N} \left\|  M_{\A}\big(\Phi (f,f)\big) (X_i)   - M_W\big(\Phi (f,f)\big) (X_i)  \right\|_\infty\\
 \leq & \frac{  \|\Phi(f,f)\|_\infty}{ \mathrm{d}_{\mathrm{min}} }  \left(\frac{\sqrt{\log(2N/p)}}{\sqrt{N}} + \frac{ \Big(\zeta\cl \big(\sqrt{\log (C_\chi)} +  \sqrt{D_\chi}\big) + \big(\sqrt{2}\cmax+\zeta\cl\big) \sqrt{\log2/p} \Big) }{ \sqrt{N}}\right) \\
& +  \frac{1}{N^{\frac{1}{2(D_\chi+1)}}\cmin} \Bigg(L_\Phi L_f +  C_\chi  \|\Phi(f,f)\| \frac{1}{\sqrt{2}} {\sqrt{\log(C_\chi)+\frac{D_\chi}{2(D_\chi+1)} \log(N)+\log(2N/p)}} \\ & + L_W\|\Phi(f,f)\|_\infty 
    +  C_\chi \left\|W \right\|_\infty \left\| \Phi(f,f) \right\|_\infty \frac{1}{\sqrt{2}}\sqrt{\log(C_\chi)+\frac{D_\chi}{2(D_\chi+1)} \log(N)+ \log(2/p)} \\ & +  \|\Phi(f,f)\|_\infty L_W + \|W\|_\infty L_\Phi L_f \Bigg)  + \textcolor{black}{\big( L_\Phi   \|f\|_\infty  + \|\Phi(0,0)\|_\infty\big) \varepsilon.} 
\end{aligned}
\end{equation}
Here, $\zeta$ is given in \eqref{eq:zeta}.

  \end{lemma}
\begin{proof}
We consider the joint event $\mathcal{F}^p_{\mathrm{Lip}}$ of probability at least $1-4p$ from \Cref{lemma:uniform montecarlo} and \Cref{degree close when perturbed}.
We calculate for  $X_i \in \{ X_1, \ldots, X_N\}$
\[
\begin{aligned}
& \left\|\frac{1}{N} \sum_{j=1}^N \frac{\A(X_i,X_j)}{\mathrm{d}_A(X_i)} \Phi\big(f(X_i),f(X_j)\big)
-  \int \frac{W(X_i,y)}{\mathrm{d}_W(X_i)}\Phi\big(f(X_i),f(y)\big) d\mu(y)\right\|_\infty \\
  &  \leq
\left\|
\frac{1}{N} \sum_{j=1}^N \frac{\A(X_i,X_j)}{\mathrm{d}_A(X_i)} \Phi\big(f(X_i),f(X_j)\big)
- 
\frac{1}{N} \sum_{j=1}^N \frac{\A(X_i,X_j)}{\mathrm{d}_W(X_i)}\Phi\big(f(X_i),f(X_j)\big)
\right\|_\infty  \\
& +
\left\|
\frac{1}{N} \sum_{j=1}^N \frac{\A(X_i,X_j)}{\mathrm{d}_W(X_i)}\Phi\big(f(X_i),f(X_j)\big) - \int \frac{W(X_i,y)}{\mathrm{d}_W(X_i)}\Phi\big(f(X_i),f(y)\big) d\mu(y) \right\|_\infty
\\
& \leq  \left| \frac{1}{\mathrm{d}_{\A}(X_i)} - \frac{1}{\mathrm{d}_W(X_i)}\right| \mathrm{d}_{\A}(X_i) \|\Phi(f,f)\|_\infty \\ & + \left|\frac{1}{\mathrm{d}_W(X_i)}\right| \left\| \frac{1}{N} \sum_{j=1}^N {\A(X_i,X_j)}  \Phi\big(f(X_i),f(X_j)\big)
- 
\int  W(X_i,y) \Phi\big(f(X_i),f(y)\big) d\mu(y) \right\|_\infty  \\
& = (1) + (2)
\end{aligned}
\] 
By \Cref{degree close when perturbed},  we have in the event of $\mathcal{F}_{\mathrm{Lip}}^p$  
\[
\begin{aligned}
 \left| \frac{1}{\mathrm{d}_{\A}(X_i)} - \frac{1}{\mathrm{d}_W(X_i)} \right|\mathrm{d}_{\A}(X_i) & \leq
\left| \frac{\mathrm{d}_W(X_i) - \mathrm{d}_{\A}(X_i)}{\mathrm{d}_{\A}(X_i)\mathrm{d}_W(X_i)} \right|\mathrm{d}_{\A}(X_i)
\\
 & \leq \frac{1}{ \mathrm{d}_{\mathrm{min}} }  \Bigg(\textcolor{black}{\varepsilon} +  \frac{\sqrt{\log(2N/p)}}{\sqrt{N}} 
 \\ & + \frac{ \Big(\zeta\cl \big(\sqrt{\log (C_\chi)} +  \sqrt{D_\chi}\big) + \big(\sqrt{2}\cmax+\zeta\cl\big) \sqrt{\log2/p} \Big) }{ \sqrt{N}}\Bigg).
\end{aligned}
\] 
Thus for every  $\Phi$ and $f$ that satisfy the conditions of \Cref{degree close when perturbed}, 
\[
\begin{aligned}
& \left| \frac{1}{\mathrm{d}_{\A}(X_i)} - \frac{1}{\mathrm{d}_W(X_i)}\right| \mathrm{d}_{\A}(X_i) \|\Phi(f,f)\|_\infty
 \leq  \frac{  \|\Phi(f,f)\|_\infty}{ \mathrm{d}_{\mathrm{min}} }  \Bigg(\textcolor{black}{\varepsilon}  \\ & + \frac{\sqrt{\log(2N/p)}}{\sqrt{N}} + \frac{ \Big(\zeta\cl \big(\sqrt{\log (C_\chi)} +  \sqrt{D_\chi}\big) + \big(\sqrt{2}\cmax+\zeta\cl\big) \sqrt{\log2/p} \Big) }{ \sqrt{N}}\Bigg).
\end{aligned} 
\]
Furthermore, (2) is bounded by \Cref{lemma:uniform montecarlo} and by  $\mathrm{d}_W \geq \cmin$, i.e., for every  $\Phi$ and $f$ that satisfy the conditions of \Cref{lemma:uniform montecarlo}
\[
\begin{aligned}
&  \left|\frac{1}{\mathrm{d}_W(X_i)}\right|\left\| \frac{1}{N} \sum_{j=1}^N {\A(X_i,X_j)}  \Phi\big(f(X_i),f(X_j)\big)
- 
\int  W(X_i,y) \Phi\big(f(X_i),f(y)\big) d\mu(y) \right\|_\infty \\ & \leq  \frac{1}{\cmin N^{\frac{1}{2(D_\chi+1)}} } \Bigg(L_\Phi L_f +  C_\chi  \|\Phi(f,f)\| \frac{1}{\sqrt{2}} {\sqrt{\log(C_\chi)+\frac{D_\chi}{2(D_\chi+1)} \log(N)+\log(2N/p)}} \\ & + L_W\|\Phi(f,f)\|_\infty 
  +  C_\chi \left\|W \right\|_\infty \left\| \Phi(f,f) \right\|_\infty \frac{1}{\sqrt{2}}\sqrt{\log(C_\chi)+\frac{D_\chi}{2(D_\chi+1)} \log(N)+ \log(2/p)} \\ & +  \|\Phi(f,f)\|_\infty L_W + \|W\|_\infty L_\Phi L_f \Bigg)   +  \textcolor{black}{\big( L_\Phi   \|f\|_\infty  + \|\Phi(0,0)\|_\infty\big) \varepsilon.}
\end{aligned}
\]
We hence get 
\begin{align*}
    & \left\|\frac{1}{N} \sum_{j=1}^N \frac{\A(X_i,X_j)}{\mathrm{d}_A(X_i)} \Phi\big(f(X_i),f(X_j)\big)
-  \int \frac{W(X_i,y)}{\mathrm{d}_W(X_i)}\Phi\big(f(X_i),f(y)\big) d\mu(y)\right\|_\infty \\
\leq &   \frac{  \|\Phi(f,f)\|_\infty}{ \mathrm{d}_{\mathrm{min}}}  \left(\frac{\sqrt{\log(2N/p)}}{\sqrt{N}} + \frac{ \Big(\zeta\cl \big(\sqrt{\log (C_\chi)} +  \sqrt{D_\chi}\big) + \big(\sqrt{2}\cmax+\zeta\cl\big) \sqrt{\log2/p} \Big) }{ \sqrt{N}}\right) \\
& + \frac{1}{\cmin N^{\frac{1}{2(D_\chi+1)}} }  \Bigg(L_\Phi L_f +  C_\chi  \|\Phi(f,f)\| \frac{1}{\sqrt{2}} {\sqrt{\log(C_\chi)+\frac{D_\chi}{2(D_\chi+1)} \log(N)+\log(2N/p)}} \\ &  + L_W\|\Phi(f,f)\|_\infty 
  +  C_\chi \left\|W \right\|_\infty \left\| \Phi(f,f) \right\|_\infty \frac{1}{\sqrt{2}}\sqrt{\log(C_\chi)+\frac{D_\chi}{2(D_\chi+1)} \log(N)+ \log(2/p)}\\ & +  \|\Phi(f,f)\|_\infty L_W + \|W\|_\infty L_\Phi L_f \Bigg)   +  \textcolor{black}{\left( \frac{\|\Phi(f,f)\|}{\cmin} +  L_\Phi   \|f\|_\infty  + \|\Phi(0,0)\|_\infty\right) \varepsilon.} 
\end{align*} 
\end{proof}

\begin{corollary}
\label{cor:monte carlo after update function}
Let $(\chi,d, \mu) $ be a metric-probability space and
$W$ be an admissible graphon.  Suppose that $\X = \{X_1, \ldots, X_N\} \sim \mu^N$ are drawn i.i.d. from $\chi$ via $\mu$, $(V,g)$ is drawn via $\nu$ from $B_\varepsilon^\infty(\chi^2)\times B_\varepsilon^\infty(\chi)$   and $\A(X_k,X_i) \sim \mathrm{Ber}\big( W(X_k,X_i) + V(X_k,X_i)\big)$.  Let $p \in (0, \frac{1}{4})$. If $N \in \mathbb{N}$ satisfies \Cref{eq:lowerBoundGraphSizeN},
there exists an event $\mathcal{F}_{\rm Lip}^p$
with probability $\mu(\mathcal{F}_{\rm Lip}^p) \geq 1-4p$  such that  for every  choice of constants $L_f,L_{\Phi},L_{\Psi}>0$ and Lipschitz continuous functions $f:\chi \to \mathbb{R}^{F}$ with  Lipschitz constant bounded by $L_f$, $\Phi: \mathbb{R}^{2F} \to \mathbb{R}^{H}$ with   Lipschitz constant bounded by $L_\Phi$, and $\Psi: \mathbb{R}^{F+H} \to \mathbb{R}^{F'}$ with   Lipschitz constant bounded by $L_\Psi$,
\begin{equation}
\begin{aligned}
 & \max_{X_1, \ldots, X_N} \left\| \Psi\Big( f(X_i), M_\A\big(\Phi (f,f)\big) (X_i)   - \Psi\Big( f(X_i), M_W\big(\Phi (f,f)\big) (X_i)   \right\|_\infty 
 \\   &\leq   L_\Psi \Bigg( \frac{  \|\Phi(f,f)\|_\infty}{ \mathrm{d}_{\mathrm{min}} }  \Bigg(\frac{\sqrt{\log(2N/p)}}{\sqrt{N}} \\ &
 + \frac{ \Big(\zeta\cl \big(\sqrt{\log (C_\chi)} +  \sqrt{D_\chi}\big) + \big(\sqrt{2}\cmax+\zeta\cl\big) \sqrt{\log2/p} \Big) }{ \sqrt{N}}\Bigg) \\
& +  \frac{1}{N^{\frac{1}{2(D_\chi+1)}}\cmin} \Bigg(L_\Phi L_f +  C_\chi  \|\Phi(f,f)\| \frac{1}{\sqrt{2}} {\sqrt{\log(C_\chi)+\frac{D_\chi}{2(D_\chi+1)} \log(N)+\log(2N/p)}} 
 \\
  & 
  + L_W\|\Phi(f,f)\|_\infty 
+  C_\chi \left\|W \right\|_\infty \left\| \Phi(f,f) \right\|_\infty \frac{1}{\sqrt{2}}\sqrt{\log(C_\chi)+\frac{D_\chi}{2(D_\chi+1)} \log(N)+ \log(2/p)}\\ 
  & +  \|\Phi(f,f)\|_\infty L_W   + \|W\|_\infty L_\Phi L_f \Bigg) \\ & +  \textcolor{black}{\left( \frac{\|\Phi(f,f)\|}{\cmin} +  L_\Phi   \|f\|_\infty  + \|\Phi(0,0)\|_\infty\right) \varepsilon.}  \Bigg).
\end{aligned}
\end{equation}
\end{corollary}
\begin{proof}
The proof follows directly from the Lipschitz continuity of the update function $\Psi$ and \Cref{lemma:C2}.
\end{proof}

From this point, the proof of \Cref{thm:main convergence} closely follows the proof of Theorem 3.1 in \cite{maskey2022generalization}, which can be located in the appendix (Section B) of the same publication. Nevertheless, for the purpose of presenting a complete argument, we outline the intermediate results leading to the proof of \Cref{thm:main convergence} and reference the corresponding results and proofs from \cite{maskey2022generalization}.

We start by introducing the following two results from \cite{maskey2022generalization}.

\begin{lemma}[Lemma B.7 in \citep{maskey2022generalization}]
\label{lemma:RecRelNorm}
Let $(\chi,d, \P) $ be a metric-probability space,
$W$ be an admissible graphon and $\Theta = ((\Phi^{(l)})_{l=1}^T, (\Psi^{(l)})_{l=1}^T)$  be a MPNN. Consider a metric-space signal  $f: \chi \to \mathbb{R}^F$ with $\|f\|_\infty < \infty$. Then, for $l=0, \ldots, T-1$, the cMPNN output $f^{(l+1)}$ satisfies 
\[
\|f^{(l+1)}\|_\infty \leq B_1^{(l+1)} + \|f\|_\infty B_2^{(l+1)},
\]
where
\begin{equation}
    \label{eq:B'}
    B_1^{(l+1)} = \sum_{k=1}^{l+1}  \big(
L_{\Psi^{(k)}} \|\Phi^{(k)}(0,0)\|_\infty+ \|\Psi^{(k)}(0,0)\|_\infty \big) \prod_{l' = k+1}^{l+1}  L_{\Psi^{(l')}} \big( 1 + L_{\Phi^{(l')}} \big) 
\end{equation}
and
\begin{equation}
    \label{eq:B''}
    B_2^{(l+1)} = \prod_{k = 1}^{l+1} L_{\Psi^{(k)}} \left(1  +  L_{\Phi^{(k)}} \right).
\end{equation}

\end{lemma}

\begin{lemma}[Lemma B.9 in \citep{maskey2022generalization}]
\label{lemma:lip mpnn output}
Let $(\chi,d, \P) $ be a metric-probability space,
$W$ be an admissible graphon and $\Theta = ((\Phi^{(l)})_{l=1}^T, (\Psi^{(l)})_{l=1}^T)$  be a MPNN. Consider a Lipschitz continuous 
metric-space signal  $f: \chi \to \mathbb{R}^F$ with $\|f\|_\infty < \infty$.  and Lipschitz constant $L_f$. Then, for $l=0, \ldots, T-1$,
\[
 \Lipfl  \leq Z_1^{(l)} + Z_2^{(l)}\|f\|_\infty + Z_3^{(l)} \Lipf ,
\]
where $Z_1^{(l)}$, $Z_2^{(l)}$ and $Z_3^{(l)}$ are independent of $f$ and defined as 
\begin{equation}
\label{eq:z1z2z3}
  \begin{aligned}
&  Z_1^{(l)} = \sum_{k=1}^{l}  \Bigg(\Big(
L_{\Psi^{(k)}}\frac{\cl}{\cmin}  \|\Phi^{(k)}(0,0)\|_\infty  +
L_{\Psi^{(k)}}\cmax
\|\Phi^{(k)}(0,0)\|_\infty 
 \frac{\cl}{\cmin^2}\Big)  \\
 & +  B_1^{(k-1)} \Big(
L_{\Psi^{(k)}}\frac{\cl}{\cmin}  L_{\Phi^{(k)}}   +
L_{\Psi^{(k)}}\cmax
 L_{\Phi^{(k)}} 
 \frac{\cl}{\cmin^2}\Big) \Bigg)  \prod_{l' = k+1}^{l} L_{\Psi^{(l')}}  \Big(1+\frac{\cmax}{\cmin} L_{\Phi^{(l')}}  \Big), \\
 & Z_2^{(l)} = \sum_{k=1}^{l}   B_2^{(k-1)}   \Big(
L_{\Psi^{(k)}}\frac{\cl}{\cmin}  L_{\Phi^{(k)}}   +
L_{\Psi^{(k)}}\cmax
 L_{\Phi^{(k)}} 
 \frac{\cl}{\cmin^2}\Big)  \prod_{l' = k+1}^{l} L_{\Psi^{(l')}}  \Big(1+\frac{\cmax}{\cmin} L_{\Phi^{(l')}}  \Big), \\
 & Z_3^{(l)} =     \prod_{k=1}^{l}
L_{\Psi^{(k)}} \Big(1+\frac{\cmax}{\cmin} L_{\Phi^{(k)}}\Big),
\end{aligned}
\end{equation}
where $B_1^{(k)}$ and $B_2^{(k)}$ are defined in \Cref{eq:B'} and \Cref{eq:B''}.
\end{lemma}

\begin{corollary}
\label{cor:C5}
Let $(\chi,d, \P) $ be a metric-probability space and
$W$ be an admissible graphon. Let $p \in (0, \frac{1}{4})$.  
 Consider a graph-signal $\{G,\mathbf{f}\} \sim_{\nu} \{W,f\}$ with $N$ nodes and corresponding graph features,  where $N$ satisfies  \Cref{eq:lowerBoundGraphSizeN}. If the event $\mathcal{F}_{\rm Lip}^p$ from \Cref{lemma:C2} occurs, then the following is satisfied.  
 For every MPNN $\Theta$ and $f:\chi \to \mathbb{R}^F$ with Lipschitz constant $L_f$, 
  we have 
\begin{equation}
\label{eq:lemmac4-1}
  \d\left( \Lambda_{\Theta_\A}^{(l+1)}(S^X f^{(l)}),  \Lambda_{\Theta_W}^{(l+1)}(f^{(l)})\right)
    \leq Q^{(l+1)}  
\end{equation}
for all $l = 0, \ldots, T-1$, where $f^{(l)}=\Theta_W^{(l)}(f)$ as defined in \Cref{cMPNNdef}, and $\Lambda_{\Theta_\A}^{(l+1)}$ and $\Lambda_{\Theta_W}^{(l+1)}$ are defined in \Cref{def:LayerMapping}.  Here,
 \begin{equation}
 \label{eq:defDl2}
 \begin{aligned}
& Q^{(l+1)}  =  L_{\Psi^{(l+1)}} \Bigg( \frac{ \|\Phi^{(l+1)}(f^{(l)},f^{(l)})\|_\infty}{ \mathrm{d}_{\mathrm{min}} }  \Bigg(\frac{\sqrt{\log(2N/p)}}{\sqrt{N}} \\& + \frac{ \Big(\zeta\cl \big(\sqrt{\log (C_\chi)} + \sqrt{D_\chi}\big) + \big(\sqrt{2}\cmax+\zeta\cl\big) \sqrt{\log2/p} \Big) }{ \sqrt{N}}\Bigg) \\& +  \frac{1}{N^{\frac{1}{2(D_\chi+1)}}\cmin} \Bigg(L_{\Phi^{(l+1)}} L_{f^{(l)}} +  C_\chi  \|\Phi^{(l+1)}(f^{(l)},f^{(l)})\| \frac{1}{\sqrt{2}} \\ & \cdot {\sqrt{\log(C_\chi)+\frac{D_\chi}{2(D_\chi+1)} \log(N)+\log(2N/p)}}  + L_W\| \Phi^{(l+1)}(f^{(l)},f^{(l)})\|_\infty 
\\ & +  C_\chi \left\|W\right\|_\infty  \left\|  \Phi^{(l+1)}(f^{(l)},f^{(l)}) \right\|_\infty \frac{1}{\sqrt{2}}\sqrt{\log(C_\chi)+\frac{D_\chi}{2(D_\chi+1)} \log(N)+ \log(2/p)}  \\ & +  \|\Phi^{(l+1)}(f^{(l)},f^{(l)})\|_\infty L_W + \|W\|_\infty L_{\Phi^{(l+1)}} L_{f^{(l)}} \Bigg)  \\ &  +  \textcolor{black}{\left( \frac{\|\Phi^{(l+1)}(f^{(l)},f^{(l)})\|}{\cmin} +  L_\Phi   \|f^{(l)}\|_\infty  + \|\Phi(0,0)\|_\infty\right) \varepsilon.}  \Bigg).
 \end{aligned}
 \end{equation}
\end{corollary}
\begin{proof}
Let $l=0,\ldots, T-1$. 
Note that $f^{(l)}$ is bounded by \Cref{lemma:RecRelNorm} and Lipschitz continuous by \Cref{eq:z1z2z3}. Then, we can apply  \Cref{cor:monte carlo after update function} to $f^{(l)}$.
Hence, for every admissible $\Phi^{(l+1)}, \Psi^{(l+1)}$ and $f$, 
\[
\begin{aligned}
& \d\big( \Lambda_{\Theta_\A}^{(l+1)}(S^X f^{(l)}),  \Lambda_{\Theta_W}^{(l+1)}(f^{(l)})\big) \\
& =  \| \Lambda_{\Theta_\A}^{(l+1)}(S^X f^{(l)}) - S^X \Lambda_{\Theta_W}^{(l+1)}(f^{(l)})  \|_{\infty;\infty} \\ & = 
\max_{i=1, \ldots, N} \| \Lambda_{\Theta_\A}^{(l+1)}(S^X f^{(l)}) (X_i) - S^X \Lambda_{\Theta_W}^{(l+1)}(f^{(l)}) (X_i)\|_\infty \\
& \leq L_{\Psi^{(l+1)}} \Bigg( \frac{  \|\Phi^{(l+1)}(f^{(l)},f^{(l)})\|_\infty}{ \mathrm{d}_{\mathrm{min}} }  \Bigg(\frac{\sqrt{\log(2N/p)}}{\sqrt{N}} 
\\&  + \frac{ \Big(\zeta\cl \big(\sqrt{\log (C_\chi)} +  \sqrt{D_\chi}\big) + \big(\sqrt{2}\cmax+\zeta\cl\big) \sqrt{\log2/p} \Big) }{ \sqrt{N}}\Bigg)+   \Bigg(L_{\Phi^{(l)}} L_{f^{(l)}}  \\
& + \|\Phi(f^{(l)},f^{(l)})\| \frac{ C_\chi }{\sqrt{2}} {\sqrt{\log(C_\chi)+\frac{D_\chi}{2(D_\chi+1)} \log(N)+\log(2N/p)}}  + L_W\|\Phi^{(l+1)}(f^{(l)},f^{(l)})\|_\infty 
\\
&
 +  C_\chi \left\|W \right\|_\infty \left\| \Phi^{(l+1)}(f^{(l)},f^{(l)}) \right\|_\infty \frac{1}{\sqrt{2}}\sqrt{\log(C_\chi)+\frac{D_\chi}{2(D_\chi+1)} \log(N)+ \log(2/p)}  
 \end{aligned}
 \]
 \[
\begin{aligned}
  & +  \|\Phi^{(l+1)}(f^{(l)},f^{(l)})\|_\infty L_W + \|W\|_\infty L_{\Phi^{(l)}} L_{f^{(l)}} \Bigg)\frac{1}{N^{\frac{1}{2(D_\chi+1)}}\cmin}
  \\
  &+  \textcolor{black}{\left( \frac{\|\Phi^{(l+1)}(f^{(l)},f^{(l)})\|}{\cmin} +  L_\Phi   \|f^{(l)}\|_\infty  + \|\Phi(0,0)\|_\infty\right) \varepsilon.} \Bigg),
\end{aligned}
 \]
\end{proof}

\begin{lemma}
\label{lemma:C5}
Let $(\chi,d, \P) $ be a metric-probability space and
$W$ be an admissible graphon.
Let $p \in (0, \frac{1}{4})$. Consider a graph-signal $\{G,\mathbf{f}\} \sim_{\nu} \{W,f\}$ with $N$ nodes and corresponding graph features, where $N$ satisfies \Cref{eq:lowerBoundGraphSizeN}.
Denote, for $l=1,\ldots,T$, 
\[
 \delta^{(l)} = \d( \Theta^{(l)}_{\A}(\mathbf{f}), \Theta^{(l)}_{W}(f)  ),
\]
and $\delta^{(0)} = \d(\mathbf{f}, f)$.
If the event $\mathcal{F}_{\rm Lip}^p$ from \Cref{lemma:C2} occurs, then,  for every MPNN $\Theta$ and $f:\chi \to \mathbb{R}^F$ with Lipschitz constant $L_f$, the following recurrence relation holds: 
\[
\begin{aligned}
\delta^{(l+1)} \leq K^{(l+1)} \delta^{(l)} + Q^{(l+1)}
\end{aligned}
\]
for $l=0, \ldots, T-1$. Here, $Q^{(l+1)}$ is defined in \Cref{eq:defDl2}, and
\begin{equation}
    \label{eq:lemmaC6}
K^{(l+1)}  =     L_{\Psi^{(l+1)}}  \max\left\{1, L_{\Phi^{(l+1)}} \right\}   .
 \end{equation}
\end{lemma}
\begin{proof}
In the event $\mathcal{F}_{\rm Lip}^p$,
 by \Cref{cor:C5}, we have for every MPNN $\Theta$  and $f:\chi \to \mathbb{R}^F$ with Lipschitz constant $L_f$, 
\begin{equation}
    \label{eq:lemmaC6-00} \d\left( \Lambda_{\Theta_\A}^{(l+1)}(S^X f^{(l)}),  \Lambda_{\Theta_W}^{(l+1)}(f^{(l)})\right)
    \leq Q^{(l+1)} 
\end{equation}
 for $l=0,\ldots, T-1$, and for every $i=1, \ldots, N$
\begin{equation}
    \label{eq:lemmaC6-0} 
|\mathrm{d}_\A(X_i)| \geq \frac{\mathrm{d}_{\mathrm{min}}}{2} - \textcolor{black}{\varepsilon}.
\end{equation}
Let $l=0,\ldots,T-1$.
We have
\begin{equation}
\label{eq:lemmaC6-1}
\begin{aligned}
  & \d( \Theta^{(l+1)}_{\A}(\mathbf{f})  , \Theta^{(l+1)}_{W}(f)  ) \\ &  = \| \Theta^{(l+1)}_{\A}(\mathbf{f}) - S^X \Theta^{(l+1)}_{W}(f) \|_{\infty; \infty} \\
    & \leq \| \Theta^{(l+1)}_{\A}(\mathbf{f}) - \Lambda^{(l+1)}_{\Theta_\A}(S^X f^{(l)})  \|_{\infty; \infty}   +  \| \Lambda^{(l+1)}_{\Theta_\A}(S^X f^{(l)}) - S^X\Theta^{(l+1)}_{\Theta_W}(f)  \|_{\infty;\infty} \\
        & = \| \Lambda^{(l+1)}_\A(\mathbf{f}^{(l)}) - \Lambda^{(l+1)}_\A(S^X f^{(l)})  \|_{\infty; \infty}  +  \| \Lambda^{(l+1)}_{\Theta_\A}(S^X f^{(l)}) - S^X\Lambda^{(l+1)}_{\Theta_W}(f^{(l)})  \|_{\infty; \infty} \\
& \leq  \| \Lambda^{(l+1)}_{\Theta_\A}(\mathbf{f}^{(l)}) - \Lambda^{(l+1)}_{\Theta_\A}(S^X f^{(l)})\|_{\infty; \infty}  + Q^{(l+1)}.
\end{aligned}
\end{equation}
We bound the first term on the right-hand-side of \Cref{eq:lemmaC6-1} as follows.  
\begin{equation}
\label{eq:lemmaC6-2}  
\begin{aligned}
 &\| \Lambda^{(l+1)}_{\Theta_\A}(\mathbf{f}^{(l)}) - \Lambda^{(l+1)}_{\Theta_\A}(S^X f^{(l)})\|_{\infty; \infty}
 \\
 & =  \max_{i=1, \ldots, N} \Big\|
 \Psi^{(l+1)}  \Big( \mathbf{f}^{(l)}_i, M_\A \big( \Phi^{(l+1)} (\mathbf{f}^{(l)}, \mathbf{f}^{(l)}) \big)(X_i) \Big)
\\ 
& - \Psi^{(l+1)} \Big( (S^X f^{(l)})_i, M_\A\big( \Phi^{(l+1)} (S^X f^{(l)}, S^X f^{(l)} ) \big)(X_i) \Big)
 \Big\|_\infty \\
 & \leq  L_{\Psi^{(l+1)}} \max_{i=1, \ldots, N}
   \Big\|
  \Big( \mathbf{f}^{(l)}_i, M_\A\big( \Phi^{(l+1)} (\mathbf{f}^{(l)}, \mathbf{f}^{(l)}) \big)(X_i) \Big)
  \\
& -  \Big( (S^X f^{(l)})_i, M_\A\big( \Phi^{(l+1)} (S^X f^{(l)}, S^X f^{(l)} ) \big)(X_i) \Big)
\Big\|_\infty  \\
& \leq  L_{\Psi^{(l+1)}} 
 \max_{i=1, \ldots, N}\max\Big( \Big\|
  \mathbf{f}^{(l)}_i - (S^X f^{(l)})_i  \Big\|_\infty, 
  \\ & \quad\quad\quad\quad\quad\quad\quad 
   \Big\|
 M_\A\big( \Phi^{(l+1)} (\mathbf{f}^{(l)}, \mathbf{f}^{(l)}) \big)(X_i) 
-  M_\A\big( \Phi^{(l+1)} (S^X f^{(l)}, S^X f^{(l)} ) \big)(X_i)
 \Big\|_\infty \Big)\\
 & \leq L_{\Psi^{(l+1)}} \max \Big( \d(\mathbf{f}^{(l)}, f^{(l)}),  \max_{i=1, \ldots, N}  \big\|
 M_\A\big( \Phi^{(l+1)} (\mathbf{f}^{(l)},
  \\ & \quad\quad\quad\quad\quad\quad\quad\quad\quad\quad\quad\quad \quad\quad\;
   \mathbf{f}^{(l)}) \big)(X_i)
-  M_\A\big( \Phi^{(l+1)} (S^X f^{(l)}, S^X f^{(l)} )\big)(X_i) \big\|_\infty
 \Big)
 \\
 & \leq  L_{\Psi^{(l+1)}}  \max \Big(
\delta^{(l)}, \max_{i=1, \ldots, N} \big\|
 M_\A\big( \Phi^{(l+1)} (\mathbf{f}^{(l)}, \mathbf{f}^{(l)}) \big)(X_i) \\ & \quad\quad\quad\quad\quad\quad\quad\quad\quad\quad\quad\quad\quad\quad\quad\quad\quad\quad\;\;
-  M_\A\big( \Phi^{(l+1)} (S^X f^{(l)}, S^X f^{(l)} )\big)(X_i) \big\|_\infty
 \Big).
\end{aligned}
\end{equation}
 Now, for every $i=1, \ldots, N$, we have  
\begin{equation}
    \label{eq:lemmaC6-3}
\begin{aligned}
 & \Big\|
 M_\A \Big( \Phi^{(l+1)} \big(\mathbf{f}^{(l)}, \mathbf{f}^{(l)} \big)  \Big)(X_i)
-  M_\A \Big(  \Phi^{(l+1)} \big( S^X f^{(l)}, S^X f^{(l)} \big) \Big)(X_i) \Big\|_\infty  \\
& = \Big\|
\frac{1}{N} \sum_{j=1}^N \frac{\A(X_i, X_j)}{\mathrm{d}_A(X_i)} \Phi^{(l+1)} \big(\mathbf{f}^{(l)}(X_i), \mathbf{f}^{(l)}(X_j) \big) \\&- 
\frac{1}{N}\sum_{j=1}^N \frac{\A(X_i, X_j)}{\mathrm{d}_A(X_i)} \Phi^{(l+1)} \big(S^X f^{(l)}(X_i), S^X f^{(l)}(X_j) \big)
\Big\|_\infty \\
& = \Big\|
\frac{1}{N} \sum_{j=1}^N \frac{\A(X_i, X_j)}{\mathrm{d}_A(X_i)} \Big(\Phi^{(l+1)} \big(\mathbf{f}^{(l)}(X_i), \mathbf{f}^{(l)}(X_j) \big) -  \Phi^{(l+1)} \big(S^X f^{(l)}(X_i), S^X f^{(l)}(X_j) \big)
\Big)
\Big\|_\infty \\
& = 
\frac{1}{N} \sum_{j=1}^N \Big|\frac{\A(X_i, X_j)}{\mathrm{d}_A(X_i)}\Big| \max_{j=1, \ldots, N}\Big\|\Phi^{(l+1)} \big(\mathbf{f}^{(l)}(X_i), \mathbf{f}^{(l)}(X_j) \big) -  \Phi^{(l+1)} \big(S^X f^{(l)}(X_i), S^X f^{(l)}(X_j) \big)
\Big\|_\infty 
\\
& \leq  \max_{j=1, \ldots, N}\Big\|\Phi^{(l+1)} \big(\mathbf{f}^{(l)}(X_i), \mathbf{f}^{(l)}(X_j) \big) -  \Phi^{(l+1)} \big(S^X f^{(l)}(X_i), S^X f^{(l)}(X_j) \big)
  \Big\|_\infty 
\\
& \leq \max_{j=1, \ldots, N} \max \left( L_{\Phi^{(l+1)}}  \big\|\mathbf{f}^{(l)}(X_i) - S^X f^{(l)}(X_i)\big\|_\infty,  L_{\Phi^{(l+1)}}  \big\|\mathbf{f}^{(l)}(X_j) - S^X f^{(l)}(X_j)\big\|_\infty\right) \\
& = \max_{j=1, \ldots, N} L_{\Phi^{(l+1)}}     \big\|\mathbf{f}^{(l)}(X_j) - S^X f^{(l)}(X_j)\big\|_\infty \\
& =   L_{\Phi^{(l+1)}}     \delta^{(l)}.
\end{aligned}
\end{equation}
Hence, by inserting \Cref{eq:lemmaC6-3} into \Cref{eq:lemmaC6-2}, we have
\[
\begin{aligned}
& \| \Lambda^{(l+1)}_{\Theta_G}(\mathbf{f}^{(l)}) - \Lambda^{(l+1)}_{\Theta_G}(S^X f^{(l)})\|_{\infty; \infty} \\
 & \leq L_{\Psi^{(l+1)}} \max\Big(
\delta^{(l)}, \max_{i=1, \ldots, N}  \big\|
  M_G\big( \Phi^{(l)}  (\mathbf{f}^{(l)}, \mathbf{f}^{(l)} ) \big) (X_i)
-  M_G \big( \Phi^{(l)} (S^X f^{(l)}, S^X f^{(l)} ) \big)(X_i) \big\|_\infty 
 \Big) \\
  & \leq  L_{\Psi^{(l+1)}}  \max \Big(
\delta^{(l)},
 \max_{i=1, \ldots, N}  \max_{j=1, \ldots, N} L_{\Phi^{(l+1)}}  \max \left(  \big\|\mathbf{f}^{(l)}(X_i) - S^X f^{(l)}(X_i)\big\|_\infty,   \delta^{(l)}\right)\\
& \leq L_{\Psi^{(l+1)}}  \max  \big(
 \delta^{(l)} , L_{\Phi^{(l+1)}} \delta^{(l)}  \big)  \\
 & = L_{\Psi^{(l+1)}}  \max\left(1, L_{\Phi^{(l+1)}} \right)   
 \delta^{(l)}  
\end{aligned}
\]
\end{proof}

\begin{corollary}
\label{cor:solRecRel}
Let $(\chi,d, \P) $ be a metric-probability space and
$W$ be an admissible graphon. 
Let $p \in (0, \frac{1}{4})$.  Consider a graph-signal $\{G,\mathbf{f}\} \sim_{\nu} \{W,f\}$ with $N$ nodes and corresponding graph features, where $N$ satisfies \Cref{eq:lowerBoundGraphSizeN}. If the event $\mathcal{F}_{\rm Lip}^p$ from \Cref{lemma:C2} occurs, then,   for every MPNN $\Theta$  and every Lipschitz continuous $f:\chi \to \mathbb{R}^F$ with Lipschitz constant $L_f$,
\[
 \d \big(\Theta_\A(f(X))  ,\Theta_W(f) \big) \leq \sum_{l=1}^{T} Q^{(l)} \prod_{l' = l+1}^{T} K^{(l')} + \varepsilon  \prod_{l = 1}^{T} K^{(l)},
\]
  where $Q^{(l)}$ and $K^{(l')}$ are defined in \Cref{eq:defDl2} and \Cref{eq:lemmaC6}, respectively. 
\end{corollary}
\begin{proof}
The proof follows the exact lines of the proof of Corollary B.14 in \cite{maskey2022generalization}.
\end{proof}

\begin{theorem}
\label{thm:convwithoutpooling}
Let $(\chi,d, \P) $ be a metric-probability space and
$W$ be an admissible graphon. 
Let $p \in (0, \frac{1}{4})$. Consider a graph-signal $\{G,\mathbf{f}\} \sim_{\nu} \{W,f\}$ with $N$ nodes and corresponding graph features, where $N$ satisfies \Cref{eq:lowerBoundGraphSizeN}. If the event $\mathcal{F}_{\rm Lip}^p$ from \Cref{lemma:C2} occurs, then for every MPNN $\Theta$ and $f:\chi \to \mathbb{R}^{F}$ with Lipschitz constant $L_f$,
\[
\begin{aligned}
& \d \big(\Theta_\A(f(X))  ,\Theta_W(f) \big)  \leq  \Omega_1 \frac{ \sqrt{\log(C_\chi)+\frac{3\big(D_\chi +2/3\big)}{2(D_\chi +1)} \log(N)+\log(2/p)} }{ N^{\frac{1}{2(D_\chi+1)}}}\\ &  +  \Omega_2 \|f\|_\infty \frac{\sqrt{\log(C_\chi)+\frac{3\big(D_\chi +2/3\big)}{2(D_\chi +1)} \log(N)+\log(2/p)} }{ N^{\frac{1}{2(D_\chi+1)}} }   \\
& +  \Omega_3  \frac{1}{N^{\frac{1}{2(D_\chi + 1)}}} + \Omega_4  \frac{\|f\|_\infty}{N^{\frac{1}{2(D_\chi + 1)}}} +  \Omega_5 \frac{ L_f } { N^{\frac{1}{2(D_\chi+1)}}} \\
& + \Omega_6 \frac{\sqrt{\log(2/p)}}{\sqrt{N}} +  \Omega_{7} \|f\|_\infty \frac{\sqrt{\log(2/p)}}{\sqrt{N}} + \Omega_{8} \frac{\sqrt{\log(N)}}{\sqrt{N}} + \Omega_{9} \|f\|_\infty \frac{\sqrt{\log(N)}}{\sqrt{N}} \\
&+ \Omega_{10} \frac{1}{\sqrt{N}} + \Omega_{11} \frac{ \|f\|_\infty}{\sqrt{N}}
 + \textcolor{black}{ \Omega_{12} \varepsilon},
\end{aligned}
\] 
where $\Omega_i$, for $i=1, \ldots, 13$, are constants of the MPNN $\Theta$, defined in \Cref{eq:defConstantsUniform}, which depend only on the Lipschitz constants of the message and update functions $\{L_{\Phi^{(l)}},L_{\Psi^{(l)}}\}_{l=1}^T$, and the formal biases $\{\|\Phi^{(l)}(0,0)\|_\infty\}_{l=1}^T$. 
\end{theorem}
\begin{proof}
We follow the lines of the proof of Theorem B.15 in \citep{maskey2022generalization}.
In the event $\mathcal{F}_{\rm Lip}^p$, by \Cref{cor:solRecRel},  for every MPNN $\Theta$  and $f:\chi \to \mathbb{R}^F$ with Lipschitz constant $L_f$,
\begin{equation}
\label{eq:C8-1}
    \d \big(\Theta_\A(f(X))  ,\Theta_W(f) \big) \leq \sum_{l=1}^{T} Q^{(l)} \prod_{l' = l+1}^{T} K^{(l')} + \varepsilon  \prod_{l = 1}^{T} K^{(l)},
    \end{equation} where  
    \[
    \begin{aligned}
&Q^{(l)} = L_{\Psi^{(l)}} \Bigg( \frac{  \|\Phi^{(l)}(f^{(l-1)},f^{(l-1)})\|_\infty}{ \mathrm{d}_{\mathrm{min}} }  \Bigg(\frac{\sqrt{\log(2N/p)}}{\sqrt{N}} \\ & + \frac{ \Big(\zeta\cl \big(\sqrt{\log (C_\chi)} +  \sqrt{D_\chi}\big) + \big(\sqrt{2}\cmax+\zeta\cl\big) \sqrt{\log2/p} \Big) }{ \sqrt{N}}\Bigg) \\
& +   \Bigg(L_{\Phi^{(l)}} L_{f^{(l-1)}} +  \|\Phi^{(l)}(f^{(l-1)},f^{(l-1)})\| \frac{ C_\chi }{\sqrt{2}} {\sqrt{\log(C_\chi)+\frac{D_\chi}{2(D_\chi+1)} \log(N)+\log(2N/p)}} \\&  + L_W\|\Phi^{(l)}(f^{(l-1)},f^{(l-1)})\|_\infty 
 \\
  & +  C_\chi \left\|W \right\|_\infty   \left\| \Phi^{(l)}(f^{(l-1)},f^{(l-1)}) \right\|_\infty \frac{1}{\sqrt{2}}\sqrt{\log(C_\chi)+\frac{D_\chi}{2(D_\chi+1)} \log(N)+ \log(2/p)} \\&  +  \|\Phi^{(l)}(f^{(l-1)},f^{(l-1)})\|_\infty L_W + \|W\|_\infty L_{\Phi^{(l)}} L_{f^{(l-1)} }\Bigg)\frac{1}{N^{\frac{1}{2(D_\chi+1)}}\cmin}\Bigg) \\
  & + \textcolor{black}{\left( \frac{\|\Phi^{(l)}(f^{(l-1)},f^{(l-1)})\|}{\cmin} +  L_\Phi   \|f^{(l-1)}\|_\infty  + \|\Phi(0,0)\|_\infty\right) \varepsilon}, \end{aligned}
    \]
   and 
    \[
K^{(l')}  = L_{\Psi^{(l')}} \max \left\{ 1, L_{\Phi^{(l')}} \right\}.
\]

We plug the definition of $Q^{(l)}$ into the right-hand-side of \Cref{eq:C8-1}, to get
\begin{equation}
 \label{eq:thmC8-1}
   \begin{aligned}
& \d \big(\Theta_G(f(X))  ,\Theta_W(f) \big) 
\\ & \leq \sum_{l=1}^{T}   L_{\Psi^{(l)}} \Bigg( \frac{  \|\Phi^{(l)}(f^{(l-1)},f^{(l-1)})\|_\infty}{ \mathrm{d}_{\mathrm{min}} }  \Bigg(\frac{\sqrt{\log(2N/p)}}{\sqrt{N}} \\&  + \frac{ \Big(\zeta\cl \big(\sqrt{\log (C_\chi)} +  \sqrt{D_\chi}\big) + \big(\sqrt{2}\cmax+\zeta\cl\big) \sqrt{\log2/p} \Big) }{ \sqrt{N}}\Bigg) \\
& +   \Bigg(L_{\Phi^{(l)}} L_{f^{(l-1)}} +  \|\Phi^{(l)}(f^{(l-1)},f^{(l-1)})\| \frac{ C_\chi }{\sqrt{2}} {\sqrt{\log(C_\chi)+\frac{D_\chi}{2(D_\chi+1)} \log(N)+\log(2N/p)}} \\&  + L_W\|\Phi^{(l)}(f^{(l-1)},f^{(l-1)})\|_\infty 
 +  C_\chi  \left\|W \right\|_\infty   \left\| \Phi^{(l)}(f^{(l-1)},f^{(l-1)}) \right\|_\infty \\ & \cdot  \frac{1}{\sqrt{2}}\sqrt{\log(C_\chi)+\frac{D_\chi}{2(D_\chi+1)} \log(N)+ \log(2/p)} \\&  +  \|\Phi^{(l)}(f,f)\|_\infty L_W + \|W\|_\infty L_{\Phi^{(l)}} L_{f^{(l-1)} }\Bigg)\frac{1}{N^{\frac{1}{2(D_\chi+1)}}\cmin} \\&  +  \textcolor{black}{\left( \frac{\|\Phi^{(l)}(f^{(l-1)},f^{(l-1)})\|}{\cmin} +  L_\Phi   \|f^{(l-1)}\|_\infty  + \|\Phi(0,0)\|_\infty\right) \varepsilon}  \Bigg) \prod_{l' = l+1}^{T} K^{(l')} + \varepsilon  \prod_{l = 1}^{T} K^{(l)}.
\end{aligned}
\end{equation}
By \Cref{lemma:RecRelNorm}, we have
\begin{equation}
\label{eq:thmC8-infnorm}
||f^{(l-1)}||_{\infty} \leq B_1^{(l-1)}+B_2^{(l-1)}||f||_{\infty},
\end{equation}
where $B_1^{(l)}$, $B_2^{(l)}$ are independent of $f$.
Furthermore, we have by \Cref{lemma:lip mpnn output}
\begin{equation}
\label{eq:thmC8-lip}
L_{f^{(l-1)}} \leq Z^{(l-1)}_1 + Z^{(l-1)}_2\|f\|_\infty + Z^{(l-1)}_3 L_f,
\end{equation}
where  $Z_1^{(l)}$, $Z_2^{(l)}$ and $Z_3^{(l)}$ are independent of $f$. 
We plug the bound of $L_{f^{(l-1)}}$ from \Cref{eq:thmC8-lip} into \Cref{eq:C8-1} 
\[
\begin{aligned}
& \d \big(\Theta_G(f(X))  ,\Theta_W(f) \big) 
\\ & \leq \sum_{l=1}^{T}   L_{\Psi^{(l)}} \Bigg( \frac{\|\Phi^{(l)}(f^{(l-1)},f^{(l-1)})\|_\infty}{ \mathrm{d}_{\mathrm{min}}}  \Bigg(\frac{\sqrt{\log(2N/p)}}{\sqrt{N}}  \\ &+ \frac{ \Big(\zeta\cl \big(\sqrt{\log (C_\chi)} +  \sqrt{D_\chi}\big) + \big(\sqrt{2}\cmax+\zeta\cl\big) \sqrt{\log2/p} \Big) }{ \sqrt{N}}\Bigg) \\
& +  \frac{1}{N^{\frac{1}{2(D_\chi+1)}}\cmin} \Bigg(L_{\Phi^{(l)}} ( Z^{(l-1)}_1 + Z^{(l-1)}_2\|f\|_\infty + Z^{(l-1)}_3 L_{f}) \\
&  +  C_\chi  \|\Phi^{(l)}(f^{(l-1)},f^{(l-1)})\| \frac{1}{\sqrt{2}} {\sqrt{\log(C_\chi)+\frac{D_\chi}{2(D_\chi+1)} \log(N)+\log(2N/p)}} \\ & + L_W\|\Phi^{(l)}(f^{(l-1)},f^{(l-1)})\|_\infty 
 \\
  & +  C_\chi  \left\|W \right\|_\infty   \left\| \Phi^{(l)}(f^{(l-1)},f^{(l-1)}) \right\|_\infty \frac{1}{\sqrt{2}}\sqrt{\log(C_\chi)+\frac{D_\chi}{2(D_\chi+1)} \log(N)+ \log(2/p)} \\
  & +  \|\Phi^{(l)}(f^{(l-1)},f^{(l-1)})\|_\infty L_W + \|W\|_\infty L_{\Phi^{(l)}} ( Z^{(l-1)}_1 + Z^{(l-1)}_2\|f\|_\infty + Z^{(l-1)}_3 L_f)\Bigg)\\
  & + \textcolor{black}{\left( \frac{\|\Phi^{(l)}(f^{(l-1)},f^{(l-1)})\|}{\cmin} +  L_\Phi   \|f^{(l-1)}\|_\infty  + \|\Phi(0,0)\|_\infty\right) \varepsilon}\Bigg) \prod_{l' = l+1}^{T} K^{(l')} + \varepsilon  \prod_{l = 1}^{T} K^{(l)} .
\end{aligned}
\]
We insert the bound $\|\Phi^{(l)} (f^{{l-1}},f^{{l-1}})\|_\infty \leq ( L_{\Phi^{{(l)}}} \|f^{{(l-1)}}\|_\infty + \|\Phi^{{(l)}}(0,0)\|_\infty )$ and of $\|f^{(l-1)}\|_\infty$ from \Cref{eq:thmC8-infnorm} in the above expression, to get  
\begin{equation}
    \label{eq:thmC8-2}
    \begin{aligned}
& \d \big(\Theta_G(f(X))  ,\Theta_W(f) \big) 
\\ & \leq  \textcolor{black}{\sum_{l=1}^{T}   L_{\Psi^{(l)}} \Bigg( \frac{ ( L_{\Phi^{{(l)}}} (B_1^{(l-1)}+B_2^{(l-1)}||f||_{\infty})+ \|\Phi^{{(l)}}(0,0)\|_\infty )}{ \mathrm{d}_{\mathrm{min}} }}
\\&  \cdot \textcolor{black}{\left(\frac{\sqrt{\log(2N/p)}}{\sqrt{N}} + \frac{ \Big(\zeta\cl \big(\sqrt{\log (C_\chi)} +  \sqrt{D_\chi}\big) + \big(\sqrt{2}\cmax+\zeta\cl\big) \sqrt{\log2/p} \Big) }{ \sqrt{N}}\right)} \\
& +  \frac{1}{N^{\frac{1}{2(D_\chi+1)}}\cmin}  \Bigg( \textcolor{black}{L_{\Phi^{(l)}} ( Z^{(l-1)}_1 + Z^{(l-1)}_2\|f\|_\infty + Z^{(l-1)}_3 L_{f})}    +   C_\chi \big( L_{\Phi^{{(l)}}} (B_1^{(l-1)}+B_2^{(l-1)}||f||_{\infty})  \\ & + \|\Phi^{{(l)}}(0,0)\|_\infty \big) \frac{1}{\sqrt{2}} {\sqrt{\log(C_\chi)+\frac{3\big(D_\chi +2/3\big)}{2(D_\chi +1)} \log(N)+\log(2/p)}} 
\\ & +  \textcolor{black}{L_W( L_{\Phi^{{(l)}}}(B_1^{(l-1)}+B_2^{(l-1)}||f^{(l-1)}||_{\infty}) + \|\Phi^{{(l)}}(0,0)\|_\infty ) }
 \\
  & +  C_\chi \left\|W \right\|_\infty  \big( L_{\Phi^{{(l)}}}(B_1^{(l-1)}+B_2^{(l-1)}||f||_{\infty}) + \|\Phi^{{(l)}}(0,0)\|_\infty \big)
  \\ &\cdot \frac{1}{\sqrt{2}}\sqrt{\log(C_\chi)+\frac{D_\chi}{2(D_\chi+1)} \log(N)+ \log(2/p)} \\
  & +  \textcolor{black}{L_W ( L_{\Phi^{{(l)}}}(B_1^{(l-1)}+B_2^{(l-1)}||f^{(l-1)}||_{\infty}) + \|\Phi^{{(l)}}(0,0)\|_\infty ) }\\ &+\textcolor{black}{ \|W\|_\infty L_{\Phi^{(l)}} ( Z^{(l-1)}_1 + Z^{(l-1)}_2\|f\|_\infty + Z^{(l-1)}_3 L_{f})}\Bigg)\\
  & + \textcolor{black}{\left( \frac{\|\Phi^{(l)}(f^{(l-1)},f^{(l-1)})\|}{\cmin} +  L_\Phi   \|f^{(l-1)}\|_\infty  + \|\Phi(0,0)\|_\infty\right) \varepsilon}\Bigg) \prod_{l' = l+1}^{T} K^{(l')} + \varepsilon  \prod_{l = 1}^{T} K^{(l)}.
  \end{aligned}
\end{equation}

We now rearrange and seperate the terms and separate, i.e., \begingroup
\allowdisplaybreaks
\begin{align*}
& \leq  \sum_{l=1}^{T}   L_{\Psi^{(l)}} \frac{  (1 + \|W\|_\infty)  C_\chi ( L_{\Phi^{{(l)}}} B_1^{(l-1)}+ \|\Phi^{{(l)}}(0,0)\|_\infty )    }{ N^{\frac{1}{2(D_\chi+1)}}\cmin  } \\ & \cdot \frac{1}{\sqrt{2}} {\sqrt{\log(C_\chi)+\frac{3\big(D_\chi +2/3\big)}{2(D_\chi +1)} \log(N)+\log(2/p)}}   \prod_{l' = l+1}^{T} K^{(l')}\\
& + \sum_{l=1}^{T}   L_{\Psi^{(l)}} \frac{  (1 + \|W\|_\infty)  C_\chi  L_{\Phi^{{(l)}}} B_2^{(l-1)}  \|f\|_\infty    }{ N^{\frac{1}{2(D_\chi+1)}}\cmin  } \\ & \cdot \frac{1}{\sqrt{2}} {\sqrt{\log(C_\chi)+\frac{3\big(D_\chi +2/3\big)}{2(D_\chi +1)} \log(N)+\log(2/p)}} \prod_{l' = l+1}^{T} K^{(l')}  \\
& + \sum_{l=1}^{T}   L_{\Psi^{(l)}} \frac{ 2  L_W( L_{\Phi^{(l)}} B_1^{(l-1)} + \|\Phi(0,0)\|_\infty)  }{ N^{\frac{1}{2(D_\chi+1)}}\cmin  }  \prod_{l' = l+1}^{T} K^{(l')}  \\
& +  \sum_{l=1}^{T}   L_{\Psi^{(l)}} \frac{ 2  L_W L_{\Phi^{(l)}} B_2^{(l-1)}  \|f\|_\infty }{ N^{\frac{1}{2(D_\chi+1)}}\cmin  }   \prod_{l' = l+1}^{T} K^{(l')} \\
& + \sum_{l=1}^{T}   L_{\Psi^{(l)}} \frac{    L_{\Phi^{(l)}} Z_1^{(l-1)}  (1+ \|W\|_\infty) }{ N^{\frac{1}{2(D_\chi+1)}}\cmin  } \prod_{l' = l+1}^{T} K^{(l')}  \\
& +  \sum_{l=1}^{T}   L_{\Psi^{(l)}} \frac{    2L_{\Phi^{(l)}} Z_2^{(l-1)}  \|f\|_\infty(1+ \|W\|_\infty) }{ N^{\frac{1}{2(D_\chi+1)}}\cmin  }  \prod_{l' = l+1}^{T} K^{(l')} \\
& +  \sum_{l=1}^{T}   L_{\Psi^{(l)}} \frac{    2L_{\Phi^{(l)}} Z_3^{(l-1)} L_f (1+ \|W\|_\infty)}{ N^{\frac{1}{2(D_\chi+1)}}\cmin  }   \prod_{l' = l+1}^{T} K^{(l')} \\
& + \sum_{l=1}^{T}   L_{\Psi^{(l)}} \frac{ L_{\Phi^{{(l)}}}  B_1^{(l-1)} + \|\Phi^{{(l)}}(0,0)\|_\infty }{ \mathrm{d}_{\mathrm{min}}}   \frac{\sqrt{\log(2/p)}}{\sqrt{N}}   \prod_{l' = l+1}^{T} K^{(l')}  \\
&+  \sum_{l=1}^{T}   L_{\Psi^{(l)}} \frac{ L_{\Phi^{{(l)}}}  B_1^{(l-1)} + \|\Phi^{{(l)}}(0,0)\|_\infty }{ \mathrm{d}_{\mathrm{min}}}   \frac{\sqrt{\log(N)}}{\sqrt{N}}   \prod_{l' = l+1}^{T} K^{(l')} \\
& + \sum_{l=1}^{T}   L_{\Psi^{(l)}}  \frac{  L_{\Phi^{{(l)}}}
B_2^{(l-1)}||f||_{\infty}}{ \mathrm{d}_{\mathrm{min}} }\frac{\sqrt{\log(2/p)}}{\sqrt{N}}  \prod_{l' = l+1}^{T} K^{(l')} \\
& + \sum_{l=1}^{T}   L_{\Psi^{(l)}}  \frac{  L_{\Phi^{{(l)}}} B_2^{(l-1)}||f||_{\infty}}{ \mathrm{d}_{\mathrm{min}} }\frac{\sqrt{\log(N)}}{\sqrt{N}}   \prod_{l' = l+1}^{T} K^{(l')} \\
& + \sum_{l=1}^{T}   L_{\Psi^{(l)}} \frac{ L_{\Phi^{{(l)}}}  B_1^{(l-1)} + \|\Phi^{{(l)}}(0,0)\|_\infty }{ \mathrm{d}_{\mathrm{min}}}  \frac{ \zeta\cl \big(\sqrt{\log (C_\chi)} +  \sqrt{D_\chi}\big) }{ \sqrt{N}}   \prod_{l' = l+1}^{T} K^{(l')} 
\\
&+  \sum_{l=1}^{T}   L_{\Psi^{(l)}} \frac{ L_{\Phi^{{(l)}}}  B_1^{(l-1)} + \|\Phi^{{(l)}}(0,0)\|_\infty }{ \mathrm{d}_{\mathrm{min}}}    \frac{\big(\sqrt{2}\cmax+\zeta\cl \big)\sqrt{\log2/p} }{\sqrt{N}}  \prod_{l' = l+1}^{T} K^{(l')} \\
& + \sum_{l=1}^{T}   L_{\Psi^{(l)}}  \frac{L_{\Phi^{{(l)}}} B_2^{(l-1)}||f||_{\infty}}{ \mathrm{d}_{\mathrm{min}}}  \frac{ \zeta\cl \big(\sqrt{\log (C_\chi)} +  \sqrt{D_\chi}\big) }{ \sqrt{N}}   \prod_{l' = l+1}^{T} K^{(l')} \\
& + \sum_{l=1}^{T}   L_{\Psi^{(l)}}  \frac{L_{\Phi^{{(l)}}} B_2^{(l-1)}||f||_{\infty}}{ \mathrm{d}_{\mathrm{min}}} \frac{\big(\sqrt{2}\cmax+\zeta\cl \big)\sqrt{\log2/p} }{\sqrt{N}}   \prod_{l' = l+1}^{T} K^{(l')} \\
& + \sum_{l=1}^T  L_{\Psi^{(l)}}  {\left( \frac{\|\Phi^{(l)}(f^{(l-1)},f^{(l-1)})\|}{\cmin} +  L_\Phi   \|f^{(l-1)}\|_\infty  + \|\Phi(0,0)\|_\infty\right) \varepsilon}   \prod_{l' = l+1}^{T} K^{(l') } + \varepsilon  \prod_{l = 1}^{T} K^{(l)}  \\
& =:
\Omega_1 \frac{ \sqrt{\log(C_\chi)+\frac{3\big(D_\chi +2/3\big)}{2(D_\chi +1)} \log(N)+\log(2/p)} }{ N^{\frac{1}{2(D_\chi+1)}}} \\ & +  \Omega_2 \|f\|_\infty \frac{\sqrt{\log(C_\chi)+\frac{3\big(D_\chi +2/3\big)}{2(D_\chi +1)} \log(N)+\log(2/p)} }{ N^{\frac{1}{2(D_\chi+1)}} }   \\
& +  \Omega_3  \frac{1}{N^{\frac{1}{2(D_\chi + 1)}}} + \Omega_4  \frac{\|f\|_\infty}{N^{\frac{1}{2(D_\chi + 1)}}} +  \Omega_5 \frac{ L_f } { N^{\frac{1}{2(D_\chi+1)}}} \\
& + \Omega_6 \frac{\sqrt{\log(2/p)}}{\sqrt{N}} +  \Omega_{7} \|f\|_\infty \frac{\sqrt{\log(2/p)}}{\sqrt{N}} + \Omega_{8} \frac{\sqrt{\log(N)}}{\sqrt{N}} + \Omega_{9} \|f\|_\infty \frac{\sqrt{\log(N)}}{\sqrt{N}} \\
& + \Omega_{10} \frac{1}{\sqrt{N}} + \Omega_{11} \frac{ \|f\|_\infty}{\sqrt{N}} + \textcolor{black}{ \Omega_{12}\varepsilon}.
\end{align*}
\endgroup
where we define 
\begingroup
\allowdisplaybreaks
\begin{align}
\label{eq:defConstantsUniform}
& \Omega_1 :=   \sum_{l=1}^{T}   L_{\Psi^{(l)}} \frac{  (1 + \|W\|_\infty)  C_\chi ( L_{\Phi^{{(l)}}} B_1^{(l-1)}+ \|\Phi^{{(l)}}(0,0)\|_\infty ) \frac{1}{\sqrt{2}}     }{\cmin  }   \prod_{l' = l+1}^{T} K^{(l')} \\
& \Omega_2 := \sum_{l=1}^{T}   L_{\Psi^{(l)}} \frac{  (1 + \|W\|_\infty)  C_\chi  L_{\Phi^{{(l)}}} B_2^{(l-1)} \frac{1}{\sqrt{2}}    }{\cmin  } \prod_{l' = l+1}^{T} K^{(l')} \nonumber
\\
& \Omega_3 := \sum_{l=1}^{T}   L_{\Psi^{(l)}} \frac{ 2  L_W( L_{\Phi^{(l)}} B_1^{(l-1)} + \|\Phi(0,0)\|_\infty)  }{\cmin  } +  L_{\Psi^{(l)}} \frac{    L_{\Phi^{(l)}} Z_1^{(l-1)}  (1+ \|W\|_\infty) }{ \cmin  } \prod_{l' = l+1}^{T} K^{(l')}  \nonumber \\
& \Omega_4 : = \sum_{l=1}^{T}   L_{\Psi^{(l)}} \frac{ 2  L_W L_{\Phi^{(l)}} B_2^{(l-1)}   }{\cmin}   \prod_{l' = l+1}^{T} K^{(l')} + L_{\Psi^{(l)}} \frac{    2L_{\Phi^{(l)}} Z_2^{(l-1)}   (1+ \|W\|_\infty) }{\cmin  }  \prod_{l' = l+1}^{T} K^{(l')} \nonumber
\\
& \Omega_5 :=  \sum_{l=1}^{T}   L_{\Psi^{(l)}} \frac{    2L_{\Phi^{(l)}} Z_3^{(l-1)} L_f (1+ \|W\|_\infty)}{\cmin  }   \prod_{l' = l+1}^{T} K^{(l')} \nonumber \\
&  \Omega_6   := \sum_{l=1}^{T}   L_{\Psi^{(l)}} \frac{ L_{\Phi^{{(l)}}}  B_1^{(l-1)} + \|\Phi^{{(l)}}(0,0)\|_\infty }{ \mathrm{d}_{\mathrm{min}}}  \nonumber \\ 
& +  L_{\Psi^{(l)}} \frac{ L_{\Phi^{{(l)}}}  B_1^{(l-1)} + \|\Phi^{{(l)}}(0,0)\|_\infty }{ \mathrm{d}_{\mathrm{min}}}  \big(\sqrt{2}\cmax+\zeta\cl \big)  \prod_{l' = l+1}^{T} K^{(l')} \nonumber 
\\
&  \Omega_7 := \sum_{l=1}^{T}   L_{\Psi^{(l)}}  \frac{  L_{\Phi^{{(l)}}}
B_2^{(l-1)}}{ \mathrm{d}_{\mathrm{min}} } +   L_{\Psi^{(l)}}  \frac{L_{\Phi^{{(l)}}} B_2^{(l-1)}}{ \mathrm{d}_{\mathrm{min}}} {\big(\sqrt{2}\cmax+\zeta\cl \big)}  \prod_{l' = l+1}^{T} K^{(l')} \nonumber
\\
& \Omega_{8} := \sum_{l=1}^{T}   L_{\Psi^{(l)}} \frac{ L_{\Phi^{{(l)}}}  B_1^{(l-1)} + \|\Phi^{{(l)}}(0,0)\|_\infty }{ \mathrm{d}_{\mathrm{min}}}     \prod_{l' = l+1}^{T} K^{(l')} \nonumber
\\
& \Omega_{9} := \sum_{l=1}^{T}   L_{\Psi^{(l)}}  \frac{  L_{\Phi^{{(l)}}} B_2^{(l-1)}}{ \mathrm{d}_{\mathrm{min}} }   \prod_{l' = l+1}^{T} K^{(l')} \nonumber
\\
& \Omega_{10} := \sum_{l=1}^{T}   L_{\Psi^{(l)}} \frac{ L_{\Phi^{{(l)}}}  B_1^{(l-1)} + \|\Phi^{{(l)}}(0,0)\|_\infty }{ \mathrm{d}_{\mathrm{min}}}  \frac{ \zeta\cl \big(\sqrt{\log (C_\chi)} +  \sqrt{D_\chi}\big) }{1}   \prod_{l' = l+1}^{T} K^{(l')} \nonumber
\\
    & \Omega_{11} := \sum_{l=1}^{T}   L_{\Psi^{(l)}}  \frac{L_{\Phi^{{(l)}}} B_2^{(l-1)}}{ \mathrm{d}_{\mathrm{min}}}  \frac{ \zeta\cl \big(\sqrt{\log (C_\chi)} +  \sqrt{D_\chi}\big) }{1}   \prod_{l' = l+1}^{T} K^{(l')} \nonumber \\ 
   & \textcolor{black}{ \Omega_{12} :=  \sum_{l=1}^T  L_{\Psi^{(l)}}  {\left( \frac{\|\Phi^{(l)}(f^{(l-1)},f^{(l-1)})\|}{\cmin} +  L_\Phi   \|f^{(l-1)}\|_\infty  + \|\Phi(0,0)\|_\infty\right)}   \prod_{l' = l+1}^{T} K^{(l')} +  \prod_{l = 1}^{T} K^{(l)}} \nonumber
\end{align}
\endgroup
\end{proof}

In the following result, we generalize \Cref{{thm:convwithoutpooling}} to accommodate for a MPNN with a pooling layer applied after the its final layer. 

\begin{corollary}
\label{cor:convAfterPooling}
Let $(\chi,d, \P) $ be a metric-probability space and
$W$ be an admissible graphon. 
 Consider a graph-signal $\{G,\mathbf{f}\} \sim_\nu \{W,f\}$ with $N$ nodes and corresponding graph features, where $N$ satisfies \Cref{eq:lowerBoundGraphSizeN}. If the event $\mathcal{F}_{\rm Lip}^p$ from \Cref{lemma:C2} occurs, then for every MPNN $\Theta$ and every $f:\chi \to \mathbb{R}^{F}$ with Lipschitz constant $L_f$, 
\[
\begin{aligned}
\Big\| \Theta_G^P(\mathbf{f}) - \Theta_W^P(f) \Big\|_\infty^2 & \leq     S_1 \frac{\log(2/p)}{N^{\frac{1}{D_\chi+1}}} + S_2 \frac{\log(2/p)}{N} \\
& +  \left( \textcolor{black}{\Omega_{12}\varepsilon + } S_3 \frac{1}{N^{\frac{1}{2(D_\chi+1)}}} + S_4 \frac{ \sqrt{\log(N)}}{N^{\frac{1}{2(D_\chi+1)}}} + S_5 \frac{1}{\sqrt{N}} +  S_6 \frac{\sqrt{\log(N)}}{\sqrt{N}} 
\right)^2,
\end{aligned}
\]
where the constants are defined in \Cref{eq:constantsS1toS4} below.
\end{corollary}
\begin{proof}
We follow the lines of the proof of Corollary B.17 in \cite{maskey2022generalization}. We have
\[
\begin{aligned}
 & \Big\| \Theta_G^P(\mathbf{f}) - \Theta_W^P(f) \Big\|_\infty  
 \\
& \leq \d\big(\Theta_G(\mathbf{f}), \Theta_W(f)\big) \\
& +   N^{-\frac{1}{2(D_\chi+1)}}\Bigg(2 (Z_1^{(T)}+Z_2^{(T)}\|f\|_\infty+Z_3^{(T)}L_f) + \frac{C_\chi}{\sqrt{2}} (B_1^{(T)} + B_2^{(T)} \|f\|_\infty )  
 \\
& \cdot \sqrt{\log(C_\chi) + \frac{D_\chi}{2(D_\chi+1)} \log(N) + \log(2/p)}   \Bigg).
\end{aligned}
\]
With \Cref{thm:convwithoutpooling}, we get
\begin{equation}
    \label{eq:corC14-1}
\begin{aligned}
& \big\| \Theta_G^P(\mathbf{f}) - \Theta_W^P(f)  \big\|_\infty \\  & \leq  \Omega_1 \frac{ \sqrt{\log(C_\chi)+\frac{3\big(D_\chi +2/3\big)}{2(D_\chi +1)} \log(N)+\log(2/p)} }{ N^{\frac{1}{2(D_\chi+1)}}}  \\ & +  \Omega_2 \|f\|_\infty \frac{\sqrt{\log(C_\chi)+\frac{3\big(D_\chi +2/3\big)}{2(D_\chi +1)} \log(N)+\log(2/p)} }{ N^{\frac{1}{2(D_\chi+1)}} }   \\
& +  \Omega_3  \frac{1}{N^{\frac{1}{2(D_\chi + 1)}}} + \Omega_4  \frac{\|f\|_\infty}{N^{\frac{1}{2(D_\chi + 1)}}} +  \Omega_5 \frac{ L_f } { N^{\frac{1}{2(D_\chi+1)}}}  + \Omega_6 \frac{\sqrt{\log(2/p)}}{\sqrt{N}} +  \Omega_{7} \|f\|_\infty \frac{\sqrt{\log(2/p)}}{\sqrt{N}}\\ & + \Omega_{8} \frac{\sqrt{\log(N)}}{\sqrt{N}}  + \Omega_{9} \|f\|_\infty \frac{\sqrt{\log(N)}}{\sqrt{N}} \\
& + \Omega_{10} \frac{1}{\sqrt{N}} + \Omega_{11} \frac{ \|f\|_\infty}{\sqrt{N}}
 + \textcolor{black}{\Omega_{12}\varepsilon + }
 N^{-\frac{1}{2(D_\chi+1)}}\Bigg(2 (Z_1^{(T)}+Z_2^{(T)}\|f\|_\infty+Z_3^{(T)}L_f) \\ & + \frac{C_\chi}{\sqrt{2}} (B_1^{(T)} + B_2^{(T)} \|f\|_\infty )   \\
& \cdot  \sqrt{\log(C_\chi) + \frac{D_\chi}{2(D_\chi+1)} \log(N) + \log(2/p)}   \Bigg).
\end{aligned}
\end{equation}
Now we use the inequality
\[
\left(\sum_{i=1}^n a_i\right)^2 \leq n \sum_{i=1}^n a_i^2 
\]
for any $a_i \in \mathbb{R}_+$, $i = 1,\ldots, N$, and square both sides of \Cref{eq:corC14-1} to get three summands. The first two summands depend on $p$.
\begin{equation}
    \label{eq:thmC10-constants} 
\begin{aligned}
 & \big\| \Theta_G^P(\mathbf{f}) - \Theta_W^P(f)  \big\|_\infty^2  \\ & \leq 
3 \big( \Omega_1 + \Omega_2 \|f\|_\infty + \frac{C_\chi}{\sqrt{2}}(B_1^{T} + B_2^{T}\|f\|_\infty)\big)^2 \frac{\log(2/p)}{N^{\frac{1}{D_\chi + 1}}}  \\ & +3 (\Omega_6 + \Omega_7 \|f\|_\infty )^2
 \frac{\log(2/p)}{N } \\
& +  3\Bigg(\Omega_1 \frac{ \sqrt{\log(C_\chi)+\frac{3\big(D_\chi +2/3\big)}{2(D_\chi +1)} \log(N)} }{ N^{\frac{1}{2(D_\chi+1)}}}  +  \Omega_2 \|f\|_\infty \frac{\sqrt{\log(C_\chi)+\frac{3\big(D_\chi +2/3\big)}{2(D_\chi +1)} \log(N)} }{ N^{\frac{1}{2(D_\chi+1)}} }   \\
& +  \Omega_3 \frac{ 1 } { N^{\frac{1}{2(D_\chi+1)}}} +  \Omega_4 \frac{ \|f\|_\infty } { N^{\frac{1}{2(D_\chi+1)}}} +  \Omega_5 \frac{ L_f } { N^{\frac{1}{2(D_\chi+1)}}} \\
&  + \Omega_8 \frac{\sqrt{\log(N)}}{\sqrt{N}} + \Omega_9 \|f\|_\infty \frac{\sqrt{\log(N)}}{\sqrt{N}} \\
& + \Omega_{10} \frac{1}{\sqrt{N}} + \Omega_{11}\frac{ \|f\|_\infty }{\sqrt{N}} + \textcolor{black}{\Omega_{12}\varepsilon}
\\
 & + 
 N^{-\frac{1}{2(D_\chi+1)}}\Bigg(2 (Z_1^{(T)}+Z_2^{(T)}\|f\|_\infty+Z_3^{(T)}L_f) + \frac{C_\chi}{\sqrt{2}} (B_1^{(T)} + B_2^{(T)} \|f\|_\infty )   \\
& \cdot  \sqrt{\log(C_\chi) + \frac{D_\chi}{2(D_\chi+1)} \log(N)}   \Bigg)\Bigg)^2 \\
& =:  H_2 \log(2/p) + H_1,
\end{aligned}
\end{equation}
 where we separate the terms depending on the failure probability $p$ and the others to facilitate the following proofs in expectation.
 
We can further simplify this and separate the different terms depending on powers of $N$ to get 
\begin{equation}
    \begin{aligned}
    & \leq   
3 \big(\Omega_1+\frac{C_\chi}{\sqrt{2}}B_1^{T} + (\Omega_2+B_2^{T}) \|f\|_\infty \big)^2 \frac{\log(2/p)}{N^{\frac{1}{D_\chi + 1}}}  +3 \big(\Omega_6 + \Omega_7 \|f\|_\infty\big)^2
 \frac{\log(2/p)}{N } \\
& +  3\Bigg(\frac{\Omega_1  \sqrt{\log(C_\chi)} + \Omega_2 \|f\|_\infty \sqrt{\log(C_\chi)}+
\Omega_3 + \Omega_4  \|f\|_\infty  + \Omega_5 L_f} { N^{\frac{1}{2(D_\chi+1)}}} \\
& + \frac{\Bigg(2 (Z_1^{(T)}+Z_2^{(T)}\|f\|_\infty+Z_3^{(T)}L_f) + C_\chi\sqrt{2}^{-1} (B_1^{(T)} + B_2^{(T)} \|f\|_\infty )     \cdot \Big( \sqrt{\log(C_\chi)}    \Big)   \Bigg) }{ N^{\frac{1}{2(D_\chi+1)}}}
\\ & +  \frac{C_\chi\sqrt{2}^{-1} (B_1^{(T)} + B_2^{(T)} \|f\|_\infty ) \sqrt{\frac{D_\chi}{2(D_\chi+1)}} \sqrt{ \log(N)} }{N^{\frac{1}{2(D_\chi+1)}}}   + \Omega_1 \frac{\sqrt{\frac{3\big(D_\chi +2/3\big)}{2(D_\chi +1)}} \sqrt{ \log(N)} }{ N^{\frac{1}{2(D_\chi+1)}}}\\& + \Omega_2 \|f\|_\infty\frac{ \sqrt{\frac{3\big(D_\chi +2/3\big)}{2(D_\chi +1)}} \sqrt{ \log(N)}  }{ N^{\frac{1}{2(D_\chi+1)}} } 
\\ & + \Omega_8 \frac{\sqrt{\log(N)}}{\sqrt{N}} + \Omega_9 \|f\|_\infty \frac{\sqrt{\log(N)}}{\sqrt{N}}     + \Omega_{10} \frac{1}{\sqrt{N}} + \Omega_{11}\frac{ \|f\|_\infty }{\sqrt{N}}
 \textcolor{black}{ +\Omega_{12}\varepsilon}\Bigg)^2 \\
& =: S_1 \frac{\log(2/p)}{N^{\frac{1}{2(D_\chi+1)}}} + S_2 \frac{\log(2/p)}{N} \\& + \left( \textcolor{black}{\Omega_{12}\varepsilon + }  S_3 \frac{1}{N^{\frac{1}{2(D_\chi+1)}}} + S_4 \frac{ \sqrt{\log(N)}}{N^{\frac{1}{2(D_\chi+1)}}} + S_5 \frac{1}{\sqrt{N}} +  S_6 \frac{\sqrt{\log(N)}}{\sqrt{N}} 
\right)^2,
    \end{aligned}
\end{equation}
where
\begin{equation}
    \label{eq:constantsS1toS4}
    \begin{aligned}
     S_1 &  := 3 \big(\Omega_1+\frac{C_\chi}{\sqrt{2}}B_1^{T} + (\Omega_2+B_2^{T}) \|f\|_\infty \big)^2 \\
     S_2  & := 3 \big(\Omega_6  + \Omega_7\|f\|_\infty\big)^2 \\
   S_3  &  := \sqrt{3} \Omega_1  \sqrt{\log(C_\chi)} + \Omega_2 \|f\|_\infty \sqrt{\log(C_\chi)}+
\Omega_3 + \Omega_4  \|f\|_\infty  + \Omega_5 L_f  \\
& + \Bigg(2 (Z_1^{(T)}+Z_2^{(T)}\|f\|_\infty+Z_3^{(T)}L_f) + C_\chi\sqrt{2}^{-1} (B_1^{(T)} + B_2^{(T)} \|f\|_\infty )     \cdot \Big( \sqrt{\log(C_\chi)}    \Big)   \Bigg) 
\\
  S_4 & :=     C_\chi\sqrt{2}^{-1} (B_1^{(T)} + B_2^{(T)} \|f\|_\infty ) \sqrt{\frac{D_\chi}{2(D_\chi+1)}}      + \Omega_1  \sqrt{\frac{3\big(D_\chi +2/3\big)}{2(D_\chi +1)}}  \\ & + \Omega_2 \|f\|_\infty  \sqrt{\frac{3\big(D_\chi +2/3\big)}{2(D_\chi +1)}}  \\
  S_5 & := \Omega_{10} + \Omega_{11}\|f\|_\infty \\
  S_6 & := \Omega_{8} + \Omega_{9}\|f\|_\infty.
    \end{aligned}
\end{equation}
\end{proof}

\section{Generalization Bound}
\label{appendix: Proof Gen Bound}

The following lemma is akin to Lemma B.10. in \citep{maskey2022generalization}, and bounds determistically the norm of the output of a graph MPNN.

\begin{lemma}
\label{lemma:DeterministicMPNNBound}
Let $(\chi,d, \P) $ be a metric-probability space,
$W$ be an admissible graphon and consider a MPNN $\Theta = \big((\Phi^{(l)})_{l=1}^T, (\Psi^{(l)})_{l=1}^T \big)$. Consider a   metric-space signal  $f: \chi \to \mathbb{R}^F$ with $\|f\|_\infty < \infty$.
 Consider a graph-signal $\{G,\mathbf{f}\} \sim_\nu \{W,f\}$ with $N$ nodes and corresponding graph features.  
Then, 
\[
 \| \Theta_G(\mathbf{f})\|_{\infty;\infty} \leq A' + A'' \|f\|_\infty,
\]
where
 {
\[
\begin{aligned}
A' & = \sum_{l=1}^{T} \Big(L_{\Psi^{(l)}}\|\Phi^{(l)}(  0,0) \|_\infty
+ \|\Psi^{(l)} (0,0 )\|_\infty
 \Big) 
 \prod_{l' = l+1}^T 
   L_{\Psi^{(l')}} \max \big(1, L_{\Phi^{(l')}}\big) 
   \end{aligned}
\]
and
\[
A'' = \prod_{l=1}^T 
L_{\Psi^{(l)}} \max \big(1, L_{\Phi^{(l)}}\big).
\]}

\end{lemma}
\begin{proof}
Let $l=0,\ldots, T-1$. We have
\[
\|\mathbf{f}^{(l+1)}\|_{\infty; \infty} = \max_{i=1, \ldots, N} \|\mathbf{f}_i^{(l+1)}\|_\infty,
\]
where $ \mathbf{f}_i^{(l+1)} = \Psi^{(l+1)} ( \mathbf{f}_i^{(l)}, \mathbf{m}^{(l+1)}_i ) $ with $ \mathbf{m}^{(l+1)}_i = M_\A \big(  \Phi^{(l+1)}( \mathbf{f}^{(l)}, \mathbf{f}^{(l)} )\big)(X_i) $. By using the Lipschitz continuity of $\Psi^{(l+1)}$, we get  
\begin{equation}
    \label{eq:lemmab19-1}
    \begin{aligned}
\|\mathbf{f}_i^{(l+1)}\|_\infty & 
 \leq   \|\Psi^{(l+1)} ( \mathbf{f}_i^{(l)}, \mathbf{m}^{(l+1)}_i ) - \Psi^{(l+1)} (0,0 )\|_\infty +  \|\Psi^{(l+1)} (0,0 )\|_\infty  \\
& \leq 
  L_{\Psi^{(l+1)}}\max (\| \mathbf{f}_i^{(l)} \|_\infty, \|\mathbf{m}_i^{(l+1)}\|_\infty) 
+ \|\Psi^{(l+1)} (0,0 )\|_\infty 
\end{aligned}
\end{equation}

For the message term, we calculate
{
\[
\begin{aligned}
  \| \mathbf{m}^{(l+1)}_i \|_\infty & = \left\|  \frac{1}{ \sum_{j=1}^N \A(X_i, X_j) } \sum_{j=1}^N \A(X_i, X_j) \Phi^{(l+1)}(  \mathbf{f}_i^{(l)}, \mathbf{f}_j^{(l)} )  \right\|_\infty \\
& \leq \max_{j=1, \ldots, N}\|\Phi^{(l+1)}(  \mathbf{f}_i^{(l)}, \mathbf{f}_j^{(l)} )  \|_\infty
\end{aligned}
\]
    }

where the inequality follows from Cauchy-Schwarz inequality and the assumption of not having any isolated nodes. We have for every $i=1, \ldots, N$,
 \[
 \begin{aligned}
 \|\Phi^{(l+1)}(  \mathbf{f}_i^{(l)}, \mathbf{f}_j^{(l)} )  \|_\infty & = \|\Phi^{(l+1)}(  \mathbf{f}_i^{(l)}, \mathbf{f}_j^{(l)} ) - \Phi^{(l+1)}(  0,0) + \Phi^{(l+1)}(  0,0) \|_\infty \\
 & \leq  \|\Phi^{(l+1)}(  \mathbf{f}_i^{(l)}, \mathbf{f}_j^{(l)} ) - \Phi^{(l+1)}(  0,0) \|_\infty+ \|\Phi^{(l+1)}(  0,0) \|_\infty   \\ 
   & \leq  
   L_{\Phi^{(l+1)}}   \max\big( \| \mathbf{f}_i^{(l)} \|_\infty, \|\mathbf{f}_j^{(l)}\|_\infty \big) +  \|\Phi^{(l+1)}(  0,0) \|_\infty.
 \end{aligned}
 \]
 Hence,
\begin{equation}
    \label{eq:lemmab19-2}
    \begin{aligned}
    \|\mathbf{m}_i^{(l+1)}\|_\infty &
 \leq  \max_{j=1, \ldots, N}L_{\Phi^{(l+1)}}    \|\mathbf{f}_j^{(l)}\|_\infty  +  \|\Phi^{(l+1)}(  0,0) \|_\infty
    \\
    & \leq  L_{\Phi^{(l+1)}}    \|\mathbf{f}^{(l)}\|_{\infty;\infty}  +  \|\Phi^{(l+1)}(  0,0) \|_\infty.
    \end{aligned}
\end{equation}
By \Cref{eq:lemmab19-1} and \Cref{eq:lemmab19-2}, we have
\[
\begin{aligned}
 \|\mathbf{f}^{(l+1)}\|_{\infty; \infty} & \leq   \max_{i=1, \ldots, N}   L_{\Psi^{(l+1)}}\max (\| \mathbf{f}_i^{(l)} \|_\infty, \|\mathbf{m}_i^{(l+1)}\|_\infty) 
+ \|\Psi^{(l+1)} (0,0 )\|_\infty  \\
 &  \leq   \max_{i=1, \ldots, N}   L_{\Psi^{(l+1)}}\max \big(\| \mathbf{f}_i^{(l)} \|_\infty, (L_{\Phi^{(l+1)}}    \|\mathbf{f}^{(l)}\|_{\infty;\infty}  +  \|\Phi^{(l+1)}(  0,0) \|_\infty)\big) 
\\ &+ \|\Psi^{(l+1)} (0,0 )\|_\infty \\
& =    L_{\Psi^{(l+1)}}\max \big(\| \mathbf{f}^{(l)} \|_{\infty;\infty}, L_{\Phi^{(l+1)}}    \|\mathbf{f}^{(l)}\|_{\infty;\infty}  +  \|\Phi^{(l+1)}(  0,0) \|_\infty\big) 
+ \|\Psi^{(l+1)} (0,0 )\|_\infty \\
& \leq  L_{\Psi^{(l+1)}} \max \big(1, L_{\Phi^{(l+1)}}\big)\| \mathbf{f}^{(l)} \|_{\infty;\infty}   + L_{\Psi^{(l+1)}}\|\Phi^{(l+1)}(  0,0) \|_\infty
+ \|\Psi^{(l+1)} (0,0 )\|_\infty.
\end{aligned}
\]

 Hence, by $\|\mathbf{f}\|_{2;\infty}^2 \leq \|f\|_\infty^2$ and \Cref{lemma:RecRecGen}, we have
 \[
 \begin{aligned}
  \|\mathbf{f}^{(T)}\|_{\infty; \infty}
 & \leq \sum_{l=1}^{T} \Big(L_{\Psi^{(l)}}\|\Phi^{(l)}(  0,0) \|_\infty
+ \|\Psi^{(l)} (0,0 )\|_\infty
 \Big) 
 \prod_{l' = l+1}^T 
   L_{\Psi^{(l')}} \max \big(1, L_{\Phi^{(l')}}\big) .  \\
& +  \|f\|_\infty\prod_{l=1}^T 
L_{\Psi^{(l)}} \max \big(1, L_{\Phi^{(l)}}\big).
 \end{aligned}
 \]
\end{proof}

\begin{theorem}
\label{thm:unifExpValue}
Let $(\chi,d, \P) $ be a metric-probability space and
$W$ be an admissible graphon.   Consider a graph-signal $\{G,\mathbf{f}\} \sim_{ \nu} \{W,f\}$ with $N$ nodes and corresponding graph features. Then,
for every $f:\chi \to \mathbb{R}^{F}$ with Lipschitz constant $L_f$, 
\[
\begin{aligned}
&  \E_{X_1, \ldots, X_N \sim \mu^N} \left[\sup_{\Theta \in \mathrm{Lip}_{L,B}} \left\| 
\Theta^P_G(\mathbf{f}) - 
\Theta^P_W(f) \right\|_\infty^2 \right] \\&  \leq  4(1 + \sqrt{\pi}) \Bigg(  T_1 \frac{1+\log(N)}{N^{\frac{1}{D_\chi+1} }\cmin^2} + T_2 \frac{1+\log(N)}{N \cmin^2}
 + \textcolor{black}{T_3 \frac{1}{\cmin^2} \varepsilon} \Bigg) +  \mathcal{O} \left(   \exp(-N)\right) .
\end{aligned}
\]
where the constants are defined in \Cref{eq:constantsS1toS4}. 
\end{theorem}
\begin{proof}
The proof follows the lines of the proof of Theorem B.18 in \citep{maskey2022generalization}. For any $p >0$, we have with probability at least $1-4p$ for every $\Theta \in \mathrm{Lip}_{L,B}$, by \Cref{cor:convAfterPooling}, that
\[
\big\| \Theta_G^P(\mathbf{f}) - \Theta_W^P(f) \big\|_{\infty}^2  \leq  H_1 + H_2 \log(2/p)
\]
if \Cref{eq:lowerBoundGraphSizeN} holds, where   $H_1$ and $H_2$ are specified in the proof of  \Cref{cor:convAfterPooling}, Equation \Cref{eq:thmC10-constants}.
Further, for every $p \in (0,1/4)$, we consider $k > 0$ such that $ p = 2 \exp(-k^2) $. This means, if $p$ respectively $k$ satisfies \Cref{eq:lowerBoundGraphSizeN}, we have with probability at least $1- 8\exp(-k^2) $ for every $\Theta \in \mathrm{Lip}_{L,B}$,
\[
\big\| \Theta_G^P(\mathbf{f}) - \Theta_W^P(f) \big\|_{\infty}^2  \leq  H_1 + H_2k.
\]

If $k$ does not satisfy \Cref{eq:lowerBoundGraphSizeN}, we get
\[
k > N_0 = D_1 + D_2 \sqrt{N},
\]
where $D_1 \in \mathbb{R}$ and $D_2 > 0$ are the matching constants in \Cref{eq:lowerBoundGraphSizeN}. 
By \citet[Corollary B.9]{maskey2022generalization} and  \Cref{lemma:DeterministicMPNNBound}, we get in this case 
\begin{equation}
    \label{eq:defq(N)}
\begin{aligned}
\big\| \Theta_G^P(\mathbf{f}) - \Theta_W^P(f) \big\|_\infty & = 
 \left\| \frac{1}{N}\sum_{i=1}^{N}
  \Theta_G(\mathbf{f})_i - \int_\chi\Theta_W(f)(y) d\P(y)\right\|_\infty \\
  & \leq  
  \|\Theta_G(\mathbf{f})\|_{\infty;\infty} +  \Big\| \int_\chi\Phi_W(f)(y) d\P(y)\Big\|_\infty \\
& \leq  
  \| \Theta_G(\mathbf{f})\|_{\infty;\infty} + \|\Theta_W(f) \|_\infty \\
  & \leq  A' + A'' \|f\|_\infty +
B_1^{(T)} + \|f\|_\infty B_2^{(T)}=: q,
\end{aligned}
\end{equation}
where the first inequality holds by applying the triangle inequality.

We then calculate the expected value by partitioning the integral over the event space into the following sum. 
\begin{equation}
\label{eq:C18-1}
\begin{aligned}
     &\E_{X_1,\ldots,X_N \sim \mu^N} \left[\sup_{\Theta \in \mathrm{Lip}_{L,B}}\big\| \Theta_G^P(\mathbf{f}) - \Theta_W^P(f) \big\|_\infty^2 \right]  \\
  \leq  &   \sum_{k = 0}^{N_0} \mathbb{P}\big(H_1 + H_2k   \leq\sup_{\Theta \in \mathrm{Lip}_{L,B}} \big\| \Theta_G^P(\mathbf{f}) - \Theta_W^P(f) \big\|_\infty^2 
< H_1  + H_2(k+1) \big)
 \cdot \big(  H_1  + H_2(k+1) \big) \\
+ &\sum_{k = N_0}^{\infty} \mathbb{P}\big(H_1 + H_2k  \leq \sup_{\Theta \in \mathrm{Lip}_{L,B}}\big\| \Theta_G^P(\mathbf{f}) - \Theta_W^P(f) \big\|_\infty^2 
< H_1  + H_2(k+1) \big)    \cdot q^2
\\
\end{aligned}
\end{equation}
To bound  the second sum, note that it is a finite sum, since $\big\| \Theta_G^P(\mathbf{f}) - \Theta_W^P(f) \big\|_\infty^2$ is bounded by $q$, which is defined in \Cref{eq:defq(N)}. 
The summands are zero if $H_1 + H_2k   > q^2$, which holds for $k > \sqrt{\frac{q^2}{H_2}}$. Hence, we calculate with the right-hand-side of \Cref{eq:C18-1} by
\begin{equation}
\begin{aligned}
& \leq 4 \sum_{k=0}^{N_0} 2 \exp(-k^2) \cdot \big(  H_1  + H_2(k+1)   \big) 
 s + \sum_{k=N_0}^{\left\lceil \sqrt{\frac{q^2}{H_2}} \right\rceil}4 \exp(-N_0^2) \cdot q^2
\\
& \leq 4 \int_0^{\infty} 2 \exp(-k^2) \cdot \big(  H_1  + H_2(k+1)  \big)
 + 4 \exp(-N_0^2) q^2\left\lceil \sqrt{\frac{q^2}{H_2}}\right\rceil
,
\end{aligned}
\end{equation}
where 
$q$ is constant in $N$ as defined above.
The first term on the right-hand-side is bounded by using 
\[
 \;\int_0^\infty 2(t+1) e^{-t^2} dt , \; \int_0^\infty  2e^{-t^2} dt \leq 1 + \sqrt{\pi}.
\]
For the second term we remember that $N_0 = D_1 + D_2\sqrt{N}$. Hence,
\begin{align*}
&\E_{X_1,\ldots,X_N \sim \mu^N} \left[\sup_{\Theta \in \mathrm{Lip}_{L,B}}\big\| \Theta_G^P(\mathbf{f}) - \Theta_W^P(f) \big\|_\infty^2 \right]\\
&\leq 4(1  +\sqrt{\pi}) (H_1 + H_2 ) + \mathcal{O}(\exp(-N)N^{ {\frac{3}{2}}T-\frac{3}{2}}) \\
& = 4(1 + \sqrt{\pi}) \\ & \cdot \Bigg(  S_1 \frac{1}{N^{\frac{1}{D_\chi+1}}} + S_2 \frac{1}{N}  +\left( \textcolor{black}{\Omega_{12}\varepsilon + } S_3 \frac{1}{N^{\frac{1}{2(D_\chi+1)}}} + S_4 \frac{ \sqrt{\log(N)}}{N^{\frac{1}{2(D_\chi+1)}}} + S_5 \frac{1}{\sqrt{N}} +  S_6 \frac{\sqrt{\log(N)}}{\sqrt{N}} 
\right)^2 \Bigg) 
\\ &  + \mathcal{O}(\exp(-N))
\end{align*} 
by Definition of the constants $S_1, \ldots, S_6$, see \eqref{eq:constantsS1toS4}. Note that all terms $S_1, \ldots, S_6$ and $\Omega_{12}$ depend linearly on $\frac{1}{\cmin^2}$ and set
\begin{equation}
\label{eq:Ts}
    \begin{aligned}
    T_1 & := \cmin^2(S_1 + 5S_3^2 + 5S_4^2 ) \\
    T_2 & := \cmin^2(S_2 + 5S_5^2 + 5S_6^2v) \\
    T_3 & := \cmin^2 5 \Omega_{12}^2.
    \end{aligned}
\end{equation}
\end{proof}

We can easily generalize \Cref{thm:unifExpValue} to the case where $\{G,\mathbf{f}\} \sim_{\alpha; \nu} \{W,f\}$ for any $\alpha > 0$. This is due to fact that the scaled graphon $N^{-\alpha} W$  admits 
\[
\int_\chi N^{-\alpha} W(x,y) d\mu(y) \geq N^{-\alpha} \cmin
\]
if $\mathrm{d}_{W} \geq \cmin$.

\begin{corollary}
\label{cor:unifExpValue}
Let $(\chi,d, \P) $ be a metric-probability space and
$W$ be a kernel.   Consider a graph-signal $\{G,\mathbf{f}\} \sim_{\alpha; \nu} \{W,f\}$ with $N$ nodes and corresponding graph features. Then,
for every $f:\chi \to \mathbb{R}^{F}$ with Lipschitz constant $L_f$, 
\[
\begin{aligned}
&  \E_{X_1, \ldots, X_N \sim \mu^N} \left[\sup_{\Theta \in \mathrm{Lip}_{L,B}} \left\| 
\Theta^P_G(\mathbf{f}) - 
\Theta^P_W(f) \right\|_\infty^2 \right] \\&  \leq  4(1 + \sqrt{\pi}) \Bigg(  T_1 \frac{\big(1+\log(N)\big)N^{2\alpha}}{N^{\frac{1}{D_\chi+1} }} + T_2 \frac{\big(1+\log(N)\big) N^{2 \alpha}}{N}
 + \textcolor{black}{T_3 \varepsilon} \Bigg) +  \mathcal{O} \left(   \exp(-N)\right) .
\end{aligned}
\]
where the constants are defined in \Cref{eq:constantsS1toS4}. 
\end{corollary}

\subsection{Generalization}
In this subsection, we present the proof of \Cref{thm:main_gen_bound_deformed_graphon}. For better presentation, we reformulate \Cref{thm:main_gen_bound_deformed_graphon}.

\begin{lemma}[Proposition A.6 in \cite{Vaart}, Bretagnolle-Huber-Carol inequality]
\label{lemma:BHCineq}
If the random vector $(m_1, \ldots m_\Gamma)$ is multinomially distributed with parameters $m$ and $\gamma_1, \ldots, \gamma_\Gamma$, then 
\[
\mathbb{P}\left( \sum_{i=1}^\Gamma| m_i - m\gamma_i | \geq 2 \sqrt{m} \lambda \right)
\leq 2^\Gamma \exp(-2\lambda^2)
\]
for any $\lambda > 0$. 
\end{lemma}

\begin{theorem}
\label{thm:reformulatedTheorem2}
Let  $\{(W^j, f^j)\}_{j=1}^\Gamma$ be RGSMs on corresponding metric-probability spaces $\{(\chi^j, d^j, \mu^j)\}_{j=1}^\Gamma$. 
Let $\mathcal{T}=\big((G_1, \mathbf{f}_1,y_1), \ldots, (G_m, \mathbf{f}_m,y_m)\big)\sim \mu^m$ be a dataset of labelled graph-signals.  Then,  
\[
\begin{aligned}
  &   \E_{\mathcal{T} \sim \mu^m} \left[\sup_{\Theta \in \mathrm{Lip}_{L,B}}  \left( \frac{1}{m} \sum_{i=1}^m  \mathcal{L}(\Theta_{G_i}^P(\mathbf{f}_i), y_i) - \E_{(G, \mathbf{f},y)\sim \mu}\left[ \mathcal{L}(\Theta_{G}^P(\mathbf{f}), y) \right] \right)^2 \right]  \leq 2^\Gamma \frac{8 \|\mathcal{L}\|_\infty^2}{m} \pi \\
  & + \frac{4(1+\sqrt{\pi})}{ m} 2^\Gamma  \Gamma  \sum_{j=1}^\Gamma \gamma_j L_\mathcal{L}^2 \Bigg(  \sqrt{\pi}  \Big(  T_1^{(j)} \E_{N \sim \nu}\left[\frac{1 + \log(N)}{N^{\frac{1}{D_\chi+1}}}N^{2\alpha}\right] + T_2^{(j)} \E_{N \sim \nu}\left[\frac{1+\log(N)}{N}N^{2\alpha}\right] \\ & + \textcolor{black}{T_3^{(j)} \varepsilon}  \Big)    + \mathcal{O}\left( \E_{N \sim \nu} \left[\exp(-N) \right]\right)\Bigg),
\end{aligned}
\]
where $T_l^{(j)}$ are the according constants from \Cref{thm:unifExpValue} for each class $j$ and are defined in \Cref{eq:Ts}. \end{theorem}
\begin{proof}
The proof is similar to the Proof of Theorem C.7 in \cite{maskey2022generalization}. For completely, we include it with the appropriate modifications.
 
Given $\mathbf{m}=(m_1,\ldots,m_{\Gamma})$ with $\sum_{j=1}^{\Gamma}m_j=m$, recall that $\mathcal{G}^{\mathbf{m}}$ is the space of datasets with fixed number of samples $m_j$ from each class $j=1,\ldots,\Gamma$.  
The probability measure on $\mathcal{G}^{\mathbf{m}}$ is given by $\mu_{\mathcal{G}^{\mathbf{m}}}$.
We denote the conditional choice of the dataset on the choice of $\mathbf{m}$ by 
\[\mathcal{T}_{\mathbf{m}}:=\big\{\{G_i^j, \mathbf{f}_i^j\}_{i=1}^{m_j}\big\}_{j=1}^{\Gamma} \sim \mu_{\mathcal{G}^{\mathbf{m}}}.\]
Given $k\in\mathbb{Z}$, denote by $\mathcal{M}_k$ the set of all $\mathbf{m}=(m_1,\ldots, m_{\Gamma})\in \mathbb{N}_0^{\Gamma}$ with $\sum_{j=1}^{\Gamma}m_j=m$, such that 
$2\sqrt{m} k \leq \sum_{j=1}^\Gamma |m_j - m\gamma_j| <2\sqrt{m} (k+1)$.
Using these notations, we decompose the expected generalization error as follows.  
\begin{equation}
    \label{eq:thmC13-1}
    \begin{aligned}
    & \E_{\mathcal{T} \sim \mu^m} \left[ \sup_{\Theta \in \mathrm{Lip}_{L,B}} \left( \frac{1}{m} \sum_{i=1}^m  \mathcal{L}(\Theta_{G_i}^P(\mathbf{f}_i), y_i) - \E_{(G, \mathbf{f},y)\sim \mu}\left[ \mathcal{L}(\Theta_{G}^P(\mathbf{f}), y) \right] \right)^2 \right] \\
  &  = \E_{\mathcal{T} \sim \mu^m} \left[\sup_{\Theta \in \mathrm{Lip}_{L,B}}  \left( \frac{1}{m} \sum_{j=1}^\Gamma  \sum_{i=1}^{m_j} \mathcal{L}(\Theta_{G_i^j}^P(\mathbf{f}_i^j), y_j) - \E_{(G, \mathbf{f},y)\sim \mu}\left[ \mathcal{L}(\Theta_{G}^P(\mathbf{f}), y) \right] \right)^2 \right]\\
 & = \E_{\mathcal{T}  \sim  \mu^m} \left[\sup_{\Theta \in \mathrm{Lip}_{L,B}}\left(\sum_{j=1}^\Gamma \left( \frac{1}{m}\sum_{i=1}^{m_j} \mathcal{L}(\Theta_{G_i^j}^P(\mathbf{f}_i^j), y_j) - \gamma_j\E_{(G^j, \mathbf{f}^j )\sim \mu_{\mathcal{G}_j}}\left[ \mathcal{L}(\Theta_{G^j}^P(\mathbf{f}^j), y_j) \right] \right) \right)^2 \right] \\
& \leq \sum_{k} \mathbb{P}\big(
\mathbf{m}\in \mathcal{M}_k
\big) \times 
\sup_{\mathbf{m}\in \mathcal{M}_k} \E_{\mathcal{T}_{\mathbf{m}} \sim \mu_{\mathcal{G}^{\mathbf{m}}}} \left[\sup_{\Theta \in \mathrm{Lip}_{L,B}}\left(\sum_{j=1}^\Gamma \left( \frac{1}{m}\sum_{i=1}^{m_j} \mathcal{L}(\Theta_{G_i^j}^P(\mathbf{f}_i^j), y_j) \right.\right.\right.\\  & \quad \ \ \quad \ \ \quad \ \ \quad \ \ \quad \ \ \quad \ \ \quad \ \ \quad \ \ \quad \ \ \quad \ \ \quad \ \  \left.\left.\left. - \frac{1}{m} \sum_{i=1}^{m \gamma_j}\E_{(G^j, \mathbf{f}^j )\sim \mu_{\mathcal{G}_j}}\left[ \mathcal{L}(\Theta_{G^j}^P(\mathbf{f}^j), y_j) \right] \right) \right)^2 \right] 
 \end{aligned}
\end{equation}

We bound the last term of \Cref{eq:thmC13-1} as follows.
For $j=1, \ldots, \Gamma$, if $m_j \leq m\gamma_j$, we add "ghost samples", i.e., we add additional i.i.d. sampled graphs $(G_{m_j}^j, \mathbf{f}_{m_j}^j), \ldots, (G_{m \gamma_j}^j, \mathbf{f}_{m \gamma_j}^j) \sim (W^j, f^j) $. By convention, for any two $l,q\in\mathbb{N}_0$ with $l<q$, we define \[\sum_{j=q}^l c_j = - \sum_{j=l}^q c_j\]
for any sequence $c_j$ of reals, and define $\sum_{j=q}^q c_j=0$.
With these notations, we have 
\begin{equation}
\label{eq:lastthm-2}
\begin{aligned}
& \E_{\mathcal{T}_{\mathbf{m}} \sim \mu_{\mathcal{G}^{\mathbf{m}}}} \left[\sup_{\Theta \in \mathrm{Lip}_{L,B}}\left(\sum_{j=1}^\Gamma \left( \frac{1}{m}\sum_{i=1}^{m_j} \mathcal{L}(\Theta_{G_i^j}^P(\mathbf{f}_i^j),y_j) \right.\right.\right.\\  & \quad \ \ \quad \ \     \quad \ \  \left.\left.\left. - \frac{1}{m} \sum_{i=1}^{m \gamma_j}\E_{(G^j, \mathbf{f}^j )\sim \mu_{\mathcal{G}_j}}\left[ \mathcal{L}(\Theta_{G^j}^P(\mathbf{f}^j), y_j) \right] \right) \right)^2 \right] 
\\
& =  \E_{\mathcal{T}_{\mathbf{m}} \sim  \mu_{\mathcal{G}^{\mathbf{m}}}} \Bigg[\sup_{\Theta \in \mathrm{Lip}_{L,B}}\Bigg(\sum_{j=1}^\Gamma \Bigg( \frac{1}{m}\sum_{i=1}^{m\gamma_j} \mathcal{L}(\Theta_{G_i^j}^P(\mathbf{f}_i^j),y_j) +  \frac{1}{m}\sum_{i=m\gamma_j}^{m_j} \mathcal{L}(\Theta_{G_i^j}^P(\mathbf{f}_i^j), y_j)
 \\ & - \frac{1}{m} \sum_{i=1}^{m \gamma_j}\E_{(G^j, \mathbf{f}^j )\sim \mu_{\mathcal{G}_j}}\big[ \mathcal{L}(\Theta_{G^j}^P(\mathbf{f}^j),y_j) \big] \Bigg) \Bigg)^2 \Bigg] 
\\
& \leq \E_{\mathcal{T}_{\mathbf{m}} \sim  \mu_{\mathcal{G}^{\mathbf{m}}}} \left[\sup_{\Theta \in \mathrm{Lip}_{L,B}}2 \left(\sum_{j=1}^\Gamma \left( \frac{1}{m}\sum_{i=1}^{m \gamma_j} \mathcal{L}(\Theta_{G_i^j}^P(\mathbf{f}_i^j),y_j) \right.\right.\right.\\  & \quad \ \ \quad \ \     \quad \ \ \quad   \left.\left.\left. - \frac{1}{m} \sum_{i=1}^{m \gamma_j}\E_{(G^j, \mathbf{f}^j )\sim \mu_{\mathcal{G}_j}}\left[ \mathcal{L}(\Theta_{G^j}^P(\mathbf{f}^j), y_j) \right] \right) \right)^2 \right]
\\ & + \E_{\mathcal{T}_{\mathbf{m}} \sim  \mu_{\mathcal{G}^{\mathbf{m}}}} \left[2
\left(\sum_{j=1}^\Gamma \left(
\frac{1}{m} |m\gamma_j - m_j| \|\mathcal{L}\|_\infty
\right)\right)^2
\right].
\end{aligned}
\end{equation}
Let us first bound the last term of the above bound. Since any $\mathbf{m\in \mathcal{M}_k}$ satisfies $\sum_{j=1}^\Gamma |m_j - m\gamma_j| <2\sqrt{m} (k+1)$,  we have
\[
\begin{aligned}
  \E_{\mathcal{T}_{\mathbf{m}} \sim \mu_{\mathcal{G}^{\mathbf{m}}}} \left[2
\left(\sum_{j=1}^\Gamma \left(
\frac{1}{m} |m\gamma_j - m_j| \|\mathcal{L}\|_\infty
\right)\right)^2
\right] 
& \leq    \frac{2}{m^2} \|\mathcal{L}\|_\infty^2 \left(\sum_{j=1}^\Gamma |m\gamma_j -m_j | \right)^2
\\
& \leq    \frac{2}{m^2} \|\mathcal{L}\|_\infty^2 4 m (k+1)^2 
  = \frac{8 \|\mathcal{L}\|_\infty^2}{m} (k+1)^2.
\end{aligned}
\]
Hence, by \Cref{lemma:BHCineq}, 
\[
\begin{aligned}
&  \sum_{k} \mathbb{P}\big(
\mathbf{m}\in \mathcal{M}_k
\big) \times
\sup_{\mathbf{m}\in \mathcal{M}_k} \E_{\mathcal{T}_{\mathbf{m}} \sim \mu_{\mathcal{G}^{\mathbf{m}}}} \left[2
\left(\sum_{j=1}^\Gamma \left(
\frac{1}{m} |m\gamma_j - m_j| \|\mathcal{L}\|_\infty
\right)\right)^2
\right] \\
& \leq  \sum_{k} \mathbb{P}\big(
\mathbf{m}\in \mathcal{M}_k
\big) \times \frac{8 \|\mathcal{L}\|_\infty^2}{m} (k+1)^2 \\
&  \leq \sum_{k} 2^\Gamma \exp(-2 k^2)  \frac{8 \|\mathcal{L}\|_\infty^2}{m} (k+1)^2 \\
&  \leq \int_0^\infty2^\Gamma \exp(-2 k^2)  \frac{8 \|\mathcal{L}\|_\infty^2}{m} (k+1)^2 dk\\
& =  2^\Gamma  \frac{8 \|\mathcal{L}\|_\infty^2}{m} \int_0^\infty \exp(-2 k^2)  (k+1)^2  dk \\ &\leq   2^\Gamma \frac{8 \|\mathcal{L}\|_\infty^2}{m} \pi.
\end{aligned}
\]
To bound the first term of the right-hand-side of \Cref{eq:lastthm-2}, we have 
\[
\begin{aligned}
 & \E_{\mathcal{T}_{\mathbf{m}} \sim \mu_{\mathcal{G}^{\mathbf{m}}}} \left[\sup_{\Theta \in \mathrm{Lip}_{L,B}}\left(\sum_{j=1}^\Gamma \left( \frac{1}{m}\sum_{i=1}^{m \gamma_j} \mathcal{L}(\Theta_{G_i^j}^P(\mathbf{f}_i^j),y_j)  \right.\right.\right.\\  & \quad \ \ \quad \ \     \quad \ \  \left.\left.\left. - \frac{1}{m} \sum_{i=1}^{m \gamma_j}\E_{(G^j, \mathbf{f}^j )\sim \mu_{\mathcal{G}_j}}\left[\sup_{\Theta \in \mathrm{Lip}_{L,B}} \mathcal{L}(\Theta_{G^j}^P(\mathbf{f}^j), y_j) \right] \right) \right)^2 \right]\\
\leq &  \Gamma \sum_{j=1}^\Gamma\E_{\mathcal{T}_{\mathbf{m}} \sim \mu_{\mathcal{G}^{\mathbf{m}}}}\left[\sup_{\Theta \in \mathrm{Lip}_{L,B}} \left( \frac{1}{m}\sum_{i=1}^{m\gamma_j} \mathcal{L}(\Theta_{G_i^j}^P(\mathbf{f}^j_i),y_j)  \right.\right. \\  & \quad \ \ \quad \ \ \quad \ \   \quad \ \ \quad  \left.\left.  - \frac{1}{m}\sum_{i=1}^{m \gamma_j}\E_{(G^j, \mathbf{f}^j ) \sim \mu_{\mathcal{G}_j}}\left[\sup_{\Theta \in \mathrm{Lip}_{L,B}} \mathcal{L}(\Theta_{G^j}^P(\mathbf{f}^j), y_j) \right] \right)^2 \right] \\
= & \Gamma\sum_{j=1}^\Gamma \Var_{(G^j, \mathbf{f}^j) \sim \mu_{\mathcal{G}_j}}\left[\sup_{\Theta \in \mathrm{Lip}_{L,B}} \frac{1}{m} \sum_{i=1}^{\gamma_j\cdot m} \mathcal{L}(\Theta_{G^j}^P(\mathbf{f}^j),y_j) \right]  \\
= & \Gamma\sum_{j=1}^\Gamma \frac{\gamma_j}{m} \Var_{(G^j, \mathbf{f}^j) \sim \mu_{\mathcal{G}_j}}\left[\sup_{\Theta \in \mathrm{Lip}_{L,B}} \mathcal{L}(\Theta_{G^j}^P(\mathbf{f}^j),y_j) \right]
\\
\leq & \Gamma\sum_{j=1}^\Gamma \frac{\gamma_j}{m} \E_{(G^j, \mathbf{f}^j) \sim \mu_{\mathcal{G}_j}}\left[\sup_{\Theta \in \mathrm{Lip}_{L,B}} \left| \mathcal{L}(\Theta_{G^j}^P(\mathbf{f}^j),y_j) - \mathcal{L}(\Theta_{W^j}^P (f^j),y_j)\right|^2 \right]\\
\leq & \Gamma\sum_{j=1}^\Gamma \frac{\gamma_j}{m} \E_{(G^j, \mathbf{f}^j) \sim \mu_{\mathcal{G}_j}}\left[\sup_{\Theta \in \mathrm{Lip}_{L,B}} L_\mathcal{L}^2\| \Theta_{G^j}^P(\mathbf{f}^j) - \Theta_{W^j}^P (f^j)\|_\infty^2 \right].
\end{aligned}
\]
We now apply \Cref{cor:unifExpValue} to get 
\[
\begin{aligned}
\leq \Gamma & \sum_{j=1}^\Gamma \frac{\gamma_j}{m} L_\mathcal{L}^2 \Bigg(4(1 + \sqrt{\pi}) \Bigg(  T_1 \frac{\big(1+\log(N)\big)N^{2\alpha}}{N^{\frac{1}{D_\chi+1} }} + T_2 \frac{\big(1+\log(N)\big) N^{2 \alpha}}{N}
 + \textcolor{black}{T_3 \varepsilon} \Bigg) \\ &+  \mathcal{O} \left(   \exp(-N)\right)  \Bigg).
\end{aligned}
\]

Hence, by \Cref{lemma:BHCineq}, 
\[
\begin{aligned}
& \sum_{k} \mathbb{P}\big(
\mathbf{m}\in \mathcal{M}_k
\big) \times
\sup_{\mathbf{m}\in \mathcal{M}_k} \E_{\mathcal{T}_{\mathbf{m}} \sim \mu_{\mathcal{G}^{\mathbf{m}}}} \left[\sup_{\Theta \in \mathrm{Lip}_{L,B}}\left(\sum_{j=1}^\Gamma \left( \frac{1}{m}\sum_{i=1}^{m \gamma_j} \mathcal{L}(\Theta_{G_i^j}^P(\mathbf{f}_i^j),y_j)   \right.\right.\right.\\  & \quad \ \ \quad \ \ \quad \ \ \quad \ \ \quad \ \ \quad \ \ \quad \ \ \quad \ \ \quad \ \ \quad \ \ \quad   \left.\left.\left. - \frac{1}{m} \sum_{i=1}^{m \gamma_j}\E_{(G^j, \mathbf{f}^j )\sim \mu_{\mathcal{G}_j}}\left[ \mathcal{L}(\Theta_{G^j}^P(\mathbf{f}^j),y_j) \right] \right) \right)^2 \right] \\
& \leq \frac{\sqrt{\pi}}{2} 2^\Gamma \sum_{j=1}^\Gamma \frac{\gamma_j}{m} \E_{\mathcal{T}_{\mathbf{m}} \sim \mu_{\mathcal{G}^{\mathbf{m}}}} \left[\sup_{\Theta \in \mathrm{Lip}_{L,B}}\left(\sum_{j=1}^\Gamma \left( \frac{1}{m}\sum_{i=1}^{m \gamma_j} \mathcal{L}(\Theta_{G_i^j}^P(\mathbf{f}_i^j),y_j)   \right.\right.\right.\\  & \quad \ \ \quad \ \ \quad \ \ \quad \ \ \quad \ \ \quad \ \ \quad \ \ \quad \ \   \left.\left.\left. - \frac{1}{m} \sum_{i=1}^{m \gamma_j}\E_{(G^j, \mathbf{f}^j )\sim \mu_{\mathcal{G}_j}}\left[ \mathcal{L}(\Theta_{G^j}^P(\mathbf{f}^j),y_j) \right] \right) \right)^2 \right]\\
& \leq  \frac{\sqrt{\pi}}{2} 2^\Gamma  \Gamma  \sum_{j=1}^\Gamma \frac{\gamma_j}{m} L_\mathcal{L}^2 \Bigg(4(1 + \sqrt{\pi}) \Bigg(  T_1 \frac{\big(1+\log(N)\big)N^{2\alpha}}{N^{\frac{1}{D_\chi+1} }} + T_2 \frac{\big(1+\log(N)\big) N^{2 \alpha}}{N}
 + \textcolor{black}{T_3 \varepsilon} \Bigg) \\ &  +  \mathcal{O} \left(   \exp(-N)\right) \Bigg) ,
\end{aligned}
\]
where $S_l^{(j)},R_l^{(j)},T_l^{(j)}$ are the according constants from \Cref{thm:unifExpValue} for each class $j$ and are defined in \Cref{eq:thmC10-constants}.
All in all, we get 
\[
\begin{aligned}
  &   \E_{\mathcal{T} \sim \mu^m} \left[ \sup_{\Theta \in \mathrm{Lip}_{L,B}} \left( \frac{1}{m} \sum_{j=1}^\Gamma  \sum_{i=1}^{m_j} \mathcal{L}(\Theta_{G_i^j}^P(\mathbf{f}_i^j),y_j) - \E_{(G, \mathbf{f} , y)\sim \mu}\left[ \mathcal{L}(\Theta_{G}^P(\mathbf{f}), y) \right] \right)^2 \right] \leq 2^\Gamma \frac{8 \|\mathcal{L}\|_\infty^2}{m} \pi  \\
  &+ \frac{\sqrt{\pi}}{ m} 2^\Gamma  \Gamma  \sum_{j=1}^\Gamma \gamma_j L_\mathcal{L}^2 \Bigg(4(1 + \sqrt{\pi}) \Bigg(  T_1 \frac{\big(1+\log(N)\big)N^{2\alpha}}{N^{\frac{1}{D_\chi+1} }} + T_2 \frac{\big(1+\log(N)\big) N^{2 \alpha}}{N}
 + \textcolor{black}{T_3 \varepsilon} \Bigg)  \\
  & +  \mathcal{O} \left(   \exp(-N)\right) \Bigg).
\end{aligned}
\]
 We define 
\begin{equation}
\label{eq:defEndConstant}
    C= 4(1+  \sqrt{\pi}) \max_{j=1, \ldots, \Gamma} \left(\sum_{i=1}^3 T_i^{(j)} \right),
\end{equation}
 leading to 
\[
\begin{aligned}   & \E_{\mathcal{T}\sim \mu^m_\mathcal{G}}\Big[\sup_{\Theta \in \mathrm{Lip}_{L,B}}\Big(R_{emp}(\Theta^P)   - R_{exp}(\Theta^P) \Big)^2 \Big]  \\ &\leq   \frac{2^\Gamma8\|\mathcal{L}\|_\infty^2\pi}{m} + \frac{2^\Gamma \Gamma L_{\mathcal{L}}^2   C}{m}       \left( \E_{N \sim \nu} \left[ \frac{1}{N} + \frac{1 + \log(N)}{N^{1/D_j + 1}} + \mathcal{O}\left( \exp(-N)N^{ {\frac{3}{2}}T-\frac{3}{2}} \right) \right] \right).
\end{aligned}
\]
\end{proof}

\section{Miscellaneous Results}
\label{sec:appendix other results}
\begin{theorem}[Hoeffding's Inequality]
\label{thm:Hoeffdings}
Let $\{Y_1, \ldots, Y_N\}$ be independent random variables such that $a \leq Y_i \leq b$ almost surely. Then, for every $k>0$, 
\[
\mathbb{P}\Big(
\Big|
\frac{1}{N} \sum_{i=1}^N (Y_i - \E[Y_i] )
\Big| \geq k
\Big) \leq 
2 \exp\Big(-
\frac{2 k^2 N}{( b-a)^2}
\Big).
\]
\end{theorem}

\begin{lemma}[Lemma 4 in \cite{keriven2020convergence}, Lemma B.1 in \cite{maskey2022generalization}]
\label{lemma:kerivenLemma4}
Let $(\chi,d, \P) $ be a metric-probability space and
$W$ be an admissible graphon. Consider a metric-space signal  $f: \chi \to \mathbb{R}$ with $\|f\|_\infty < \infty$. Suppose that $\X = \{X_1, \ldots, X_N\}$ are drawn i.i.d. from $\chi$ via $\P$ and let $p \in (0,1)$. Then, with probability at least $1-p$, we have 
\begin{equation*}
\begin{aligned}
& \left\|
\frac{1}{N} \sum_{i=1}^N W(\cdot,X_i)f(X_i) - \int_\chi W(\cdot, x) f(x) d\P(x)\right\|_\infty
 \\ & \leq  
\frac{\|f\|_\infty \Big(\zeta \cl (\sqrt{\log (C_\chi)} +  \sqrt{D_\chi}) + (\sqrt{2}\cmax+\zeta \cl) \sqrt{\log2/p} \Big) }{ \sqrt{N}},
\end{aligned}
\end{equation*}
where 
\begin{equation}
    \label{eq:zeta2}
    \zeta  := \frac{2}{\sqrt{2}}e\Big(\frac{2}{\ln(2)} +1 \Big)\frac{1}{\sqrt{\ln(2)}} C
\end{equation}
 and $C$ is the universal constant from Dudley's inequality (see  \cite[Theorem 8.1.6]{vershynin_2018}).
\end{lemma}
\begin{lemma}[Lemma B.2 in \cite{maskey2022generalization}]
\label{lemma:LowerBoundDegree}
Let $(\chi,d, \P) $ be a metric-probability space and
$W$ be an admissible graphon.
Suppose that $\X = \{X_1, \ldots, X_N\}$ are drawn i.i.d. from $\chi$ via $\P$  and let $p \in (0,1)$.
Let
\begin{equation*}
\sqrt{N} \geq  2 \Big(\zeta   \frac{\cl}{\cmin} \big(\sqrt{\log (C_\chi)} +  \sqrt{D_\chi}\big) +
    \frac{\sqrt{2} \cmax + \zeta \cl}{ \cmin } \sqrt{\log 2/p}\Big),
\end{equation*}
where $\zeta$ is defined in \Cref{eq:zeta}. Then, with probability at least $1-p$ the following two inequalities hold: For every $x\in\chi$,
\begin{equation}
    \label{eq:d_XLowerBound}
\mathrm{d}_X( x) \geq \frac{\cmin}{2}
\end{equation}
and 
\begin{equation}
\label{eq:lemmab1-1}
     \begin{aligned}
& \left\|
\frac{1}{N} \sum_{i=1}^N W(\cdot,X_i)  - \int_\chi W(\cdot, x)   d\P(x)\right\|_\infty
 \\ & \leq  
\frac{  \Big(\zeta \cl (\sqrt{\log (C_\chi)} +  \sqrt{D_\chi}) + (\sqrt{2}\cmax+\zeta \cl) \sqrt{\log2/p} \Big) }{ \sqrt{N}}.
\end{aligned}
\end{equation}
\end{lemma}

\begin{lemma}
 \label{lemma:RecRecGen}
 Let $(\eta^{(l)})_{l=0}^T$ be a sequence of real numbers satisfying $\eta^{(l+1)} \leq a^{(l+1)}\eta^{(l)} + b^{(l+1)}$ for $l = 0,\ldots, T-1$,  for some real numbers $a^{(l)}, b^{(l)}$, $l = 1, \ldots, T$. Then
 \[
 \eta^{(T)} \leq \sum_{l=1}^T b^l\prod_{l' = l+1}^Ta^{(l')} + \eta^{(0)}\prod_{l=1}^Ta^{(l)},
 \]
 where we define the product $\prod_{T+1}^{T}$ as $1$.
 \end{lemma}

\section{Probability Measure of Graph Datasets}
\label{sec:appendix_prob_measure}
In this section, we establish rigorously the probability measure on the graph-signal space $\mathcal{S}^F$ that we consider throughout this work. To facilitate this, we initially define the spaces $\mathcal{P}_j^N$ for a class $j = 1, \ldots, \Gamma$ and $N \in \mathbb{N}$ as
\begin{equation*}
    \mathcal{P}_j^N = \{0,1\}^{N\times N}\times (\chi^j)^N\times \Big(L^{\infty}((\chi^j)^2)\times L^{\infty}(\chi^j)\Big).
\end{equation*} 
and subsequently denote
\begin{equation*}
  \mathcal{P} =  \bigcup_{N \in \mathbb{N}} \bigcup_{j=1}^\Gamma  \, \mathcal{P}_j^N
\end{equation*}

For every fixed $N \in \mathbb{N}$ and class $j=1, \ldots, \Gamma$, we  define the probability measure for $G \in \{0,1\}^{N \times N}$ conditioned on the sample $\X\in(\chi^j)^N$ and noise $(V^j, g^j)\in L^{\infty}((\chi^j)^2)\times L^{\infty}(\chi^j)$ as
\begin{equation*}
  \begin{aligned}
    & \mu_{\{0,1\}^{N \times N}}( \{G\} \, | \, \X, (V^j, g^j)) \\
  =  &  \prod_{(i,k) \in \mathcal{E}} \left(W^j(X_i, X_k) + V^j(X_i,X_k)\right) \prod_{(i,k) \not\in \mathcal{E}}\left(1 -\left(W^j(X_i, X_k) + V^j(X_i,X_k)\right)\right).
\end{aligned}
\end{equation*}
For a measurable subset $A \subset \mathcal{P}^N_j$, we define the measure $\mu_{\mathcal{P}^N_j}$ by
\[\mu_{\mathcal{P}^N_j}(A) = \int_{L^{\infty}((\chi^j)^2)\times L^{\infty}(\chi^j)} \int_{(\chi^j)^N} \mu_{\{0,1\}^{N \times N}}( A_{\X,V^j,g^j} \, | \, \X, (V^j, g^j)) 
 (d\mu^j)^N(\X)d\sigma^j(V^j,g^j),\]
 where $A_{\X,V^j,g^j} := A \cap (\{0,1\}^{N \times N} \times \{\X\} \times \{(V^j,g^j)\})$.
We then define a measure $\mu_{\mathcal{P}}$ for $S \subset \mathcal{P}$ via
\begin{equation*}
  \begin{aligned}
    & \mu_{\mathcal{P}}(S)    = \sum_{j=1}^\Gamma \gamma_j  \sum_{N \in \mathbb{N}} \nu(N)    \mu_{\mathcal{P}^N_j}( S^N_j),
\end{aligned}
\end{equation*}
where $S^N_j$ is the restriction of $S$ to $\mathcal{P}^N_j$. 
Lastly, we consider the mapping $\mathrm{map}: \mathcal{P} \to \mathcal{S}^F$ via
\begin{equation*}
\begin{aligned}
     \mathrm{map}|_{\mathcal{P}^N_j}: \, & \mathcal{P}^N_j \to \mathcal{S}^{N \times F} \\
    (G, \X, (V,g)) &  \mapsto (G, f^j(\X) + g^j(\X) ),
\end{aligned}
\end{equation*}
and define $\mu$ as the pushforward measure of $\mu_{\mathcal{P}}$ via $\mathrm{map}$. 

\section{Details on Numerical Experiments}
\label{sec:appendix_exp}
In this section, we provide details for the numerical experiments from \Cref{sec:experiments}. We note that the experiments in \Cref{sec:experiments} largely follow the numerical experiments described in Section 4 by \cite{maskey2022generalization}, with additional details, specifically regarding the computation of the bounds, in \citep[Appendix D.2]{maskey2022generalization}.

\subsection{Dataset}
We create four synthetic datasets of random graphs from two different random graph models, each with varying sparsity.
The domains of the graphons, is taken as the Euclidean space $[0, 1]$. We consider Erdös-Rényi graphs with edge probably $0.4$ with constant signal, represented by $\{W_1, f_1\}$
with $W_1(x, y) = 0.4$ and $f_1(x) = 0.5$. We also consider a smooth version of a stochastic block model,
represented by $\{W_2, f_2\}$ with $W_2(x, y) = \sin(2\pi x) sin(2\pi y)/2\pi+0.25$ and $f_2(x) = \sin(x)/2$. For each sparsity level $\alpha=\{0,0.1,0.2,0.3\}$, we create one dataset consisting of $50K$ graphs of size $50$ from each RGSM. We split each dataset to 90\% training examples and 10\% test. Hence, the relevant constants for our generalization bound in \Cref{thm:main_gen_bound_deformed_graphon} are given by $m=90.000, N=50, \Gamma=2, \varepsilon=0, \|W\|_\infty=0.41, L_W = 0.5, \|f\|_\infty = 0.5, L_f = 0.5$ and $\cmin=0.25$. Only the sparsity level $\alpha$ varies between the datasets, taking values in $\{0,0.1,0.2,0.3\}$.  

\subsection{Model Details}
For our network, we choose MPNNs with GraphSage layers \citep{hamilton2017inductive}. We implement this with Pytorch Geometric \citep{Fey/Lenssen/2019}. We consider MPNNs with $1,2$ and $3$ layers. 
Given a node $u$, GraphSage updates node features according to 
\begin{equation*}
    \mathbf{f}^{(t)}_u  = \mathbf{W}^{(1)}\mathbf{f}^{(t-1)}_u + \mathbf{W}^{(2)} \mathbf{AGG}\left(\left\{\left\{ \mathbf{f}^{(t-1)}_v \right\}\right\}_{v \in \mathcal{N}(u)}\right),
\end{equation*}
where $\mathbf{W}^{(1)} \in \mathbb{R}^{128\times128}$, $\mathbf{W}^{(1)} \in \mathbb{R}^{128\times1}$ and $\mathbf{W}^{(2)}, \mathbf{W}^{(3)} \in \mathbb{R}^{128\times128}$. The aggregation $\mathbf{AGG}$ given by mean aggregation or normalized sum aggregation.
The message functions are then defined by
\begin{equation*}
\phi^{(t)}(\mathbf{f}^{(t-1)}_u, \mathbf{f}^{(t-1)}_v) = \mathbf{f}^{(t-1)}_v
\end{equation*}
Finally, the update functions are given by
\begin{equation*}
\psi^{(t)}(\mathbf{f}^{(t-1)}_u, \mathbf{m}^{(t-1)}) = \mathbf{W}^{(1)}\mathbf{f}^{(t-1)}_u + \mathbf{W}^{(2)}\mathbf{m}^{(t-1)},
\end{equation*}

We then consider a global mean pooling layer, and apply a last linear layer $\mathbf{Q} \in \mathbb{R}^{2\times128}$ (including bias) with input dimension 128 and output dimension 2. This last linear layer is seen as part of the loss function in the analysis, and contributes to the generalization bound via the Lipschitz constant and infinity norm of the loss, as seen in Theorem 3.3.

\subsection{Experimental Setup}
The loss is given by soft-max composed with cross-entropy. We
consider Adam \citep{kingma2017adam} with learning rate $\mathrm{lr} = 0.01$. For experiments including weight decay, we use an
$l_2$-regularization on the weights with factor $0.1$ for mean aggregation, and factor $0.05$ for normalized sum aggregation. We train for $1$ epoch. The batch size is $64$. We consider $1,
2$ and $3$ layers.

 \end{document}